\documentclass[nohyperref]{article}

\usepackage{microtype}
\usepackage{graphicx}
\usepackage{subfigure}
\usepackage{booktabs} 

\usepackage{hyperref}

\usepackage[accepted]{icml2023}

\usepackage{amsmath}
\usepackage{amssymb}
\usepackage{mathtools}
\usepackage{amsthm}

\usepackage{dsfont}

\renewcommand{\epsilon}{\varepsilon}

\makeatletter
\def\th@plain{%
  \thm@notefont{}
  \itshape 
}
\def\th@definition{%
  \thm@notefont{}
  \normalfont 
}
\makeatother

\DeclareMathOperator*{\Argmin}{arg\,min}
\DeclareMathOperator{\Diag}{diag}
\DeclareMathOperator{\distHamming}{d_H} 
\DeclareMathOperator{\Exp}{exp}
\DeclareMathOperator{\Exps}{e}
\DeclareMathOperator{\Identity}{I}
\DeclareMathOperator*{\Minimize}{Minimize}
\DeclareMathOperator{\Vecspan}{Vec}

\newcommand{\abs}[1]{\left\lvert#1\right\rvert}
\newcommand{\ballHamming}[2]{B_{#2}(#1)}
\newcommand{\card}[1]{\left\lvert#1\right\rvert} 
\newcommand{\condproba}[2]{\mathbb{P}\left(#1\middle|#2\right)}

\newcommand{\diag}[1]{\Diag\left(#1\right)}
\newcommand{\defeq}{\vcentcolon =}
\renewcommand{\exp}[1]{\Exp\left(#1\right)}
\newcommand{\exps}[1]{\Exps^{#1}}
\newcommand{\Indic}{\mathds{1}}
\newcommand{\indic}[1]{\Indic_{#1}}
\newcommand{\kernel}[1]{\mathrm{Ker}\left(#1\right)}
\newcommand{\lambdamax}{\lambda_{\max}}
\newcommand{\lambdamin}{\lambda_{\min}}
\newcommand{\mutilde}{\tilde{\mu}}

\newcommand{\Nnint}{\mathbb{N}}
\newcommand{\norm}[1]{\left\lVert#1\right\rVert}
\newcommand{\opnorm}[1]{\norm{#1}_{\mathrm{op}}}
\newcommand{\scalar}[2]{\langle#1,#2\rangle}
\newcommand{\range}[1]{\mathrm{Im}(#1)}
\newcommand{\Reals}{\mathbb{R}}
\newcommand{\RR}{\mathbb{R}}
\newcommand{\spec}[1]{\mathrm{Spec}\left(#1\right)} 
\newcommand{\vecspan}[1]{\Vecspan\left(#1\right)}

\newcommand{\bigo}[1]{\mathcal{O}\left(#1\right)}

\theoremstyle{plain}
\newtheorem{theorem}{Theorem}[section]
\newtheorem{proposition}[theorem]{Proposition}
\newtheorem{lemma}[theorem]{Lemma}
\newtheorem{corollary}[theorem]{Corollary}
\theoremstyle{definition}
\newtheorem{definition}[theorem]{Definition}
\newtheorem{assumption}[theorem]{Assumption}
\theoremstyle{remark}
\newtheorem{remark}[theorem]{Remark}

\newcommand{\temps}{\mu}
\newcommand{\eigen}{\lambda}
\newcommand{\sol}{q}

\newcommand{\radius}{\rho}

\newcommand{\supgradient}{M}
\newcommand{\lipconstant}{L}
\newcommand{\dimim}{ N_0 }

\newcommand{\dd}{\mathop{}\!\mathrm{d}}
\newcommand{\rank}[1]{\mathrm{Rank}\left(#1\right)}

\newcommand{\funclip}{\Lambda}

\DeclareMathOperator{\dimension}{dim}
\newcommand{\subspace}{E}
\newcommand{\cst}{\gamma}
\newcommand{\Cst}{\Gamma}

\newcommand{\Concat}{\varphi}
\newcommand{\concat}[1]{\Concat (#1)}
\newcommand{\context}[1]{c(#1)}
\newcommand{\Corpus}{\mathcal{C}}
\newcommand{\Dico}{\mathcal{D}}
\newcommand{\Docs}{[D]^{*}}
\newcommand{\docs}[1]{[D]^{#1}}
\newcommand{\Doctovec}{\varphi}
\newcommand{\doctovec}[1]{\Doctovec(#1)}
\newcommand{\Indset}{\mathcal{S}}
\newcommand{\Indsetchange}{\mathcal{E}}
\newcommand{\freq}{f}
\newcommand{\idf}{v}
\newcommand{\maxidf}{\idf_{\max}}
\newcommand{\minidf}{\idf_{\min}}
\newcommand{\maxmult}{\mult_{\max}}
\newcommand{\minsingular}[1]{\sigma_{\min}(#1)}
\newcommand{\maxsingular}[1]{\sigma_{\max}(#1)}
\newcommand{\mult}{m}
\newcommand{\mtilde}{\tilde{m}}
\newcommand{\nhood}[1]{\gamma(#1)}
\newcommand{\Normtfidf}{\phi}
\newcommand{\normtfidf}[1]{\Normtfidf (#1)}
\newcommand{\onehot}[1]{\indic{#1}}
\newcommand{\Philinear}{\Phi^{\text{lin}}}
\newcommand{\Psilinear}{\Psi^{\text{lin}}}
\newcommand{\pitilde}{\tilde{\pi}}
\newcommand{\Proj}{P}
\newcommand{\Red}{R}
\newcommand{\Softmax}{\sigma}
\newcommand{\softmax}[1]{\Softmax\left(#1\right)}
\newcommand{\Tfidf}{\varphi}
\newcommand{\tfidf}[1]{\Tfidf (#1)}
\newcommand{\tmax}{T_{\max}}
\newcommand{\Token}{Q}
\newcommand{\winsize}{\nu}
\newcommand{\xtilde}{\tilde{x}}

\def\restriction#1#2{\mathchoice
              {\setbox1\hbox{${\displaystyle #1}_{\scriptstyle #2}$}
              \restrictionaux{#1}{#2}}
              {\setbox1\hbox{${\textstyle #1}_{\scriptstyle #2}$}
              \restrictionaux{#1}{#2}}
              {\setbox1\hbox{${\scriptstyle #1}_{\scriptscriptstyle #2}$}
              \restrictionaux{#1}{#2}}
              {\setbox1\hbox{${\scriptscriptstyle #1}_{\scriptscriptstyle #2}$}
              \restrictionaux{#1}{#2}}}
\def\restrictionaux#1#2{{#1\,\smash{\vrule height .8\ht1 depth .85\dp1}}_{\,#2}} 

\newcommand{\cQ}{\mathcal{Q}}

\newcommand{\ie}{{\em i.e.,~}}

\icmltitlerunning{On the Robustness of Text Vectorizers}

\begin{document}

\twocolumn[
\icmltitle{On the Robustness of Text Vectorizers}



\icmlsetsymbol{equal}{*}

\begin{icmlauthorlist}
\icmlauthor{R\'emi Catellier}{uca,inria}
\icmlauthor{Samuel Vaiter}{uca,cnrs}
\icmlauthor{Damien Garreau}{uca,inria}
\end{icmlauthorlist}

\icmlaffiliation{uca}{Universit\'e C\^ote d'Azur, CNRS, LJAD, France}
\icmlaffiliation{inria}{Inria, France}
\icmlaffiliation{cnrs}{CNRS, France}

\icmlcorrespondingauthor{Damien Garreau}{damien.garreau@unice.fr}

\icmlkeywords{Natural Language Processing, Theory, Robustness, Embedding, Tokenization}

\vskip 0.3in
]



\printAffiliationsAndNotice{} 

\begin{abstract}
A fundamental issue in machine learning is the robustness of the model with respect to changes in the input. 
In natural language processing, models typically contain a first embedding layer, transforming a sequence of tokens into vector representations.
While the robustness with respect to changes of continuous inputs is well-understood, the situation is less clear when considering discrete changes, for instance replacing a word by another in an input sentence. 
Our work formally proves that popular embedding schemes, such as concatenation, TF-IDF, and Paragraph Vector ({\em a.k.a.}~\texttt{doc2vec}), exhibit robustness in the H\"older or Lipschitz sense with respect to the Hamming distance.
We provide quantitative bounds for these schemes and demonstrate how the constants involved are affected by the length of the document.
These findings are exemplified through a series of numerical examples.
\end{abstract}


\section{Introduction}
\label{sec:introduction}

Recent advances in natural language processing (NLP) have exceeded all expectations. 
In particular, the advent of large language models such as BERT \citep{devlin_et_al_2018} and GPT \citep{brown_et_al_2020} are transforming radically the way we interact with computers. 
They typically rely on a deep neural network (DNN) architecture and are trained on a variety of tasks such as sentiment analysis, translation, and text summarization. 

A known issue with DNNs is the existence of \emph{adversarial examples}: examples modified in order to radically change the output of the model. 
Initially popularized in the context of image classification \citep{szegedy_et_al_2014}, such examples also exist in NLP and a flourishing literature exists on this topic \citep{zhang_et_al_2020}. 
This problem has sparked a tremendous interest into the \emph{robustness} of models with respect to small changes in the input. 
In this paper, we focus on the robustness of the \emph{vectorization} NLP pipelines: the transformation of the input document into a vector representation. 
We will consider documents as ordered sequences of tokens, not necessarily corresponding to words. 
For instance, GPT 2 uses Byte Pair encoding \citep{gage_1994,sennrich_et_al_2016}, which relies on tokens corresponding to sub-words. 

As far as we reckon, there are essentially three main schools of thought when it comes to vectorization:\\
\emph{(i)} \textbf{concatenation} of vectors corresponding to each token of the document.
These vectors are often called \emph{word vectors} when the tokens are individual words.
They can either be one-hot representations of the tokens, or obtained by a mapping learned from data. 
A celebrated approach to produce word vectors is \texttt{word2vec} \citep{mikolov_et_al_2013,mikolov_et_al_2013_a}, which transports semantic properties to the embedding space. 
Many other methods exist, such as GloVe \citep{pennington_et_al_2014}, EMF \citep{li_et_al_2015}, WordPiece \citep{wu_et_al_2016}, FastText \citep{bojanowski_et_al_2017}, and ELMo \citep{peters_et_al_2018}. 
Positional information is typically added to the token embeddings. 
\\
\emph{(ii)} \textbf{TF-IDF (term frequency - inverse document frequency)}, taking words as tokens and simply considering the frequencies of each individual word in the document. 
These frequencies are reweighted by an overall importance term to take into account the lesser importance of frequently appearing words such as articles. 
This is the historical approach to text vectorization \citep{luhn_1957,jones_1972}.\\
\emph{(iii)} \textbf{\emph{ad hoc} approaches}. 
Notably, Paragraph Vector (also known as \texttt{doc2vec}  \citep{le_mikolov_2014}) extends the ideas of \texttt{word2vec}. 
Although we will focus on \texttt{doc2vec} in this work, we emphasize that there exists other \emph{ad hoc} approaches, such as skip-thought vectors \citep{kiros_et_al_2015}, quick-thought \citep{logeswaran_honglak_2018}, or universal sentence encoder \citep{cer_et_al_2018}. 

\emph{A priori}, vectorizers are not designed to be robust to small changes. 
Even when modifying a single word of the input document, the embedding could change drastically.
Thus, we ask the following question: 
\begin{quote}
    \emph{Are text vectorizers \textbf{provably} robust with respect to modifying a small subset of the document?}
\end{quote}

Typical notions of robustness in machine learning deals with \emph{continuous} input data: changing slightly the observation means that for instance its $\ell^2$-norm evolves infinitesimally.
The challenge of our analysis is the fundamentally  \emph{discrete} nature of text data. 
Changing a word in a document is usually not innocuous -- one can think of extreme cases where the meaning of this word is flipped -- and vectorizers sensitive to the semantics of input documents should capture this phenomenon.  
Nevertheless, we show that the answer is positive for all vectorizers that we study. 
Another difficulty is that the mathematical formalization of some of these vectorizers was not the main concern of the community.
A necessary first step is thus to give an unequivocal definition of our objects of interest.


\paragraph{Contributions.}
In this paper, we analyze the robustness of vectorizers as their local regularity (Lipschitz, H\"older) with respect to the \textbf{Hamming distance} (Section~\ref{sec:framework}). We prove:\\
$\bullet$ the $1/2$-H\"older continuity of \textbf{concatenation of token and positional embeddings} (Proposition~\ref{prop:concatenation-robustness}); \\
$\bullet$ the Lipschitz continuity of \textbf{TF-IDF} (Proposition~\ref{prop:robustness-non-normalized-tfidf}), and the $1/2$-H\"older continuity of it normalized variant (Proposition~\ref{prop:robustness-normalized-tfidf});\\
$\bullet$ the Lipschitz continuity of \textbf{\texttt{doc2vec}} (Theorem~\ref{th:bounded-traj}).
As a necessary step to derive the latter, we make two new mathematical contributions (see Appendix), we propose:\\
$\bullet$ a \textbf{local Lipschitz analysis of the softmax} (Theorem~\ref{lemma:softmax-lipschitz-appendix});\\
$\bullet$ a \textbf{Gr\"onwall--Bellman--Bahouri result} (Theorem~\ref{lemma:CS-appendix}) needed when casting the \texttt{doc2vec} analysis as an ODE problem.
The code for all experiments of the paper is available at \url{https://github.com/dgarreau/vectorizer-robustness}. 


\paragraph{Related work.} 
\emph{(Adversarial examples).} A major motivation for studying robustness is its impact on the existence of adversarial examples. 
In the case of DNNs, robustness often means Lipschitz continuity with respect to the inputs. 
For instance, one can show that a network having a small Lipschitz constant prevents the existence of small adversarial changes. 
More precisely, \citet{hein_andriushchenko_2017} provide a lower bound on the norm of the input manipulation needed to change the classifier decision inversely proportional to the Lipschitz constant of the network. 
This was later extended by \citet{weng_et_al_2018_a} to DNNs with ReLU activations. 
Quantitatively, \citet{weng_et_al_2018} show that fully connected layers have a Lipschitz constant potentially as large as the operator norm of the weight matrix. 
From a practical point of view, it has also been noticed that enforcing the Lipschitz constants of the layers to remain low does improve the robustness~\citep{cisse_et_al_2017}.\\
\emph{(Generalization \& interpolation).} It is known that robust algorithms generalize better. 
In particular, \citet{xu_mannor_2012} derive generalization bounds for generic algorithms depending in their robustness. 
The definition of robustness here includes Lipschitz continuous DNNs.  
More recently, \citet{bubeck_selke_2021} extending \citep{bubeck_et_al_2020} showed that in order to train Lipschitz continuous models, one has to take a large number of parameters.\\
\emph{(Theory of vectorizers).} Surprisingly, the robustness of vectorizers received little attention until now on the theoretical side, and all previous works on robustness assume \emph{continuous} input. 
Nevertheless, there exist some theoretical works on similar problems. 
Most notably, \citet{arora_et_al_2016} analyze a large class of word vectorizers and explain how the intriguing alignment properties observed experimentally appear. 

\paragraph{Notations.}
For $u\in\Reals^p$, we denote by $\norm{u}$ its Euclidean norm.
Let $g : \Reals \times \Reals^d \to \Reals$ be a function. 
The derivative in the time variable ($\temps$) is denoted by $\partial_\mu g$ whereas $\nabla g$ (resp. $\nabla^2 g$) denotes the Jacobian (resp. the Hessian) of $g$ in the space variable. 
We let $\Indic = (1,\ldots,1)^\top \in \Reals^d$.
For a matrix $R$, $\minsingular{R}$ is its smallest singular value.
For a given set $\Indset$, $\card{\Indset}$ is its cardinal.


\section{Framework}
\label{sec:framework}

Let us now present the mathematical framework in which we perform our analysis. 
We consider tokens from a finite dictionary~$\Dico$, identified as $[D]\defeq \{1,\ldots,D\}$. 
A \emph{document} $x$ built on $\Dico$ is a finite sequence of elements of $\Dico$, and we write $\Docs$ for the set of all documents. 
Thus the central object of our work, a vectorizer, is simply a mapping $\varphi : \Docs \to \Reals^d$, where $d$ is the dimension of the embedding. 
The \emph{length} of $x$ will be denoted by $T(x)$, and therefore $x$ can be written as $(x_1,\ldots,x_{T(x)})$. 
The set of all documents over $\Dico$ of length $T$ will be denoted $\docs{T} \subset \Docs$. 
When there is no ambiguity, we remove the dependency in $x$ from our notation, \emph{e.g.}, $T(x)$ becomes $T$.

As discussed in the related work, robustness is often synonym with \emph{Lipschitz continuity} of the model -- distance between outputs lies within a constant factor of the distance between inputs.
As distance between input documents $x$ and $\tilde{x}$ of same length, we consider the \emph{Hamming distance}, which is the number of indices such that $x_t$ and $\tilde{x}_t$ differ:
\[
\distHamming(x,\tilde{x}) \defeq \card{\{ t \in [T] : x_t \neq \tilde{x}_t \}}
\, .
\]
The distance between outputs will simply be measured by the Euclidean norm in $\Reals^d$.
In definitive, for a given document length $T$, what we call \textbf{Lipschitz continuity} of the vectorizer $\varphi$ can be written as
\begin{equation}
\label{eq:def-lipschitz-continuity}
\forall x,\tilde{x}\in\docs{T}, \quad \norm{\varphi(x)-\varphi(\tilde{x})} \leq C \distHamming(x,\tilde{x})
\, ,
\end{equation}
where $C$ is called the Lipschitz constant. 
Another way to quantify robustness is to allow for an exponent in Eq.~\eqref{eq:def-lipschitz-continuity}: 
\begin{equation}
\label{eq:def-holder}
\forall x,\tilde{x}\in\docs{T}, \quad \norm{\varphi(x)-\varphi(\tilde{x})} \leq C \distHamming(x,\tilde{x})^\beta 
\, ,
\end{equation}
with $1\geq \beta >0$. 
This is known as \textbf{H\"older continuity}, and coincides with Lipschitz continuity whenever $\beta=1$. 
While it is known that Lipschitz continuity implies H\"older continuity on the real line when $\beta \leq 1$, this is not the case here, since $\distHamming$ takes values in $\Nnint$. 
Thus in our setting, \textbf{Lipschitz continuity is a weaker notion of robustness than H\"older continuity. }

Often we obtain more precise results, depending explicitly on the set of indices such that the documents differ. 
To this extent, for a given subset $\Indset$ of $[T]$, we define the \emph{set of $\Indset$-close documents} $\ballHamming{x}{\Indset}$ of $x \in \docs{T}$ as
\[
    \ballHamming{x}{\Indset} =
    \{
        \tilde x \in [D]^T \, : \,
        x_i = \tilde x_i \text{ for } i \not\in \Indset
    \} 
\, .
\]
Said alternatively, $\xtilde \in \ballHamming{x}{\Indset}$ if it is  obtained by replacing the tokens of $x$ with indices belonging to $\Indset$ by arbitrary tokens in $\Dico$. 
We note that $\ballHamming{x}{\Indset}$ is a subset of the Hamming ball of radius $\card{S}$. 
Let us consider for instance the document $x=$ ``the quick brown fox'' and the set of perturbed indices $\Indset=\{2,3\}$ 
Here, $x$ has length $T=4$, $\card{\Indset}=2$, and an element of $\ballHamming{x}{\Indset}$ is the document $\xtilde=$ ``the slow blue fox.''


\section{Warm-up: concatenation}
\label{sec:concatenation}

Concatenation embeddings generally proceed by first mapping each token $x_t$ of $x$ to a vector $u(x_t,t)\in\Reals^d$. 
In a second step, these vector representations are concatenated together to form $\concat{x}$. 
We assume that the representation $u(x_t,t)$ can be written as 
\begin{equation}
\label{eq:def-concat-first-step}
u(x_t,t) = [u_e(x_t) ; u_p(t)]\in\Reals^d
\, ,
\end{equation}
where $u_e\in\Reals^{d_e}$ denotes vector representations of individual tokens, 
while $u_p\in\Reals^{d_p}$ encodes positional information, and we define $d\defeq d_e+d_p$. 

\paragraph{Token embeddings.}
As noted in the introduction, there are essentially two widespread choices for~$u_e$: either use \emph{sparse} representations for individual tokens or use \emph{dense} representations. 
The first approach is often synonymous with the use of \emph{one-hot encodings}, hence considering the mapping $u_e:j\mapsto \onehot{j}$ as a building brick, 
where, for any $j\in\Dico$, we define $\onehot{j}$ the $j$-th vector of the canonical basis of $\Reals^D$. 
This has the advantage of simplicity. 
One caveat is that, although sparse, one-hot vectors have dimensionality $d_e=D$---the size of the dictionary.
Regarding dense embeddings, as discussed in the introduction, the mapping $j\mapsto u_e(j)$ is learned from data and can encompass some semantic properties. 
In all these examples, $u_e(j)$ typically has dimensionality $d_e\ll D$ (for instance, \texttt{gensim} takes $d_e=100$ in its \texttt{word2vec} implementation).

\paragraph{Positional embeddings.}
A common choice is to learn positional embeddings, jointly with token embeddings. 
It is also possible to use deterministic positional embeddings, such as one-hot vectors --- $u_p(t)=\indic{t}\in\Reals^{\tmax}$, where $\tmax$ is a maximal document size, or more complicated functions of~$t$. 
For instance, the original transformers architecture uses a sinusoidal transformation of~$t$ as positional embedding \citep{vaswani_et_al_2017}. 
Further, it is also possible to incorporate additional positional information in the embedding -- for instance BERT incorporates segment position information corresponding to the index of the sentence the token belongs to \citep[Figure~2]{devlin_et_al_2018}. 
Finally, one can simply ignore $u_p$ altogether, relying simply on the order of the $u(x_t)$ to convey the positional information. 
Let us note that when $d_e=d_p$, one can add $u_e$ and $u_p$ in Eq.~\eqref{eq:def-concat-first-step} instead of concatenating them, a possibility to which our analysis is robust. 

\paragraph{Concatenation.}
For a given $u$, the embedding $\concat{x}$ of a document $x$ is formed by \emph{concatenating} the $u(x_t,t)$s for $t\in [T]$. 
Formally, if $T\geq \tmax$, then the concatenation $\concat{x}$ of $(x_1,\dots,x_{T})$ is defined as
\[
\concat{x} \defeq [u(x_1,1) ; \ldots ; u(x_{\tmax},\tmax)] \in \Reals^{d\tmax}
\, ,
\]
and if $T < \tmax$, as (\emph{zero-padding}),
\[
\concat{x} \defeq [u(x_1,1) ; \ldots ; u(x_T,T) ; 0 ; \ldots ; 0] \in \Reals^{d\tmax}
\, .
\]
Since the embedding is explicit in this case, it is straightforward to show the following:

\begin{proposition}[Robustness of concatenation]
\label{prop:concatenation-robustness}
Let $x \in \docs{T}$, $\Indset \subseteq [T]$, and $\xtilde \in \ballHamming{x}{\Indset}$.
Then 
\[
\norm{\concat{x} - \concat{\xtilde}} \leq \max_{j\neq k} \norm{u_e(j)-u_e(k)} \cdot \sqrt{\card{\Indset} \wedge \tmax}
\, .
\]
\end{proposition}

In particular, for small perturbation of the input document, \textbf{concatenation is $1/2$-H\"older with respect to the Hamming distance}. 
Closer inspection of the proof reveals that the constant depends only on the perturbed tokens: if the changes made are close from the point of view of $u_e$, then $\concat{x}$ and $\concat{\xtilde}$ remain close.


\section{TF-IDF transform}
\label{sec:tfidf}

Let $x$ be a document of length $T$ built on $\Dico$. 
In this section, we will assume that tokens correspond to individual words. 
Forgetting the sequential nature of natural language, one can simply look at the words appearing in $x$ with repetitions -- this is informally called a \emph{bag-of-words} representation.
Any given word $j\in\Dico$ appears in this representation with \emph{multiplicity} $\mult_j(x)$. 
The TF-IDF transform of $x$ is a vector $\tfidf{x}\in\Reals^D$, with each coordinate of $\tfidf{x}$ corresponding to a word of the dictionary. 
Component-wise, $\tfidf{x}$ is a product of two terms: the \emph{term frequency} $\freq_j$ and the \emph{inverse document frequency}~$\idf_j$:
\begin{equation}
\label{eq:def-tfidf}
\forall j\in\Dico , \quad 
\begin{cases}
\freq_j &\defeq \frac{\mult_j}{T}
\, ,  \\
\idf_j &\defeq \log \frac{\card{\Corpus}}{\card{\{z\in\Corpus \text{ s.t. } j\in z\}}}
\, , 
\end{cases}
\end{equation}
where $\Corpus$ is a set of documents. 
We will assume that $\idf_j >0$. 
The exact expressions appearing in Eq.~\eqref{eq:def-tfidf}  can vary depending on implementation, we use here the most common definitions (in particular, they are the default choices used by \texttt{scikit-learn} \citep{pedregosa_et_al_2011}). 
The (non-normalized) TF-IDF of $x$ can be written $\tfidf{x}_j=\freq_j\idf_j$ for all $j\in\Dico$. 
Intuitively, one wants to quantify the importance of each word in the document, while ignoring common words appearing in many documents such as articles. 
Finally, it is common to normalize $\tfidf{x}$, generally using the Euclidean norm. 
We denote by $\normtfidf{x}\defeq \tfidf{x}/\norm{\tfidf{x}}$ the normalized TF-IDF of $x$. 


\subsection{Robustness results}
\label{sec:tfidf-robustness}

As we saw in the previous section, the TF-IDF transform of a given document can be given \emph{in closed-form} as a function of the word multiplicities and the given coefficients. 
This allows a simple analysis, at least in the non-normalized case.

\begin{proposition}[Robustness of non-normalized TF-IDF]
\label{prop:robustness-non-normalized-tfidf}
Let $x \in \docs{T}$, $\Indset \subseteq [T]$, and $\xtilde \in \ballHamming{x}{\Indset}$.
Let $\maxmult$ be the maximal word multiplicity in $x$ and $\maxidf$ be the maximal inverse document frequency over $\Dico$. 
Then 
\[
\norm{\tfidf{x} - \tfidf{\xtilde}} \leq 4\maxmult \maxidf \frac{\card{\Indset}}{T}
\, .
\]
\end{proposition}

In other words, non-normalized \textbf{TF-IDF is Lipschitz continuous for the Hamming distance}, with Lipschitz constant inversely proportional to the common length of the documents. 
In reality, the dependency in $T$ is slightly more complicated since nothing prevents $\maxmult$ from being as large as $T$ in pathological cases (when all the words of the document are identical). 
In any case, we uncover a satisfying fact about TF-IDF: \textbf{small changes in long documents do not matter much.} 
Taking into account the normalization, we have a similar result:

\begin{proposition}[Robustness of normalized TF-IDF]
\label{prop:robustness-normalized-tfidf}
Let $x \in \docs{T}$.
Let $\minidf$ be the minimal inverse document frequency associated to the words of $x$. 
Let $\Indset \subseteq [T]$ such that $\card{\Indset}\leq \norm{\tfidf{x}}/(4\maxmult\maxidf)$ and $\xtilde \in \ballHamming{x}{\Indset}$. 
Then 
\[
\norm{\normtfidf{x} - \normtfidf{\xtilde}} \leq \frac{4\maxmult^{1/2}\maxidf^{1/2}D^{1/4}}{\minidf^{1/2}} \sqrt{\frac{\card{\Indset}}{T}}
\, .
\]
\end{proposition}

In plain words, \textbf{normalized TF-IDF is $1/2$-H\"older with respect to the Hamming distance}. 
Again, the constant appearing decreases with the length of the base document. 
A close inspection of the proof also reveals that the $D$ is actually equal to $D(x)$, the size of the local dictionary.


\begin{figure}[t]
\centering
\includegraphics[scale=0.3]{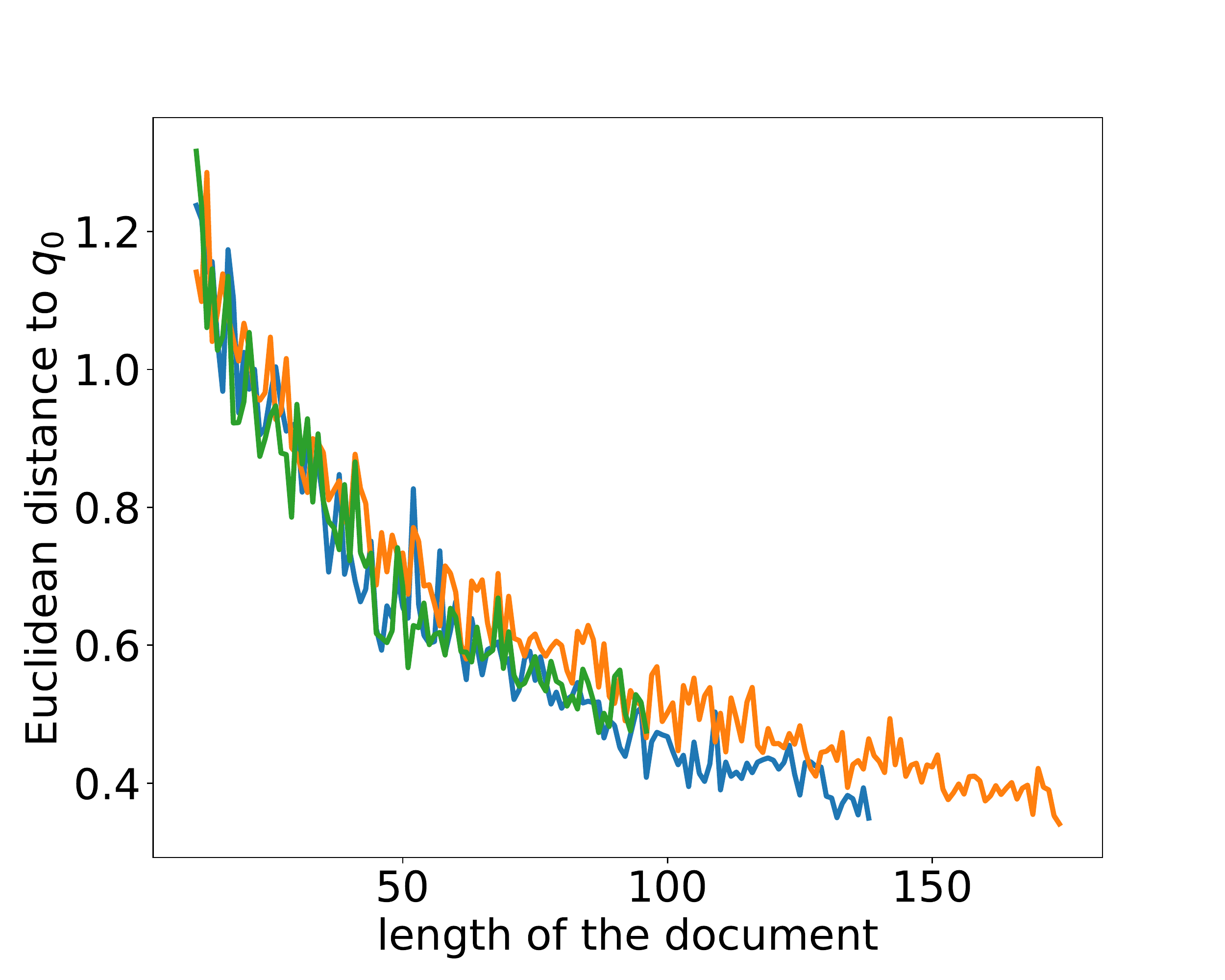}
\vspace{-0.2in}
\caption{\label{fig:influence-length-document-tfidf}Normalized TF-IDF, influence of $T$. Documents of increasing length $t$, $5$ random replacements. Proposition~\ref{prop:robustness-normalized-tfidf} gives a bound in $\mathcal{O}(1/\sqrt{T})$.} 
\end{figure}


\subsection{Experimental validation}

In order to check the accuracy of Proposition~\ref{prop:robustness-normalized-tfidf}, we ran some numerical experiments.
We considered movie reviews from the IMDB dataset as documents and the TF-IDF implementation from \texttt{scikit-learn} with $L^2$ normalization.

\paragraph{Influence of the document length.}
In a first set of experiments, we investigated the behavior of $\norm{\normtfidf{x}-\normtfidf{\xtilde}}$ with respect to the length $T$ of $x$. 
To this extent, for several documents, we created a sequence of growing documents by considering the first $t$ words of the documents, with $t$ ranging from $5$ to $T$. 
For each value of $t$, we replaced $5$ words in the intermediary document and repeated this experiment several time.
The words to replace were chosen uniformly at random in the document, and the replacements uniformly at random in $\Dico$, and we estimated the supremum of $\norm{\doctovec{x}-\doctovec{\xtilde}}$ by taking the maximum over these repetitions. 
Proposition~\ref{prop:robustness-normalized-tfidf} predicts that, since $\card{\Indset}$ is kept constant here, the supremum of $\norm{\doctovec{x}-\doctovec{\xtilde}}$ over all possible replacements should be upper bounded by $1/\sqrt{T}$ (up to numerical constants). 
This appears to be empirically true (see Figure~\ref{fig:influence-length-document-tfidf}). 

\paragraph{Influence of the number of removals.}
In a second set of experiments, we looked at the dependency of $\norm{\normtfidf{x}-\normtfidf{\xtilde}}$ with respect to $\card{\Indset}$. 
This time keeping $x$ fixed, we gradually increased the number of replaced words from $1$ to $T$. 
Since $T$ is fixed, Proposition~\ref{prop:robustness-normalized-tfidf} predicts that the supremum of $\norm{\normtfidf{x}-\normtfidf{\xtilde}}$ over all possible replacements should behave at most as $\sqrt{\card{\Indset}}$. 
This also appears to be empirically true, see Figure~\ref{fig:influence-number-replacements-tfidf}.


\begin{figure}[t]
\centering
\includegraphics[scale=0.3]{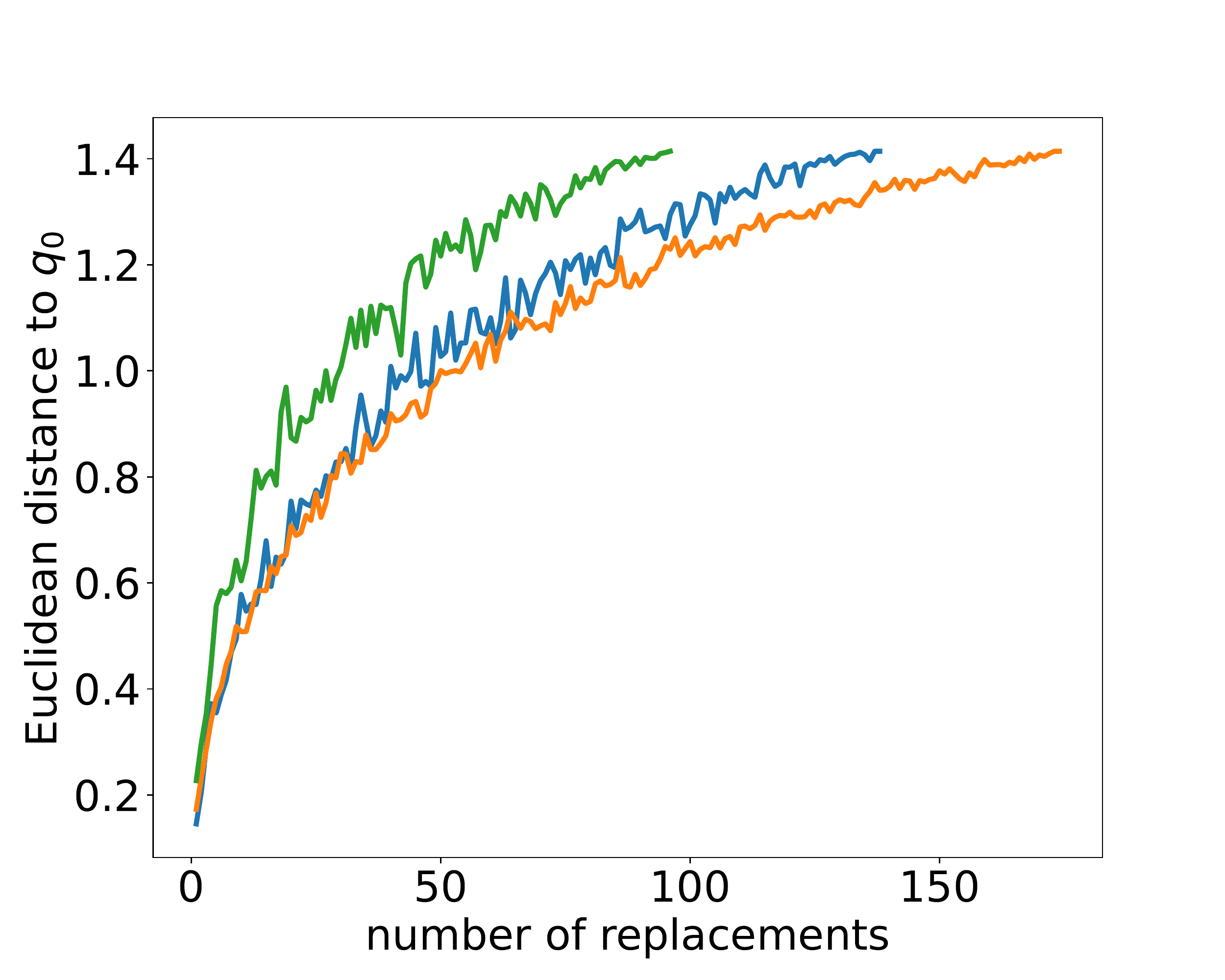}
\vspace{-0.2in}
\caption{\label{fig:influence-number-replacements-tfidf}Normalized TF-IDF, influence of $\card{S}$. For a given document, $s$ words are replaced at random with $s$ ranging from $1$ to $T$. Proposition~\ref{prop:robustness-normalized-tfidf} gives a bound in $\bigo{\sqrt{\card{\Indset}}}$.} 
\end{figure}


\section{Paragraph Vector (\texttt{doc2vec})}
\label{sec:doc2vec}

We now turn to the most challenging part of our analysis, \texttt{doc2vec}. 
On a high level, a token embedding matrix is learned jointly with a document embedding matrix on a corpus, aiming to predict correctly a missing token in a given context. 
The key difference with other vectorizers is that, at inference time, \textbf{another minimization problem is solved} by the model. 
Different documents yield different optimization problems, and therefore it is quite challenging to see where the resulting minimizer is located with respect to the original embedding. 


\subsection{A primer on \texttt{doc2vec}}
\label{sec:doc2vec-description}

The key idea underlying paragraph vector is neural probabilistic language modeling \citep{bengio_et_al_2000}: \textbf{predict words of a document} knowing (i) the \textbf{context} of the missing word in the document, and (ii) some \textbf{global information} about the document, \textbf{encoded as a vector} $q\in\Reals^d$.
Thus the key concept is the probability of observing word $j$ at position $t$ given some context $\context{t}$ and vector $q$. 
This is written informally as $\condproba{j}{\context{t},q}$, and we describe its exact formulation in the next paragraphs. 
Two models are proposed in \citet{le_mikolov_2014}: \emph{distributed memory} (PVDM) model, similar to the \emph{continuous bag of words} model of \citet{mikolov_et_al_2013}, and \emph{distributed bag of words} (PVDBOW) model, similar to the \emph{skip gram} model. 
We first focus on the PVDM model, PVDBOW being a simplified version thereof, referring to Figure~\ref{fig:doc2vec} for a visual help. 


\begin{figure*}[!th]
    \centering
    \includegraphics[width=.8\textwidth]{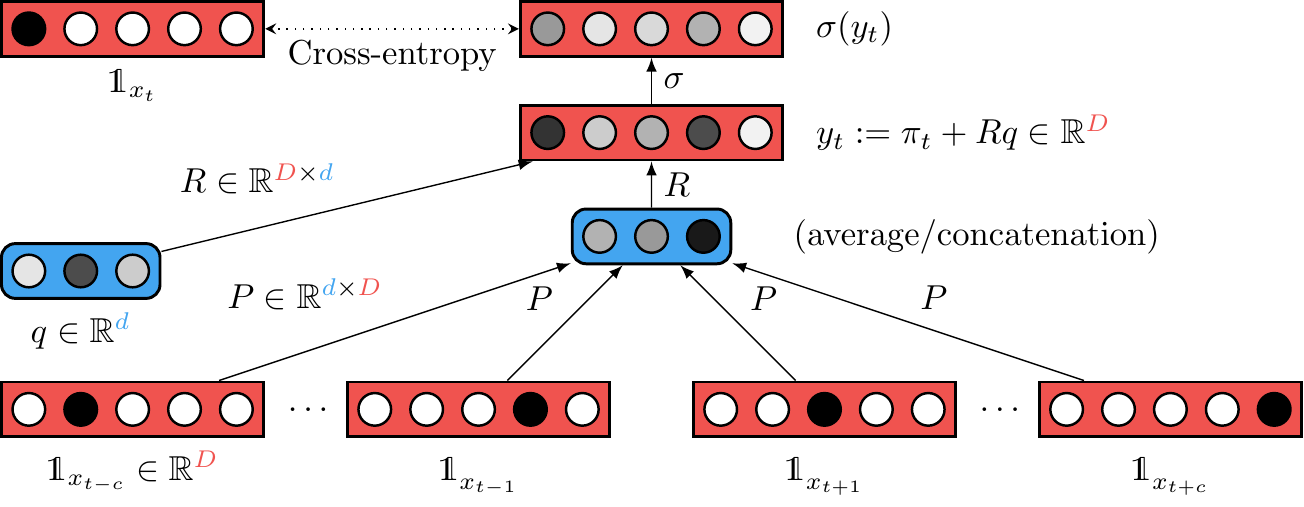}
    \vspace{-0.1in}
    \caption{\label{fig:doc2vec}Overview of the \texttt{doc2vec} vectorizer, PVDM model. For a given document, for each token position $t$, the model considers the context $\context{t}$ of $x_t$. The one-hot representation of the tokens in $\context{t}$ (which are of size $D$) are either average or concatenated, then projected to the embedding layer ($\Reals^d$, in blue). At this stage, the document embedding $q\in\Reals^d$ is added to this local representation, which is then lifted back to $y_t\in \Reals^D$. Taking a softmax transform of $y_t$ yields a discrete distribution on $\Dico$, which is compared to the truth ($x_t$) using cross-entropy (top part, dotted line). During training, PV minimizes objective \eqref{eq:PV-training} to find satisfying token embeddings $\Proj$, document embedding $q$, and lifting $\Red$. At inference time, $\Proj$ and $\Red$ are frozen and only $q$ is allowed to vary. 
    }
\end{figure*}


\paragraph{Local information.}
For a document $x$ with length $T$, for any  $\winsize < t < T-\winsize$, we define the \emph{neighborhood} of $t$ as 
\begin{equation}
\label{eq:def-neighborhood}
\nhood{t} \defeq (t-\winsize,\ldots,t-1,t+1,\ldots,t+\winsize)
\, .
\end{equation}
Here, $\winsize$ is an hyperparameter often called \emph{context size} (or window size), quantifying the breath of the context considered by the model. 
To this neighborhood corresponds the \emph{context} 
\begin{equation}
\label{eq:def-context}
\context{t} \defeq (x_{t-\winsize},\ldots,x_{t-1},x_{t+1},\ldots,x_{t+\winsize})
\, .
\end{equation}
Intuitively, $\context{t}$ corresponds to the tokens surrounding $x_t$ in the document $x$. 
The tokens contained in $\context{t}$ are then mapped to their one-hot representations, which are aggregated together. 
There are two natural ways to do this, either computing the \emph{mean} (PVDMmean) or the \emph{concatenation} of these vectors (PVDMconcat). 
Thus, at this stage, the local information at index $t$ is summarized as a vector~$h_t$, with 
\[
h_t\defeq  \frac{1}{2\winsize} \sum_{s\in \nhood{t}} \indic{x_s} \in\Reals^D
\]
if average is used, and 
\[
h_t \defeq  [\indic{x_{t-\winsize}};\ldots; \indic{x_{t-1}} ;\indic{x_{t+1}};\ldots;\indic{x_{t+\winsize}}] \in\Reals^{2\winsize D}
\]
if concatenation is used (see bottom layer of Figure~\ref{fig:doc2vec}). 

\paragraph{Projecting and lifting.}
This local information is then projected into $\Reals^d$, with $d\ll D$, the embedding space. 
At this stage, the document vector $q\in\Reals^d$ is added to the local representation. 
This intermediary representation is lifted back to $\Reals^D$. 
PVDM relies on two matrices $\Proj$ and $\Red$ such that each context is mapped to
\[
y_t \defeq \Red (\Proj h_t + q) = \pi_t + \Red q \in\Reals^D
\, ,
\]
where $\pi_t \defeq \Red \Proj h_t \in\Reals^D$. 
Here, $\Proj$ has size $d\times D$ for PVDMmean, and $d\times 2\winsize D$ for PVDMconcat, while $\Red$ has size $D\times d$. 
When tokens are words, the columns of $\Proj$ are called \emph{word vectors}, since they correspond to $d$ dimensional embeddings for individual words. 
We refer to the intermediate layers of Figure~\ref{fig:doc2vec} for a visual help. 

\paragraph{Prediction.}
Finally, the prediction for $x_t$ is encoded as the \emph{softmax} of $y_t$, where the softmax $\sigma : \Reals^D \to \Reals^D$ is defined for $u \in \Reals^D$ as
\begin{align} \label{eq:def-softmax}
\sigma(u) = \left( \frac{\exps{u_j}}{\sum_{k=1}^D \exps{u_k}}\right)_{1\leq j\leq D}
\, .
\end{align}
In particular, all components of $\sigma(y_t)$ lie between $0$ and $1$ and sum to one, and reading coordinate $j$ of $\sigma(y_t)$ can be interpreted as reading the predicted probability of token~$j$. 
To summarize, $\sigma(y_t)$ encodes a discrete distribution over $\Dico$ that depends on the context of $x_t$ and the document vector~$q$ (topmost layer of Figure~\ref{fig:doc2vec}).

\paragraph{Training.}
Let us call $x^{(1)},\ldots,x^{(N)}$ the documents in our training set, with lengths $T_1,\ldots,T_N$. 
To each of these documents correspond an embedding $q^{(i)}\in\Reals^d$, which can be seen as the columns of a matrix $\Token\in \Reals^{d\times N}$, each giving rise to $y_{t}^{(i)}$. 
The columns of $\Token$ are often referred to as \emph{document vectors}. 
The key idea here is to learn $\Proj,\Token,$ and $\Red$ so that the predicted tokens at position $t$ are accurate for all documents. 
Seeing $\sigma(y_t^{(i)})$ as a discrete probability distribution on $\Dico$, a natural way to compare it to the groundtruth ($x_t^{(i)}$) is to compute the \emph{cross-entropy} between the distribution putting mass one at $x_t^{(i)}$ and $\sigma(y_t^{(i)})$, that is, 
\[
\ell_t^{(i)} \defeq -\log \sigma(y_t^{(i)})_{x_t^{(i)}} \defeq \psi_{x_t^{(i)}}(y_t^{(i)})
\, ,
\]
where we defined $\psi\defeq -\log \sigma$ coordinate-wise. 
The optimization problem solved by PV is written 
\begin{equation}
\label{eq:PV-training}
\Minimize_{\Proj,\Token,\Red} \sum_{i=1}^N \frac{1}{T_i} \sum_{t\in x^{(i)}} \psi_{x_{t}^{(t)}}(y_{t}^{(i)})
\, ,
\end{equation}
 where $t\in x^{(i)}$ means $t$ ranging from $\winsize+1$ to $T_i-\winsize-1$.  
Problem~\eqref{eq:PV-training} is solved by stochastic gradient descent, or ADAM~\citep{kingma_ba_2015}.

\paragraph{Inference.}
Let us describe the embedding of a new document $x$, assuming that the model was trained on a corpus. 
The way inference works for the PV model is \textbf{to keep $\Proj$ and $\Red$ fixed}, and to optimize solely in $q\in\Reals^d$
\begin{equation}
\label{eq:PV-inference}
\Minimize_{q\in\Reals^d} \frac{1}{T}\sum_{t\in x} \psi_{x_t}(y_t)
\, .
\end{equation}
An important observation is that $q\mapsto \psi_{x_t}(\pi_t+\Red q)$ is a convex function, although not strictly (see Appendix). 
Therefore, a regularization term is often added to Eq.~\eqref{eq:PV-inference}, a point which we will clarify in the next section. 
Also noting that $q$ has only $d$ parameters, solving PV inference \eqref{eq:PV-inference} efficiently is not too challenging.

\paragraph{The case of PVDBOW.}
PVDBOW is another model falling under the PV umbrella. 
In a nutshell, following the idea of the distributed bag of word model, PVDBOW works the other way around and uses only the representation of the document to predict tokens. 
At position $t$, no local information is taken into account and we put $\pi_t=0$ in that case. 
The predicted token distribution for the document is encoded as before (as $\sigma(y_t)=\sigma(\Red q)$), and its quality  also measured as $\psi_{x_t}(y_t)$ for all tokens in the document, leading to the same optimization problems. 
To summarize, PVDBOW is a simplified, lightweight version of PVDM, simply obtained by taking $\pi_t=0$ in our framework. 
In particular, there is no matrix $\Proj$, which leads to fewer parameters, and thus easier training and inference, a fact which was pointed out by \citet{le_mikolov_2014}. 
Nevertheless, they recognize that PVDBOW still performs well as an embedding, and recommend considering as an embedding the concatenation of PVDM and PVDBOW. 

\paragraph{Hierarchical softmax and negative sampling.}
In practice, as advocated by \citet{le_mikolov_2014}, two additional expedients are used. 
First, the softmax is replaced by \emph{hierarchical softmax} \citep{morin_bengio_2005}. 
In a nutshell, each call of $\sigma$ has a computational cost linear in $D$, which can be as large as $10^5$ in practice. 
A solution is to replace the softmax by a tree-based approximation thereof, which computation is much faster. 
Second, following \citet{mikolov_et_al_2013}, it is common to incorporate tokens with a negative association to the token to predict when computing $\ell_t$, leading to faster training. 
These two possibilities are non-trivial modifications to the PV model and we do not consider them in our analysis. 


\subsection{Robustness result}
\label{sec:doc2vec-robustness}

Before stating our robustness result, let us explain why it is challenging and outline the proof technique. 
As detailed in the previous section, the embedding of a document $x$ of length $T$ is found by solving 
\begin{equation}
\label{eq:optim-problem}
q_0 = 
\Argmin_{q\in \Reals^d} \left\{ F(q)  + \frac{\alpha}{2} \norm{q}^2
\right\}
\, ,
\end{equation}
where $F(q)\defeq \frac{1}{T}\sum_{t\in x} \psi_{x_t}(\pi_t + \Red q)$. 
The regularization term  $\alpha\norm{q}^2/2$ with $\alpha >0$ ensures uniqueness of the solution. 
Indeed, the softmax is invariant by translation by a vector proportional to $\Indic$, and solutions to \eqref{eq:PV-inference} are not unique. 
As before, consider $\xtilde$, a modified version of $x$ where tokens with indices in $\Indset$ have been replaced by others.
The embedding $q_1$ of $\xtilde$ is found by solving
\begin{equation}
\label{eq:optim.problem.modified-appendix}
q_1 = \Argmin_{t\in x} \left\{ G(q) +\frac{\alpha}{2}\norm{q}^2 \right\}
\, ,
\end{equation}
where $G(q)\defeq \frac{1}{T}\sum_{t\in x} \psi_{\xtilde_t}(\pitilde_t + \Red q)$, and $\pitilde_t$ is defined analogously to $\pi_t$. 
The main challenge here is that \textbf{$q_0$ and $q_1$ are solutions of distinct optimization problems, which can be quite different if $\card{\Indset}$ is large}. 

\paragraph{From discrete to continuous.}
The solution we propose to connect between these two problems is to interpolate smoothly between them. 
There are many ways to do this, and we settle for the simplest: linear interpolation. 
More precisely, we define for all $\temps \in [0,1]$ and $q\in \Reals^d$ by 
\begin{equation}
\label{eq:def-interpolation-scheme}
\Psilinear(\temps,q) \defeq (1-\temps) F(q) + \temps G(q)
\, .
\end{equation}
Subsequently, for all $\temps \in [0,1]$, we can solve the following regularized optimization problem:
\begin{equation}
\label{eq:ODE-papereq:optim.problem.general}
q(\temps) \defeq \Argmin_{q\in\Reals^d} \left\{ \Psilinear(\temps,q) + \frac{\alpha}{2}\norm{q}^2 \right\}
\, ,
\end{equation}
giving rise to a continuous trajectory in the embedding space (see Figure~\ref{fig:interpolation} for an illustration). 
One can think of $q(\temps)$ as the embedding of a fictitious document traveling halfway between $x$ and $\xtilde$ as $\temps$ ranges from $0$ to $1$. 


\begin{figure}
    \centering
    \includegraphics[scale=0.3]{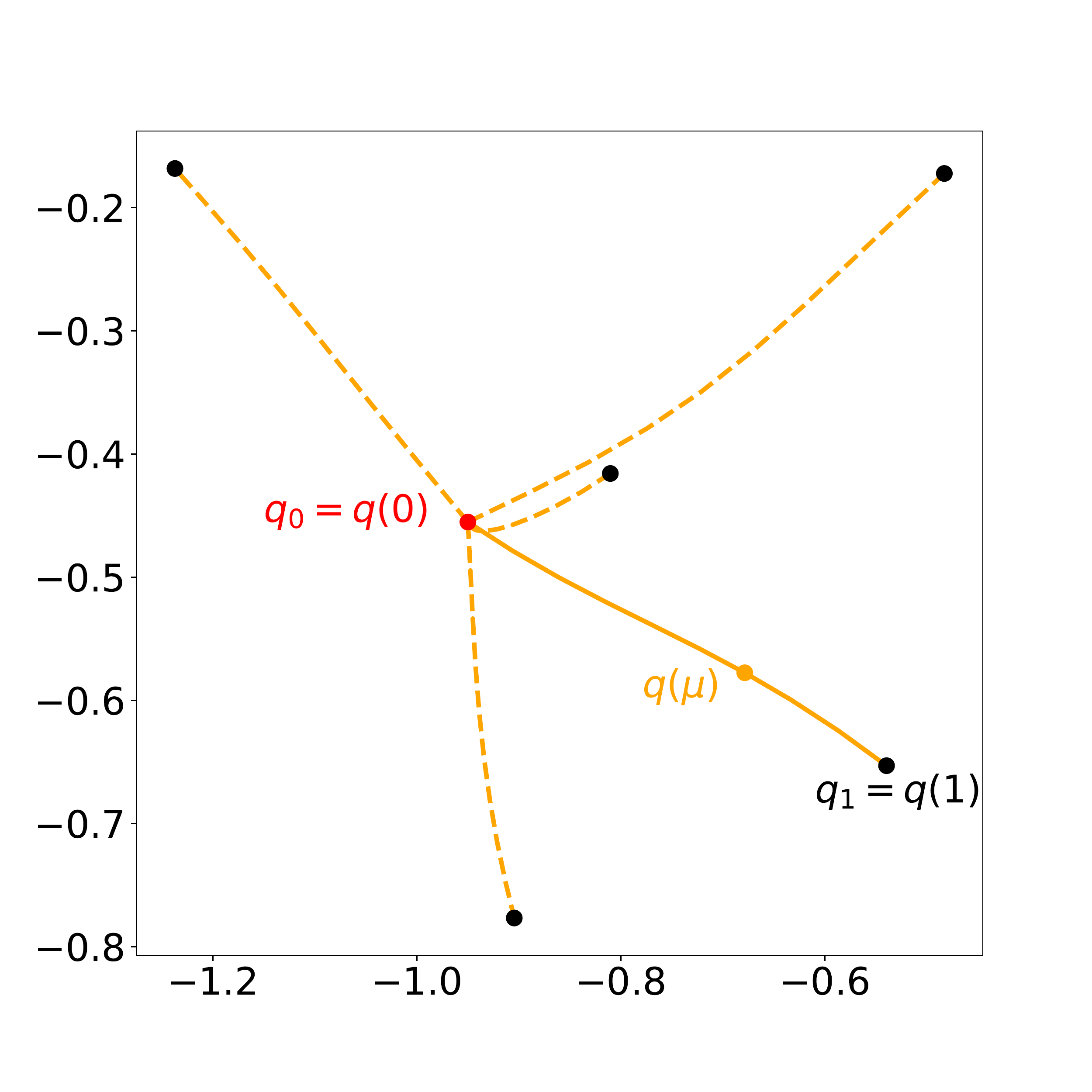}
    \vspace{-0.2in}
    \caption{\label{fig:interpolation}Continuously interpolating between $q_0$, the embedding of $x$ (in red), and $q_1$, the embedding of $\xtilde$ (in black). Visualization in a $2$-dimensional slice of $\Reals^d$. To each $\temps\in [0,1]$ corresponds a solution to \eqref{eq:ODE-papereq:optim.problem.general}, appearing here as a point of the trajectory between $q_0$ and $q_1$ (solid orange line). 
    Dynamics of this trajectory are described by Eq.~\eqref{eq:ODE-paper}. Different document perturbations lead to different embeddings and associated trajectories (dotted lines).}
\end{figure}


\paragraph{Dynamics of interpolation.}
This approach is powerful, since it allows us to transform a problem which is discrete in nature (elements of a sum are modified) to a continuous one (time parameter varies). 
In particular, the dynamics of $\temps\mapsto q(\temps)$ \textbf{are described by an ordinary differential equation (ODE)}. 
Indeed, for all $\temps \in [0,1]$, since $q\mapsto \Psilinear(\temps,q) +\frac{\alpha}{2}\norm{q}^2 $ is a strongly convex function, $q(\temps)$ is the (unique) critical point of
$q\to \nabla\Psilinear(\temps,q) + \alpha q$, where $\nabla$ denotes derivative with respect to the \emph{space} coordinate ($q$). 
That is, for all $\temps \in[0,1]$,
\begin{equation*}
\nabla \Psilinear(\temps,q(\temps) ) + \alpha q(\temps) = 0
\, .
\end{equation*}
Differentiating, we get that for all $\temps \in [0,1]$, 
\begin{equation}
\label{eq:ODE-paper}
\left(\nabla^2 \Psilinear(\temps,q(\temps) ) + \alpha \Identity \right)
q'(\temps) + 
\partial_\temps \nabla \Psilinear(\temps,q(\temps)) = 0
\, ,
\end{equation}
where $g'$ denotes derivative with respect to the \emph{time} coordinate ($\temps$) and $\Identity$ the identity matrix. 
Let us set 
\begin{equation}
\label{eq:def-philinear}
\Philinear(\temps,q) \defeq - \left( \nabla^2 \Psilinear(\temps,q) + \alpha \Identity\right)^{-1} \partial_\temps \nabla \Psilinear(\temps,q)
\, .
\end{equation}
Then, Eq.~\eqref{eq:ODE-paper} can be rewritten as $q'(\temps)=\Philinear(q(\temps),\temps)$. 

\paragraph{Spectrum of the Hessian of the log-softmax. }
Looking back at the ODE problem, it appears that one needs to understand precisely the behavior of $\Philinear$. 
Intuitively, an ill-behaved function could lead to the explosion of the solution of the ODE, preventing the existence of reasonable bounds on $\norm{q(\temps)-q(0)}$ for large $\mu$. 
This understanding relies on the control of the smallest positive eigenvalue of $\nabla^2\Psilinear$, $\lambda_1(\temps,q)$.   
Coming back to the definition of $\Psilinear$ (Eq.~\eqref{eq:def-interpolation-scheme}), $F$, and $G$, we see that $\lambda_1$ closely related to $\lambdamin$, the smallest positive eigenvalue of the Hessian of the log-softmax, for which we have precise results (Lemma~\ref{lemma:smallest-eigenvalue-appendix} and Theorem~\ref{theorem:min_softmax-appendix}). 

\paragraph{Gr\"onwall-type result.}
Once that a precise control is achieved on $\Philinear$, one may have hoped to use standard Gr\"onwall type inequalities such as \citet{pachpatteJournalInequalitiesPure2004} to obtain quantitative bounds on $\norm{q(1)-q(0)}$. 
However, in our setting, the growth of $\Philinear$ prevents us from getting explicit bounds and we had to prove a new result (Theorem~\ref{lemma:CS-appendix}) which is actually true in a more general setting than that of \texttt{doc2vec}. 
Specifying this result, we get:


\begin{figure*}[ht]
    \centering
    \includegraphics[scale=0.22]{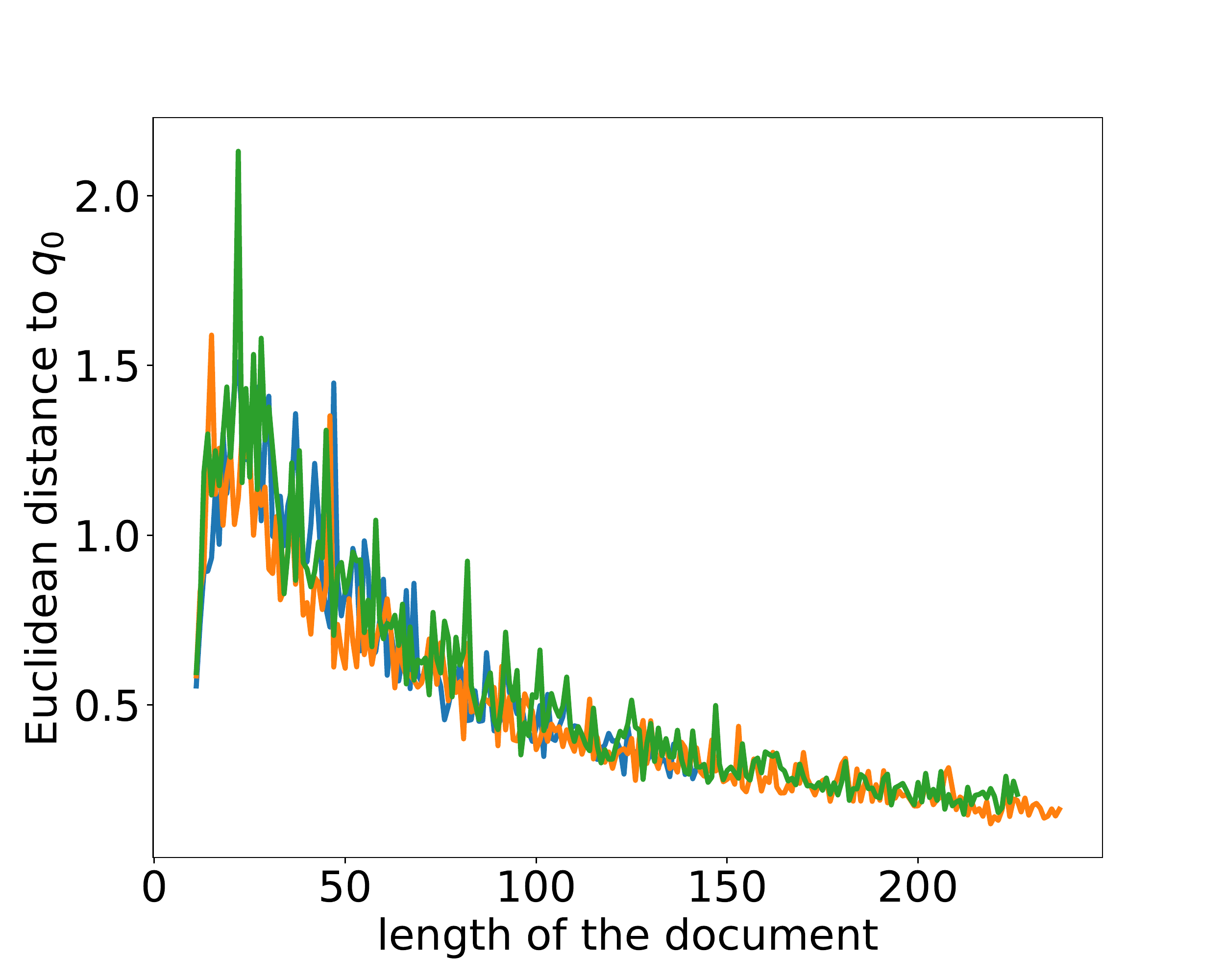}
    \hfill
    \includegraphics[scale=0.22]{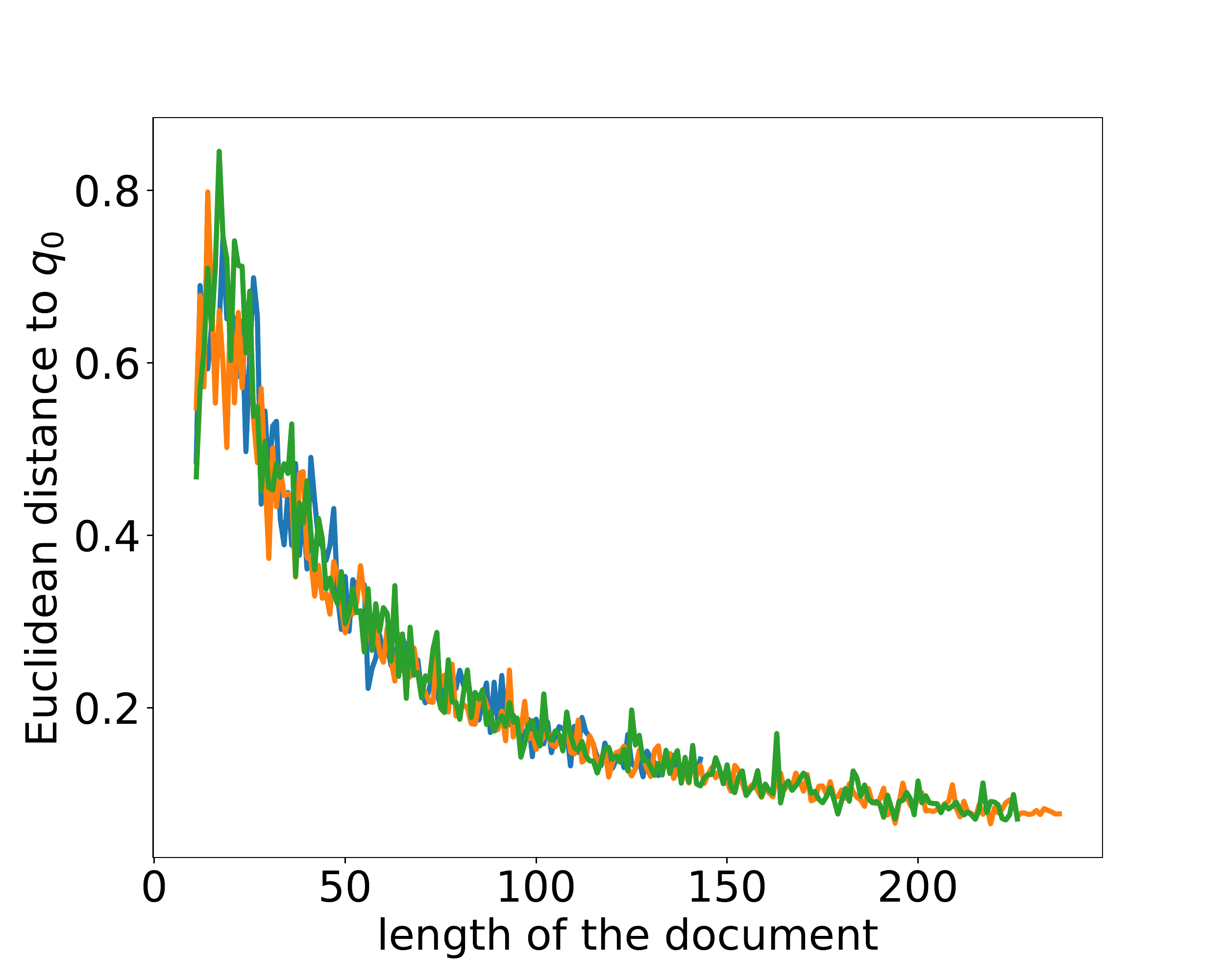}
 \hfill 
        \includegraphics[scale=0.22]{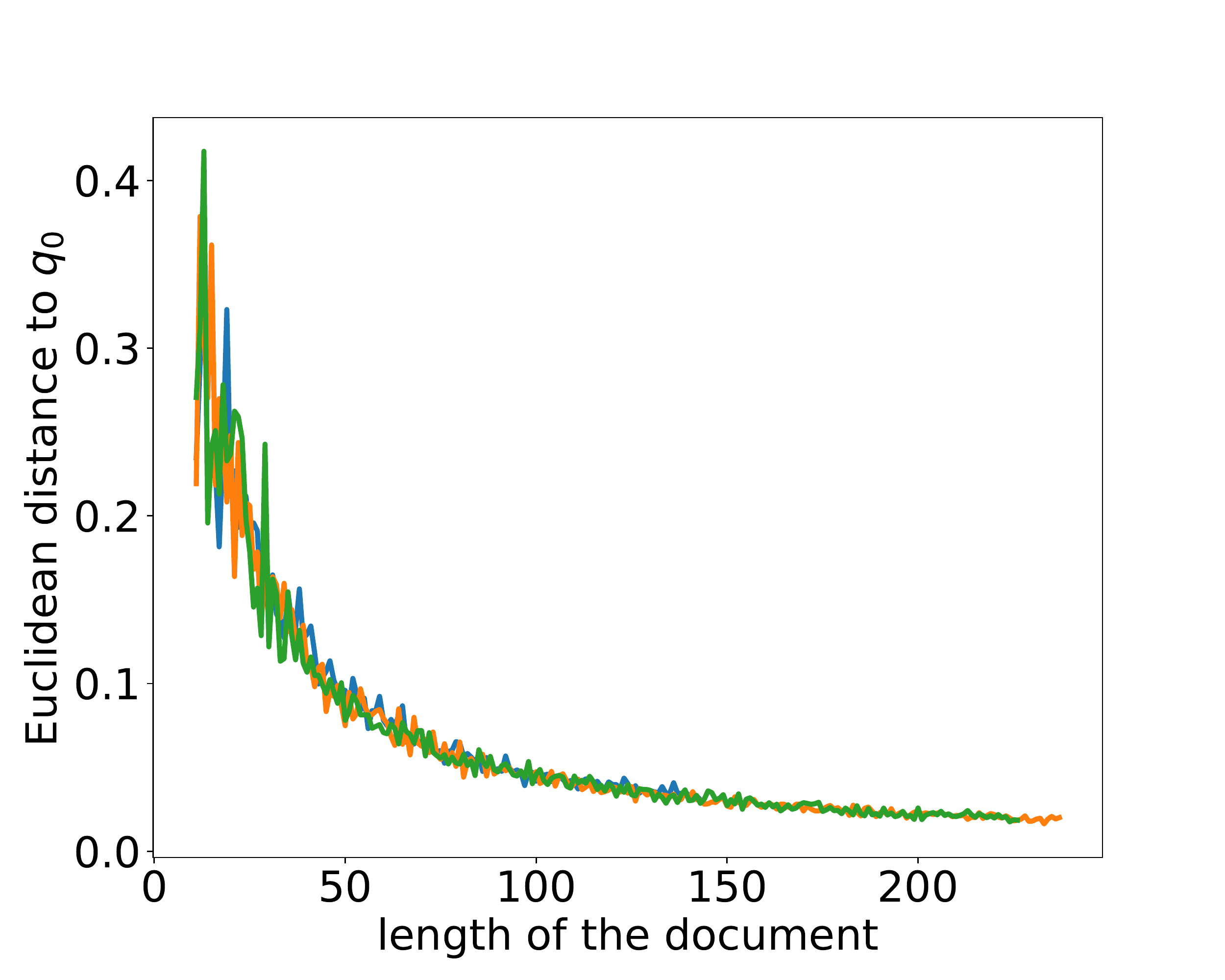}
\vspace{-0.2in}
    \caption{\label{fig:influence-length-document-perso}Influence of the length of the document on the robustness of \texttt{doc2vec}. Five random replacements, from left to right: PVDMmean, PVDMconcat, and PVDBOW.}
\end{figure*}

\begin{figure*}[!ht]
    \centering
\includegraphics[scale=0.22]{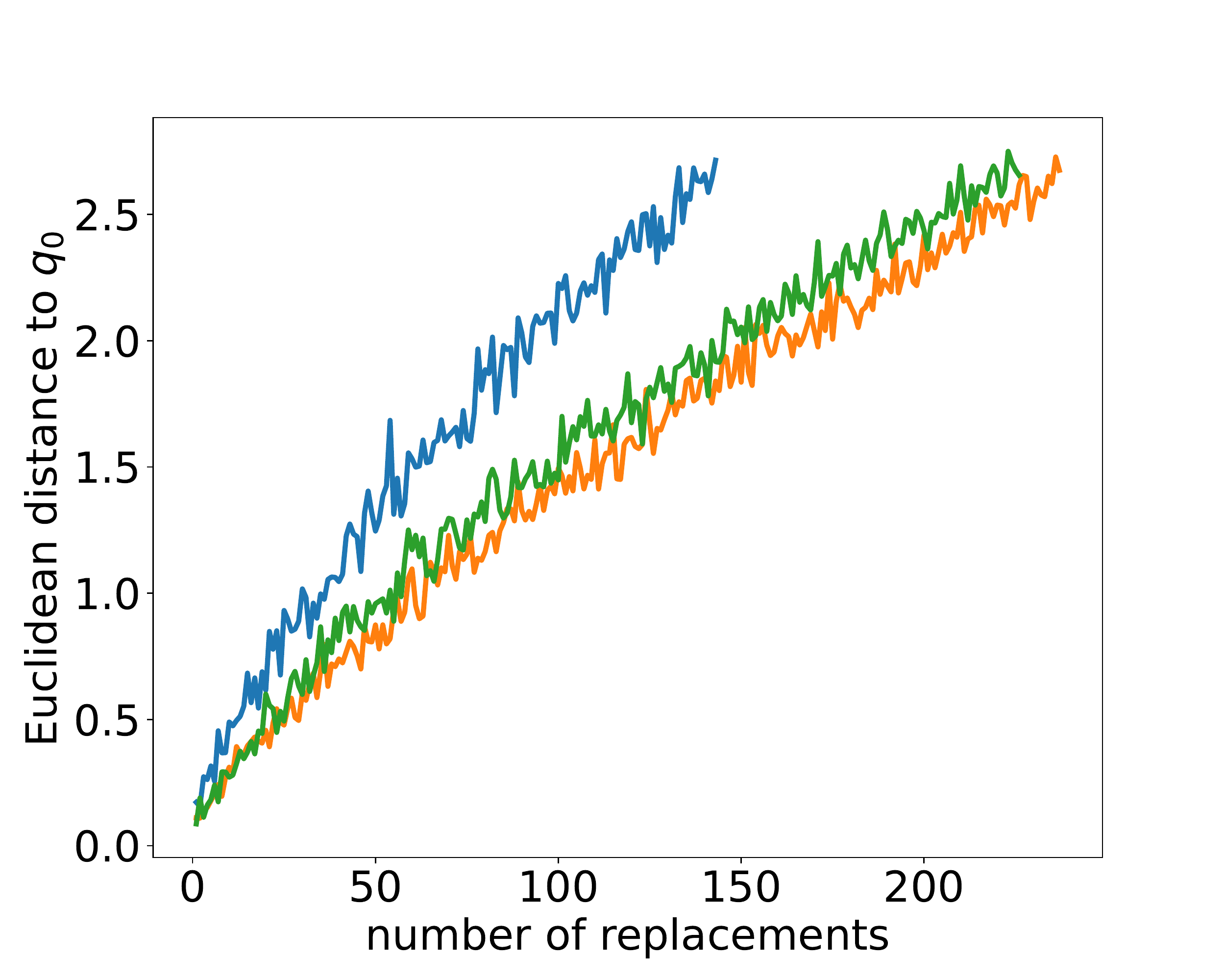}
\hfill
\includegraphics[scale=0.22]{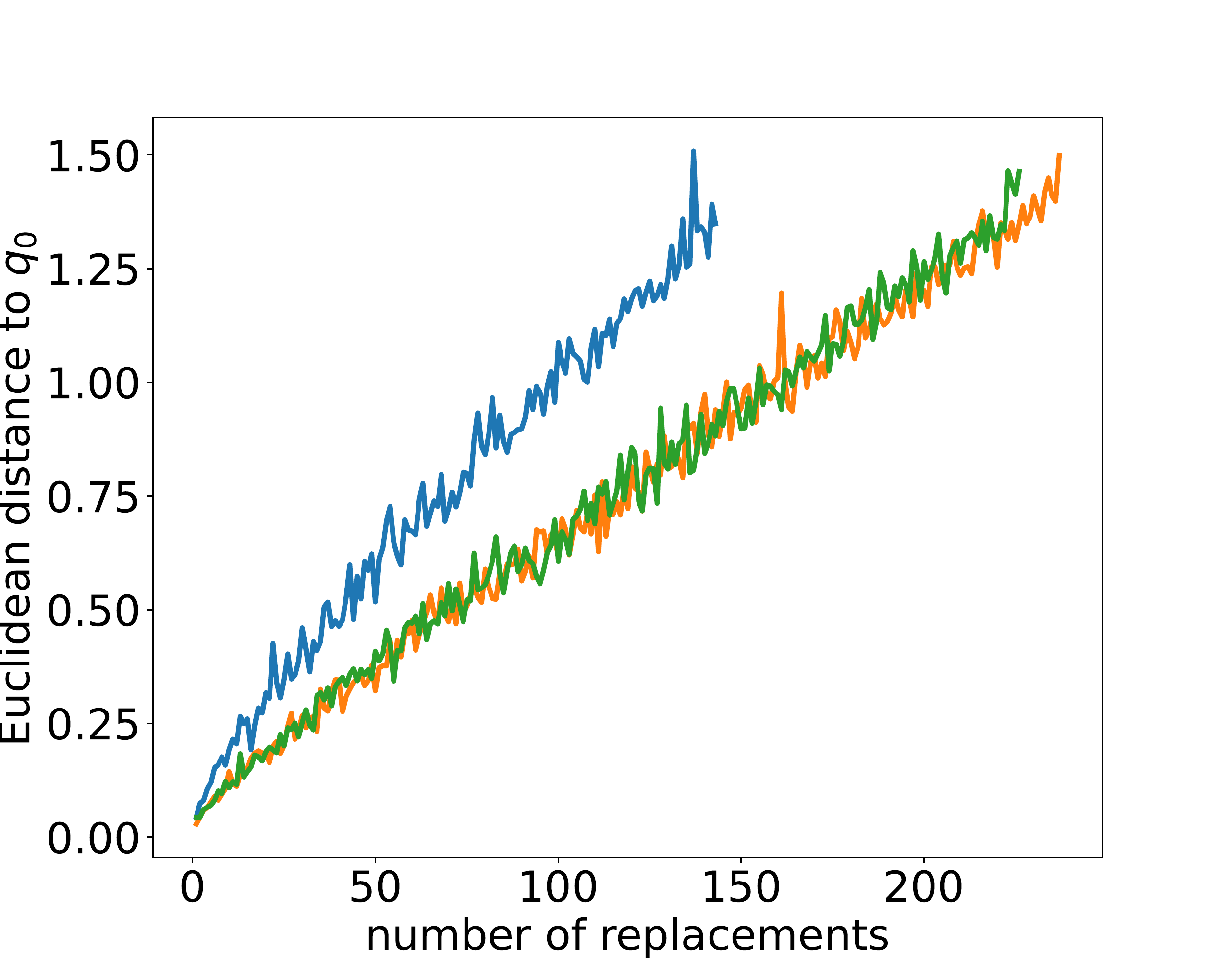}
\hfill 
\includegraphics[scale=0.22]{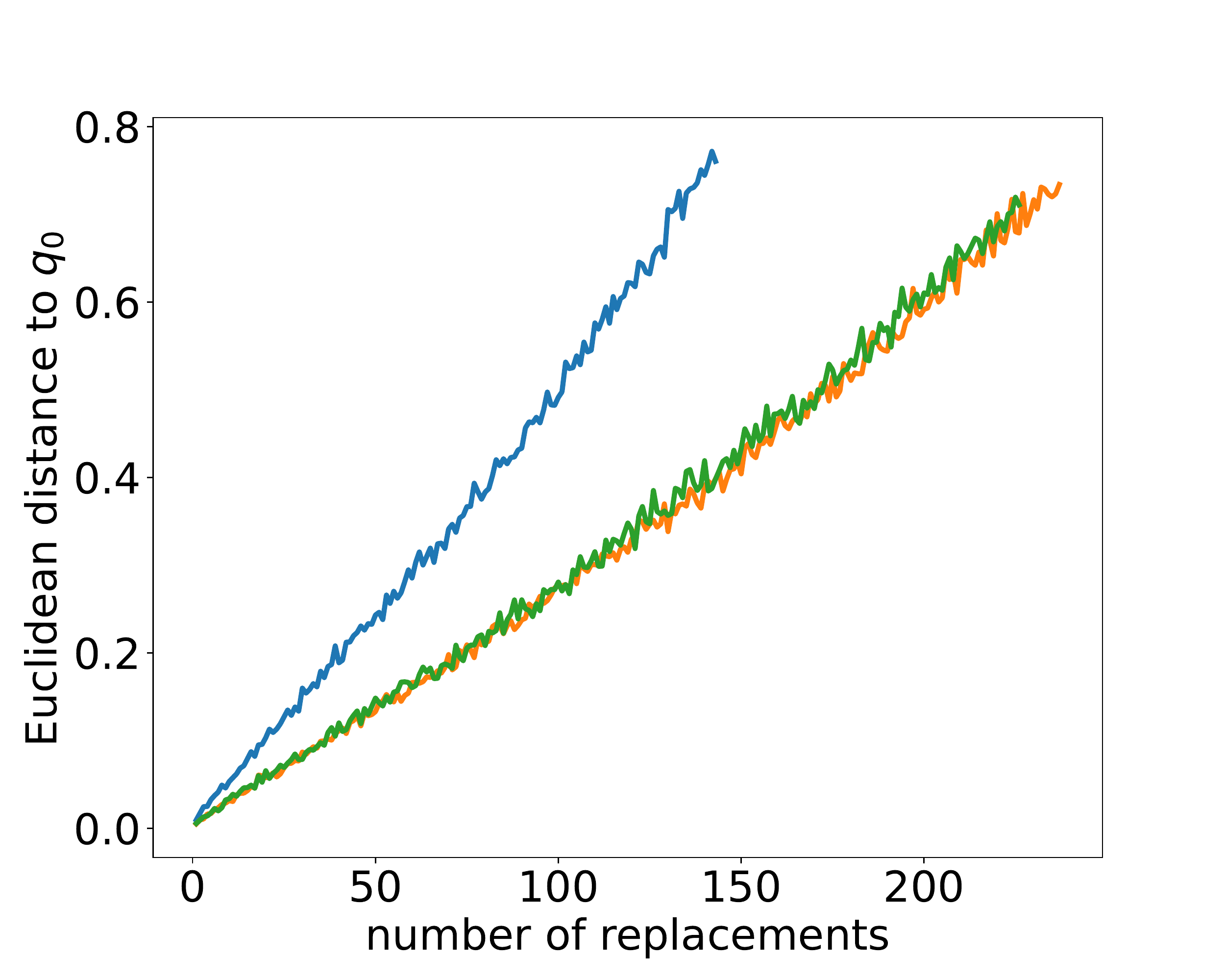}
\vspace{-0.2in}
    \caption{\label{fig:influence-number-replaced-perso}Influence of the number of words replaced on the robustness of \texttt{doc2vec}. From left to right: PVDMmean, PVDMconcat, and PVDBOW.}
\end{figure*}


\begin{theorem}[Bounded trajectories]
\label{th:bounded-traj}
Let $x \in \docs{T}$, $\Indset \subseteq [T]$, and $\xtilde \in \ballHamming{x}{\Indset}$. 
Suppose that $\Red \in \Reals^{D\times d}$ is 
such that $\minsingular{R} >0$ and $\range{R} \subset \Indic^{\perp}$. 
Let $\mu\mapsto q(\mu)$ be the solution of ODE \eqref{eq:ODE-paper}. 
Then, there exist two constants $c=c(\alpha)>0$ and $L=L(\norm{q(0)})>0$ depending explicitly on $\Proj,\Red$, $\winsize$, and $D$ such that, whenever $\card{\Indset}/T\leq c$, 
\[
\sup_{\temps\in [0,1]} \norm{q(\mu) - q(0)} \leq L \frac{\card{\Indset}}{T}
\, .
\]
\end{theorem}

Since $\doctovec{x}=q(0)$ and $\doctovec{\xtilde}=q(1)$, a corollary of Theorem~\ref{th:bounded-traj} is that the \textbf{\texttt{doc2vec} embedding is Lipschitz continuous with respect to the Hamming distance}, with Lipschitz constant at most inversely proportional to the document lengTheorem 
Coming back to our initial question, Theorem~\ref{th:bounded-traj} guarantees that, for documents of reasonable length and small perturbations, \texttt{doc2vec} embeddings can not vary too greatly. 
We emphasize that Theorem~\ref{th:bounded-traj} is true for all three \texttt{doc2vec} models. 

The key assumption here is that $\card{\Indset}$ is small enough. 
We argue that it is only natural to ask so: indeed, if one is allowed to modify every single token of $x$, this yield a completely different document (although having the same length), which could \emph{a priori} be embedded anywhere. 
The other main assumptions concern the matrix $\Red$. 
Experimentally, we observe that $\minsingular{\Red}>0$ holds (see Section~\ref{sec:svd-appendix}). 
Requiring that $\range{\Red} \subset \Indic^{\perp}$ is not too restricting: because of the translation invariance by $\Indic$ of the softmax, one can always normalize $\Red$ by removing the average line from each line. 
The main limitation of Theorem~\ref{th:bounded-traj} is the dependency of $c$ and $L$ in the problems parameters. 
Exact expression can be found in Appendix (Theorem~\ref{theorem:gronwall_doc2vec-appendix}). 


\subsection{Experimental validation}

In order to verify the validity of Theorem~\ref{th:bounded-traj}, we ran similar experiments to those presented in Section~\ref{sec:tfidf}. 
We considered again movie reviews from the IMDB dataset. 
As vectorizer, we trained \texttt{doc2vec} models from scratch on a subset of the IMDB dataset ($10^3$ reviews). 
The associated dictionary has size $D=18,416$: we took tokens as words of the English dictionary. 
Note that one can also consider sub-word tokens, but in that case replacing a word in the document usually implies replacing several tokens. 
We chose $d=50$ as dimension of the embedding. 
We took $\winsize=5$ as context size parameter. 

We present results of experiments regarding the influence of the document length in Figure~\ref{fig:influence-length-document-perso}. 
Theorem~\ref{th:bounded-traj} predicts that, since $\card{\Indset}$ is kept constant here, the supremum of $\norm{\doctovec{x}-\doctovec{\xtilde}}$ over all replacements should be upper bounded by $1/T$ (up to numerical constants). 
This appears to be empirically true. 

We present results of experiments regarding the influence of the number of replaced words in Figure~\ref{fig:influence-number-replaced-perso}. 
Here we took the number of replaced words from $\winsize + 1$ to $T-\winsize-1$ to avoid for border effects. 
Since $T$ is fixed, Theorem~\ref{th:bounded-traj} predicts that the supremum of $\norm{\doctovec{x}-\doctovec{\xtilde}}$ over all possible replacements should behave at most linearly in $\card{\Indset}$. 
This appears to be empirically true.

We present in Appendix (Section~\ref{sec:more-results-appendix}) additional results with another implementation, \texttt{gensim} \citep{rehurek_sojka_2010}. 
In particular, this implementation uses hierarchical softmax. 
The results are consistent with the behavior presented here.


\section{Conclusion}
\label{sec:conclusion}

In this paper, we proved that several popular text vectorizers are robust, in the sense that they are either Lipschitz or H\"older continuous with respect to the Hamming distance.
Proving this robustness was possible for concatenation and TF-IDF thanks to elementary computations, but required a much more challenging mathematical analysis for \texttt{doc2vec} requiring  two new results (local Lipschitz continuity of the softmax and a new Gr\"onwall--Bellman--Bahouri non-explosion lemma).

Let us outline future research directions. 
First, we studied the robustness of the \emph{true} solution of~\eqref{eq:PV-training} and~\eqref{eq:PV-inference}. 
In practice, this problem is solved thanks to gradient descent, and it would be interesting to measure the impact of this approximation.
A second line of work would consist in obtaining refined results when we put a random model on the distribution of the words of the document, similarly to what is done in \citep{arora_et_al_2016}. 


\section*{Acknowledgements}

This work was supported by the French government under the management of the Agence Nationale de la Recherche, grant agreements GraVa ANR-18-CE40-0005, NEMATIC ANR-21-CE45-0010, and NIM-ML ANR-21-CE23-0005-01. 
D.G. acknowledges the support of EU Horizon 2020 project AI4Media (contract no. 951911) and would like to thank Charbel Yachouchi for his preliminary work.
S.V. would like to thank Nicolas Patry for fruitful discussion about NLP embeddings.


\bibliography{biblio}
\bibliographystyle{icml2023}

\newpage
\appendix
\onecolumn

\section{General organization}
\label{sec:orga-appendix}

This Appendix is organized as follows: in Section~\ref{sec:concatenation-appendix} (resp.~\ref{sec:tfidf-appendix}) we collect the missing proofs for Section~\ref{sec:concatenation} (resp.~\ref{sec:tfidf}) of the main paper. 

The next five sections are dedicated to the proof of Theorem~\ref{th:bounded-traj}: 
First, in Section~\ref{sec:minimizarion.ode-appendix}, we formally prove that the dynamics of the interpolation scheme between two minimizers follow an ordinary differential equation (ODE). 
We actually show a more general result and provide technical conditions on the interpolation $\Psi$ under which we are able to formulate the interpolation between minimization problems as an ODE. 
Next, in Section~\ref{sec:quantitative-bounds-appendix}, we derive quantitative bounds for the solution of this ODE. 
We show how to specialize this result in the \texttt{doc2vec} setting in Section~\ref{sec:doc2vec-appendix}, proving Theorem~\ref{th:bounded-traj} in the process. 
The main tool used to obtain these bounds is a general Gr\"onwall-Bellman-Bahouri type result for ODE with exponentially-growing coefficients.
This result (Theorem~\ref{lemma:CS-appendix}), as well as all other technical results concerning ODEs, is stated and proved in Section~\ref{sec:gronwall-appendix} .
In order to specialize our result to the \texttt{doc2vec} setting, we needed a fine-grained study of the (log-)softmax function. 
In particular, we derive a new bound on the softmax function (Theorem~\ref{theorem:min_softmax-appendix}), which is proved in Section~\ref{sec:softmax-appendix}. 

We conclude this Appendix with additional experimental results supporting our claims in Section~\ref{sec:more-results-appendix}. 


\section{Omitted proofs for concatenation}
\label{sec:concatenation-appendix}

\subsection{Proof of Proposition~\ref{prop:concatenation-robustness}}

By definition of $\Concat$ and Pythagoras theorem, 
\[
\norm{\concat{x} - \concat{\xtilde}}^2 = \sum_{t\in \Indset \cap [\tmax]} \norm{u(x_t,t)-u(\xtilde_t,t)}^2
\, .
\]
By definition of $u$ (Eq.~\eqref{eq:def-concat-first-step}), one has 
\begin{equation}
\label{eq:aux-proof-concat-1}
u(x_t,t)-u(\xtilde_t,t) = [u_e(x_t)-u(\xtilde_t) ; 0]
\, ,
\end{equation}
and therefore 
\[
\norm{u(x_t,t)-u(\xtilde_t,t)}^2 = \norm{u_e(x_t)-u(\xtilde_t)}^2
\, .
\]
We deduce that 
\[
\norm{\concat{x} - \concat{\xtilde}}^2 \leq \card{\Indset \cap [\tmax]} \cdot \max_{j\neq k} \norm{u_e(j)-u_e(k)}^2
\, .
\]
\qed 

\begin{remark}[Concatenation \emph{v.s.} sum]
Replacing the concatenation by a sum in the definition of $u$ (Eq.~\eqref{eq:def-concat-first-step}) does not change the proof. 
Indeed, the key step Eq.~\eqref{eq:aux-proof-concat-1} remains unchanged in that case: the key idea here is that position tokens are the same for words in the same position, and cancel out when forming the difference. 
\end{remark}


\section{Omitted proofs for TF-IDF vectorization}
\label{sec:tfidf-appendix}


\subsection{Proof of Proposition~\ref{prop:robustness-non-normalized-tfidf}}

By definition, we can write 
\[
\tfidf{x} = \sum_{j=1}^D \freq_j \idf_j\onehot{j} = \frac{1}{T}\sum_{j=1}^D \mult_j\idf_j\onehot{j}
\, .
\]
Similarly, since $\xtilde$ has same length as $x$,
\[
\tfidf{\xtilde} = \frac{1}{T}\sum_{j=1}^D \mtilde_j\idf_j\onehot{j}
\, ,
\]
where we let $\mtilde_j$ denote the multiplicity of word $j$ in document $\xtilde$. 
We deduce that 
\[
\norm{\tfidf{x}-\tfidf{\xtilde}}^2 = \frac{1}{T^2}\sum_{j=1}^D (\mult_j-\mtilde_j)^2\idf_j^2
\, .
\]
By letting $\maxidf$ be the maximal inverse document frequency on $\Dico$, we already see that
\[
\norm{\tfidf{x}-\tfidf{\xtilde}}^2 \leq \frac{\maxidf^2}{T^2}\sum_{j=1}^D (\mult_j-\mtilde_j)^2
\, .
\]
In the previous display, only terms such that $\mult_j\neq \mtilde_j$ count. 
Using the inequality between $p$-norms, we have
\[
\sum_{\mult_j\neq \mtilde_j}(\mult_j-\mtilde_j)^2 \leq \left(\sum_{\mult_j\neq \mtilde_j} \abs{\mult_j-\mtilde_j}\right)^2
\, .
\]
Now, by the triangle inequality, 
\[
\sum_{\mult_j\neq \mtilde_j} \abs{\mult_j-\mtilde_j} \leq \sum_{\mult_j\neq \mtilde_j} \mult_j + \sum_{\mult_j\neq \mtilde_j} \mtilde_j
\, .
\]
We notice that these two sums are equal: every removed word has to appear somewhere. 
Moreover, $\card{\{j \text{ s.t. } \mult_j\neq \mtilde_j}\leq 2\card{\Indset}$, since modifying one word changes at most two multiplicities, and this happens at most $\card{\Indset}$ times. 
Therefore, we have proved that 
\begin{equation}
\label{eq:aux-proof-robustness-non-normalized-tfidf}
\sum_{\mult_j\neq \mtilde_j} \abs{\mult_j-\mtilde_j} \leq 4\maxmult \card{\Indset} 
\, ,
\end{equation}
where we recall that $\maxmult$ is the maximal multiplicity of words of $x$. 
Backtracking, we have 
\[
\norm{\tfidf{x}-\tfidf{\xtilde}}^2 \leq \frac{\maxidf^2}{T^2} \cdot 16\maxmult^2 \card{\Indset}^2
\, ,
\]
and we can conclude by simply taking the square root of this last display. 
\qed 


\subsection{Proof of Proposition~\ref{prop:robustness-normalized-tfidf}}

We notice that 
\begin{equation}
\label{eq:euclidean-to-cosine}
\norm{\normtfidf{x}-\normtfidf{\xtilde}}^2 = 1+1-2\normtfidf{x}^\top \normtfidf{\xtilde} = 2 - 2\frac{\tfidf{x}^\top\tfidf{\xtilde}}{\norm{\tfidf{x}} \norm{\tfidf{\xtilde}}}
\, .
\end{equation}
In this last term we recognize the \emph{cosine similarity} between $\tfidf{x}$ and $\tfidf{\xtilde}$.  
Since we are working under the assumptions of Lemma~\ref{lemma:cosine-distance-robustness}, we have 
\[
\frac{\tfidf{x}^\top\tfidf{\xtilde}}{\norm{\tfidf{x}} \norm{\tfidf{\xtilde}}} \geq 1-\frac{8\maxmult\maxidf\card{\Indset}}{\norm{\tfidf{x}}}
\, .
\]
Coming back to Eq.~\eqref{eq:euclidean-to-cosine}, we see that 
\[
\norm{\normtfidf{x}-\normtfidf{\xtilde}}^2 \leq \frac{16\maxmult\maxidf\card{\Indset}}{\norm{\tfidf{x}}}
\, .
\]
We conclude by using Lemma~\ref{lemma:lower-bound-norm-non-normalized-tfidf} and taking the square root. 
\qed 


\subsection{Auxilliary results}

We have the following result, key to the proof of Prop.~\ref{prop:robustness-normalized-tfidf}, and of independent interest:

\begin{lemma}[Cosine similarity robustness]
\label{lemma:cosine-distance-robustness}
Let $x$ be a document. 
Let $\Indset \subseteq [T]$ such that $\card{\Indset}\leq \norm{\tfidf{x}}/(4\maxmult\maxidf)$ and $\xtilde \in \ballHamming{x}{\Indset}$. 
Then 
\begin{equation}
\label{eq:cosine-distance-similarity}
\frac{\tfidf{x}^\top\tfidf{\xtilde}}{\norm{\tfidf{x}} \norm{\tfidf{\xtilde}}} \geq 1-\frac{8\maxmult\maxidf\card{\Indset}}{\norm{\tfidf{x}}}
\, .
\end{equation}
\end{lemma}

\begin{proof}
By homogeneity, we can multiply numerator and denominator in Eq.~\eqref{eq:cosine-distance-similarity} by $T$ and deal with multiplicities instead of frequencies in this proof. 
We first focus on the numerator and write
\begin{equation}
\label{eq:decomp-dot-product-tfidf}
\tfidf{x}^\top \tfidf{\xtilde} = \tfidf{x}^\top (\tfidf{x} + \tfidf{\xtilde} - \tfidf{x}) = \norm{\tfidf{x}}^2 + \sum_{j=1}^D \mult_j (\mtilde_j-\mult_j)\idf_j^2
\, ,
\end{equation}
by definition of $\Tfidf$. 
Using Cauchy-Schwarz inequality, we find that 
\[
\sum_{j=1}^D \mult_j(\mult_j-\mtilde_j)\idf_j^2 \leq \sqrt{\sum_j \mult_j\idf_j^2} \sqrt{\sum_j (\mult_j-\mtilde_j)^2\idf_j^2}
\, .
\]
In the first part of the right-hand side we recognize $\norm{\tfidf{x}}$, and in the second part, the same quantity bounded in the proof of Proposition~\ref{prop:robustness-non-normalized-tfidf}. 
We deduce that 
\[
\sum_j \mult_j(\mult_j-\mtilde_j)\idf_j^2 \leq \norm{\tfidf{x}} \cdot 4\maxmult\maxidf \card{\Indset} 
\, .
\]
Coming back to Eq.~\eqref{eq:decomp-dot-product-tfidf}, we have proved that 
\[
\tfidf{x}^\top \tfidf{\xtilde} \geq \norm{\tfidf{x}}^2 - 4\maxmult\maxidf \norm{\tfidf{x}}\card{\Indset}
\, ,
\]
which is positive under our assumption. 
Let us now look into the denominator of Eq.~\eqref{eq:cosine-distance-similarity}. 
Using the triangle inequality and Proposition~\ref{prop:robustness-non-normalized-tfidf}, we write
\[
\norm{\tfidf{\xtilde}} \leq \norm{\tfidf{x}} + 4\maxmult\maxidf \card{\Indset} 
\, .
\]
Putting everything together, we have
\[
\frac{\tfidf{x}^\top\tfidf{\xtilde}}{\norm{\tfidf{x}} \norm{\tfidf{\xtilde}}} \geq \frac{\norm{\tfidf{x}}^2 - 4\maxmult\maxidf \norm{\tfidf{x}}\card{\Indset}}{\norm{\tfidf{x}}\cdot (\norm{\tfidf{x}} + 4\maxmult\maxidf \card{\Indset} )} = \frac{1-u}{1+u}
\, ,
\]
with $u\defeq 4\maxmult\maxidf\card{\Indset}/\norm{\tfidf{x}}$. 
Again, by our assumption, $u\in (0,1)$. 
It is straightforward to show that $(1-u)/(1+u)\geq 1-2u$ for all $u\in (0,1)$, and we deduce the result. 
\end{proof}

We also have the following: 

\begin{lemma}[Lower bound on $\norm{\tfidf{x}}$]
\label{lemma:lower-bound-norm-non-normalized-tfidf}
Let $x$ be a document. 
Let $\minidf$ be the minimum inverse document frequency for words contained in $x$ and $D(x)$ the size of the local dictionary. 
Then
\[
\norm{\tfidf{x}} \geq \frac{T\minidf}{\sqrt{D(x)}}
\, .
\]
\end{lemma}

\begin{proof}
Straightforward from the definitions and the comparison of $p$-norms. 
\end{proof}


\section{Dynamics of interpolation}
\label{sec:minimizarion.ode-appendix}

Recall that we are considering, for all $\temps\in [0,1]$, the following minimization problem:
\begin{equation}
\label{eq:optim.problem.general-appendix}
q(\temps) \defeq \Argmin_{q\in\Reals^d} \left\{ \Psilinear(\temps,q) + \frac{\alpha}{2}\norm{q}^2 \right\}
\, .
\end{equation}
In this section, we show that under mild regularity assumptions on $\Psi$, $q$ is the unique solution of the following ODE:
\begin{equation}
\label{eq:ODE-appendix}
\left(\nabla^2 \Psilinear(\temps,q(\temps) ) + \alpha \Identity \right)
q'(\temps) + 
\partial_\temps \nabla \Psilinear(\temps,q(\temps)) = 0
\, .
\end{equation}

\paragraph{Notation.}
For any matrix $M\in\Reals^{A\times B}$, let us define the \emph{operator norm} of $M$ as
\[
\opnorm{M} \defeq \sup \left\{ \frac{\norm{Mv}}{\norm{v}}, v\in\Reals^B \setminus \{0\}\right\}
\, .
\]
For any $\rho >0$, we also define $B_d(\rho)$ the open Euclidean ball of center $0$ and radius $\rho$. Finally, for $a_1,a_2>0$, define $a_1 \vee a_2 \defeq \max(a_1,a_2)$.

We can now state the required assumptions on $\Psi$. 

\begin{assumption}[Convexity]
\label{assumption:convex-appendix}
Let $d\ge 1$. 
We suppose that  $\Psi \in \mathcal{C}^{1,2} ([0,1] \times \Reals^d;\Reals)$ and that,
for all $(\temps,q)\in [0,1]\times \Reals^d$, $\nabla^2 \Psi(\temps,q)$ is a positive semi-definite matrix. 
\end{assumption}

Since $\alpha>0$, A.\ref{assumption:convex-appendix} this guarantees that $q(\mu)$ is uniquely-defined for each $\mu$. 
Next, we define some quantities related to the local Lipschitz continuity of $\Psi$ and its derivatives. 

\begin{definition}[Local Lipschitz semi-norms]
\label{def:local-lipschitz-continuous}
Let $\Psi \in \mathcal{C}^{1,2}([0,1]\times\Reals^d;\Reals)$. 
    For all $\radius>0$, let us define 
    \begin{equation}
    \label{eq:lipschitz.norm-appendix}
    \lipconstant_1(\radius) \defeq \sup_{\substack{\temps \in [0,1] \\q\neq \tilde{q} \in B_d(0,\radius) }}  \frac{\opnorm{\nabla^2 \Psi(\temps,q) - \nabla^2 \Psi(\temps,\tilde{q}) }}{\norm{q-\tilde{q}}} 
    \, , \quad 
    \lipconstant_2(\rho)\defeq \sup_{\substack{\temps \in [0,1] \\q\neq \tilde{q} \in B_d(0,\radius) }} \frac{\norm{\partial_\temps \nabla \Psi(\temps,q) - \partial_\temps \nabla \Psi(\temps,\tilde{q}) }}{\norm{q-\tilde{q}}}
\, ,
    \end{equation}
and
    \begin{equation}\label{eq:gradient.norm-appendix}
    \supgradient(\radius) \defeq \sup_{\substack{\temps \in [0,1]\\q\in B_d(0,\radius)}} \norm{\partial_\temps \nabla \Psi(\temps,q)}
    \, .
    \end{equation}
\end{definition}

Our second assumption on $\Psi$ at this stage is that these quantities are all finite. 

\begin{assumption}[Global Lipschitz continuity]
\label{assumption:global.lipschitz-appendix}
Let $\Psi \in \mathcal{C}^{1,2}([0,1]\times\Reals^d;\Reals)$. Suppose that
\[\sup_{\rho > 0} \big(\lipconstant_1(\radius) + \lipconstant_2(\rho)\big)< + \infty \quad \text{and} \quad \sup_{\rho>0}\supgradient(\radius) < +\infty,\]
where 
$\lipconstant_1(\radius)$, $\lipconstant_2(\radius)$,  and $\supgradient(\radius)$ are defined in Eq.~\eqref{eq:lipschitz.norm-appendix} and Eq.~\eqref{eq:gradient.norm-appendix}.
\end{assumption}

In this setting, we are able to prove the following result:

\begin{theorem}[Equivalence ODE/minimization problem]
\label{th:dynamics-interpolation-appendix}
Assume that $\Psi$ satisfies A.\ref{assumption:convex-appendix} and A.\ref{assumption:global.lipschitz-appendix}.
Then $\temps \mapsto \sol(\temps)$ is differentiable on $[0,1]$, and $\sol$ is the unique solution of Eq.~\eqref{eq:ODE-appendix}.
\end{theorem}

Note that under assumption A.\ref{assumption:convex-appendix} the matrix $\nabla^2 \Psi(\temps,q) + \alpha \Identity$ is invertible.  
One can then rewrite Eq.~\eqref{eq:ODE-appendix} in a more standard form, namely
\begin{equation}
\label{eq:ODE-standard-form-appendix}
q'(\temps) = - \left( \nabla^2 \Psi(\temps,q(\temps)) + \alpha \Identity\right)^{-1} \partial_\temps \nabla \Psi(\temps, q(\temps))
\, .
\end{equation}
Thus, to study  the ODE problem, one needs the regularity properties (local Lipschitz continuity, boundedness...) of the function 
\begin{equation}
\label{eq:def-phi-appendix}
\Phi : (\temps,q) \in [0,1]\times \Reals^d  \mapsto \Phi(\temps,q) \defeq - \left( \nabla^2 \Psi(\temps,q) + \alpha \Identity\right)^{-1} \partial_\temps \nabla \Psi(\temps, q)
\, .
\end{equation}
The interplay between $\partial_\temps \nabla \Psi$ and $\nabla^2 \Psi$ here is crucial. 
Indeed, in Section~\ref{sec:doc2vec-appendix} we will see that when specified in the \texttt{doc2vec} case, the term in $\partial_\temps$ gives the desired quantity $\frac{\card{\Indset}}{T}$ whereas the term in $\nabla^2 \Psi$ has to be handled using precise properties on the softmax function. 
Theorem~\ref{th:dynamics-interpolation-appendix} is standard in the ODE literature and holds as soon as the quantities appearing in Eq.~\eqref{eq:ODE-standard-form-appendix} are well-behaved. 
More precisely, this is the case $c=0$ of Theorem~\ref{lemma:CS-appendix} in Section~\ref{sec:gronwall-appendix}.
We now simply check that the assumptions of Theorem~\ref{lemma:CS-appendix} are satisfied in the setting of Theorem~\ref{th:dynamics-interpolation-appendix}. 
This is achieved by Lemma~\ref{lemma:norm_inverse_hessian-appendix} and Lemma~\ref{lemma:lipschitz-Phi-appendix}. 
We start by a result upper bounding the norm of the inverse Hessian.   

\begin{lemma}[Norm of inverse Hessian]
\label{lemma:norm_inverse_hessian-appendix}
Let $\Psi : [0,1]\times \Reals^d \to \Reals$. 
Assume that A.\ref{assumption:convex-appendix} holds. 
Then, 
\begin{equation} 
\label{eq:bound_norm_inverse-appendix}
\forall \temps,q \in [0,1] \times \Reals^d , \qquad    \opnorm{( \nabla^2 \Psi(\temps,q) + \alpha \Identity )^{-1}} \le \frac{1}{\alpha}
\, .
\end{equation}
\end{lemma}

The proof of Lemma~\ref{lemma:norm_inverse_hessian-appendix} exploits the fact that $\nabla^2 \Psi$ is a non-negative symmetric matrix and can be diagonalized in orthonormal basis with non-negative eigenvalues. 
The regularization of the minimization problem with the addition of the term $\frac{\alpha}{2} \norm{q}$ can be translated with the addition of the term $\alpha \Identity$ to the previous Hessian matrix, which then becomes a positive definite symmetric matrix. 
One then only has to estimate the smallest eigenvalue of the matrix to conclude.

\begin{proof}
By A.\ref{assumption:convex-appendix}, for all $\temps \in [0,1]$,  $q\mapsto \Psi(\temps,q)$ is convex and, for any $\temps,q\in[0,1]\times \Reals^d$, $\nabla^2 \Psi(\temps,q)$ is a positive semi-definite matrix with non-negative eigenvalues. 
From these, $\dimim(\temps,q) = \rank{\nabla^2 \Psi(\temps,q)}$ of them are non-zero, and they can be ranked as
\[
0<\eigen_{1}(\temps,q) \le \cdots \le \eigen_{\dimim(\temps,q)}(\temps,q) 
\, .
\]
Moreover, there exists an orthogonal matrix $P(\temps,q)$ (meaning that $P(\temps,q)P(\temps,q)^\top = \Identity$) such that 
 \[
 P(\temps,q) \nabla^2 \Psi(\temps,q) P(\temps,q)^\top = \Diag(0,\ldots,0,\eigen_{1}(\temps,q),\ldots,\eigen_{\dimim}(\temps,q))
 \, .
 \]
Furthermore since $\nabla^2 \Psi(\temps,q)$ is a symmetric matrix, its range and its kernel are orthogonal complements, $\kernel{\nabla^2 \Psi(\temps,q)} \oplus^{\perp} \range{\nabla^2 \Psi(\temps,q)} = \Reals^d$ and 
\[h\in \range{\nabla^2 \Psi(\temps,q)}\quad \text{if, and only if,} \quad P(\temps,q)h=(0,\ldots,0,h_1,\cdots,h_{\dimim}).\]
 Hence 
 \[
 P(\temps,q) \left( \nabla^2 \Psi(\temps,q) + \alpha \Identity  \right) P(\temps,q)^\top =\Diag(\alpha,\ldots,\alpha,\eigen_{1}(\temps,q)+\alpha,\ldots,\eigen_{\dimim}(\temps,q)+\alpha)
 \, , 
 \]
 which implies that $ \nabla^2 \Psi(\temps,q) + \alpha \Identity $ is an invertible positive definite matrix such that 
 \[
 P(\temps,q)\left( \nabla^2 \Psi(\temps,q) + \alpha \Identity \right)^{-1} P(\temps,q)^\top = \diag{\frac1\alpha,\ldots,\frac1\alpha,\frac{1}{\eigen_{1}(\temps,q)+\alpha},\ldots,\frac1{\eigen_{\dimim}(\temps,q)+\alpha}}
 \, .
 \]
From the last display, one readily sees that the maximum eigenvalue of $\left( \nabla^2 \Psi(\temps,q) + \alpha \Identity \right)^{-1}$ is $1/\alpha$, proving our claim. 
\end{proof}

The next lemma shows how regularity assumptions on $\Psi$ translate into regularity conditions for $\Phi$.  

\begin{lemma}[Global-Lispchitz continuity of $\Phi$]
\label{lemma:lipschitz-Phi-appendix}
    Let $\Psi$ such that  A.\ref{assumption:convex-appendix} and A.\ref{assumption:global.lipschitz-appendix} hold.
    Then $\Phi$ is globally Lipschitz continuous in $q$ uniformly in $\temps \in[0,1]$.
    Moreover, for all $\radius>0$, 
    \begin{equation}
    \label{eq:crude_lip_bound-appendix}
    \sup_{\substack{\temps\in[0,1]\\q\neq \tilde{q} \in \RR^d}} \frac{\norm{\Phi(\temps,q) - \Phi(\temps,\tilde{q})}}{\norm{q-\tilde{q}}} 
    \le 
    \frac{1}{\alpha}\left(\sup_{\rho>0} \lipconstant_2(\radius)+\frac{\big(\sup_{\rho>0} \lipconstant_1(\radius)\big) \big(\sup_{\rho>0}\supgradient(\radius)\big)}{\alpha}\right)
    \, .
    \end{equation}
\end{lemma}

The proof of Lemma~\ref{lemma:lipschitz-Phi-appendix} relies on the following identity, which is true for any non-negative symmetric matrices $A,B \in \Reals^{d \times d}$ and vectors $X,Y\in \Reals^d$:
\begin{equation}
\label{eq:product-trick}
(A+\alpha \Identity)^{-1}X - (B+\alpha \Identity)^{-1}Y  = -(A+\alpha \Identity)^{-1} (A-B) (B+\alpha \Identity)^{-1} X + (B+\alpha \Identity)^{-1}(X-Y)
\, .
\end{equation}
Lemma~\ref{lemma:norm_inverse_hessian-appendix} allows us to conclude.

\begin{proof}
Let $q,\tilde{q} \in B_d(0,\radius)$.
Using Eq.~\eqref{eq:product-trick}, we have 
    \begin{align}
    \Phi(\temps,q) - \Phi(\temps,\tilde{q}) = & -\left( (\nabla^2 \Psi(\temps,q) + \alpha I )^{-1} - (\nabla^2 \Psi(\temps,\tilde{q}) + \alpha I )^{-1} \right) \partial_\temps \nabla \Psi(\temps,q) \nonumber\\
    & - (\nabla^2 \Psi(\temps,\tilde{q}) + \alpha I )^{-1} \left(\partial_\temps \nabla \Psi(\temps,q)-\partial_\temps \nabla \Psi(\temps,\tilde{q})\right) \nonumber \\
    =& -(\nabla^2 \Psi(\temps,q) + \alpha I )^{-1} \left(\nabla^2 \Psi(\temps,\tilde{q})-\nabla^2 \Psi(\temps,q)\right) (\nabla^2 \Psi(\temps,\tilde{q}) + \alpha I )^{-1} \partial_\temps \nabla \Psi(\temps,q)\label{eq:local_lip_bound_1-appendix} \\
    & - (\nabla^2 \Psi(\temps,\tilde{q}) + \alpha I )^{-1} \left(\partial_\temps \nabla \Psi(\temps,q)-\partial_\temps \nabla \Psi(\temps,\tilde{q})\right)
    \, .
    \label{eq:local_lip_bound_2-appendix}
    \end{align}
Taking the norm and using Lemma~\ref{lemma:norm_inverse_hessian-appendix} (in particular Inequality~\eqref{eq:bound_norm_inverse-appendix}), we have for $\rho=\norm{q}\vee\norm{\tilde{q}}$,
    \begin{align*}
        \norm{\Phi(\temps,q) - \Phi(\temps,\tilde{q})} 
            \le &
        \frac{1}{\alpha^2} 
            \opnorm{ \nabla^2 \Psi(\temps,q)-\nabla^2 \Psi(\temps,\tilde{q})} 
            \norm{\partial_\temps \nabla \Psi(\temps,q)} \\
            & + \frac1\alpha \norm{ \partial_\temps \nabla \Psi(\temps,q)-\partial_\temps \nabla \Psi(\temps,\tilde{q})} \\
\norm{\Phi(\temps,q) - \Phi(\temps,\tilde{q})}  \le & \frac{1}{\alpha} \left(\frac{\lipconstant_1(\radius) \supgradient(\radius)}{\alpha} + \lipconstant_2(\radius)\right)\norm{q-\tilde{q}}
\, .
\end{align*}
Taking the supremum for $\temps\in [0,1]$, $q\neq \tilde{q}$ belonging to $B_d(0,\rho)$ and $\rho>0$  yields the claim. 
\end{proof}

We now have all the tools to prove Theorem~\ref{th:dynamics-interpolation-appendix}.

\begin{proof}[Proof of Theorem~\ref{th:dynamics-interpolation-appendix}]
Note that in that setting, using Lemma~\ref{lemma:norm_inverse_hessian-appendix}, for all $\temps\in[0,1]$,  $q\mapsto \Psi(\temps,q)+ \frac{\alpha}{2}\norm{q}^2$ is a strongly convex function and has a unique minimum, which is also the unique critical point of the gradient  $q \mapsto \nabla \Psi(\temps,q) + \alpha q$. Let $q_0\in \Reals^d$ be  such that
\[
\{q_0\} = \Argmin \Psi(0,q) + \frac{\alpha}{2} \norm{q}^2
\, .
\]
Thanks to Lemma~\ref{lemma:lipschitz-Phi-appendix}, $\Phi$ satisfies the hypothesis of Theorem~\ref{lemma:CS-appendix}, with
\begin{equation}\label{eq:Gronwall.classical-appendix}
a = \frac{1}{\alpha} \sup_{\rho>0}\supgradient(\rho),
\quad b =  \frac{1}{\alpha} \left(  \frac{\sup_{\rho>0} \lipconstant_1(\rho)\sup_{\rho>0}\supgradient(\rho)}{\alpha} + \sup_{\rho>0}\lipconstant_2(\rho)\right), \quad \text{and} \quad c = 0.
\end{equation}

Let $\temps \in [0,1] \to \sol(\temps)$ be the unique solution up to time $1$ to the ODE
\[
q'(\temps) = -\left(\nabla^2 \Psi(\temps, \sol(\temps))+ \alpha \Identity\right)^{-1} \partial_\mu \nabla \Psi(\temps,\sol(\temps)) = \Phi(\temps,\sol(\temps)),\quad \sol(0) = q_0
\, .
\]
According to Theorem~\ref{lemma:CS-appendix} applied to $\Lambda=\Phi$, it exists and is well-defined up until $\temps=1$. 

Remark that when differentiating in $\temps\in[0,1]$ the function $\temps \mapsto \nabla \Psi(\temps,\sol(\temps)) + \alpha \sol(\temps)$, we have 
\[
\left(\nabla^2 \Psi(\temps,\sol(\temps)) + \alpha \Identity\right) \sol'(\temps) + \partial_\temps \nabla \Psi(\temps,\sol(\temps)) = \left(\nabla^2 \Psi(\temps,\sol(\temps)) + \alpha \Identity\right)\left(\sol'(\temps) - \Phi(\temps,\sol(\temps))\right) = 0
\, . 
\]
Hence 
\[
\nabla \Psi(\temps,\sol(\temps)) + \alpha \sol(\temps)  = \nabla \Psi(0,\sol(0)) + \alpha \sol(0) = 0
\, . 
\]
Thus, for any $\temps \in [0,1]$, 
\[
\{\sol(\temps)\} = \Argmin \left\{\Psi(\temps,q) + \frac{\alpha}{2} \norm{q}^2\right\}
\, ,
\]
which is the promised result. 
\end{proof}

\begin{remark}[Crude bounds under mild assumptions]
\label{rem:crude-bounds-appendix}
Using the same standard result (condition $c=0$ in Theorem~\ref{lemma:CS-appendix}) could naturally give us some crude bounds on $\norm{q(\temps) - q(0)} $, relying only on assumptions A.\ref{assumption:convex-appendix} and A.\ref{assumption:global.lipschitz-appendix}. 
More precisely, these bounds would strongly depend on $\alpha$ and improve as $\alpha \to \infty$. Namely, using Eq.~\eqref{eq:Gronwall.classical-appendix} and Theorem~\ref{th:dynamics-interpolation-appendix} one have for all $\temps \in [0,1]$,
\[
\norm{q(\temps) - q(0)} \le \frac{\temps}{\alpha} \cdot \sup_{\rho} \supgradient(\rho) \cdot \exp{\frac{1}{\alpha} \left(\frac{1}{\alpha} \big(\sup_{\rho>0}\lipconstant_1(\rho)\big)\big(\sup_{\rho} \supgradient(\rho)\big)+\sup_{\rho>0}\lipconstant_1(\rho)\big)\right)\temps}
\, .
\]
This is not the regime we aim at, since $\alpha$ is a small, fixed regularization constant whose role is simply to ensure that the minimization problem is well-posed. 
\end{remark}


\section{Quantitative bounds on the trajectory}
\label{sec:quantitative-bounds-appendix}

Let us recall that $q$ is the minimizer of the interpolated problem \eqref{eq:optim.problem.general-appendix}. 
In the previous section, we have made two assumptions (A.\ref{assumption:convex-appendix} and A.\ref{assumption:global.lipschitz-appendix}), guaranteeing that $q$ is well-defined and is the unique solution to the ODE \eqref{eq:ODE-appendix}. 
In this section, we show how to obtain quantitative bounds on $\norm{q(0)-q(\temps)}$ by studying the ODE \eqref{eq:ODE-appendix}. 
To derive these bounds, we now make two additional assumptions on $\Psi$. 
The first one is an algebraic assumption which greatly improves the computations. 

\begin{assumption}[Common kernel]
\label{assumption:linear-appendix}
We assume that there exists a fixed subspace $\subspace \subset \Reals^d$ such that $\dimension \subspace = \dimim$ and for all $(\temps,q)\in[0,1]\times\Reals^d$
\[ 
\kernel{\nabla^2 \Psi(\temps,q)} = \subspace^\perp, \quad \range{\nabla^2 \Psi(\temps,q)} = \subspace,
\quad \text{ and } \quad \partial_\temps \nabla \Psi(\temps,q) \in \subspace
\, .
\]
\end{assumption} 

The second one is a refined local-Lipschitz assumption (a quantitative version of A.\ref{assumption:global.lipschitz-appendix}), which will allow us to use the case $c\neq0$ in the Gronwall-Bahouri-Bellman type result Theorem~\ref{theorem:gronwall_doc2vec-appendix}. 

\begin{assumption}[quantitative (local)-Lipschitz continuity]
\label{assumption:local_lipschitz-appendix}
Recall $\lipconstant_1$ and $\lipconstant_2$ from Definition~\ref{def:local-lipschitz-continuous}, and $\supgradient$ from Eq.~\eqref{eq:gradient.norm-appendix}. 
For any $\temps,q$, define $\eigen_1(\temps,q)$ the smallest positive eigenvalue of $\nabla^2 \Psi(\temps,q)$. 
For any $\rho >0$, define
\[
w_{-1}(\rho) \defeq \inf_{\substack{ \temps \in [0,1] \\ q\in B_d(0,\rho)}} \eigen_{1}(\temps,q)
\, .
\]
We assume that there exist positive constants $(\Cst_i)_{i\in{-1,\ldots,2}}$ and non negative constants $(\cst_i)_{i\in{-1,\ldots,2}}$, such that for all $\radius>0$, 
\[
\lipconstant_1(\rho)\le \Cst_1 \exps{\cst_1 \rho}
\, , \quad 
\lipconstant_1(\rho)\le \Cst_2 \exps{\cst_2 \rho}
\, , \quad 
\supgradient(\radius)\le \Cst_0 \exps{\cst_0 \rho}
\, ,
\]
and
\[
w_{-1}(\rho) \ge \frac{1}{\Cst{-1}}\exps{-\cst_{-1}\rho}
\, .
\]
\end{assumption}

Under these stronger assumptions, we can obtain the following:

\begin{theorem}[Quantitative bounds on the trajectory]
\label{theorem:minimizer.quantitative-appendix}
Assume that $\Psi$ satisfies A.\ref{assumption:convex-appendix}, A.\ref{assumption:linear-appendix}, and A.\ref{assumption:local_lipschitz-appendix}.
Suppose furthermore that 
\begin{equation}
\label{eq:final-appendix}
4 \Cst_{-1} (\Cst_{0}\Cst_{-1}\Cst_1+\Cst_2) < \exp{- 2 \big((\cst_{-1} + \cst_{0} + \cst_{1})\vee \cst_2\big) \left(\norm{q_0} + \Cst_{-1} \Cst_0 \exps{(\cst_{-1} + \cst_{0} + \cst_{1})\vee \cst_2 \norm{q_0}}\right)}
\, .
\end{equation}
Then $\temps \mapsto \sol(\temps)$ is differentiable on $[0,1]$, it is the unique solution of Eq.~\eqref{eq:ODE-appendix} and furthermore
\[
\forall \temps \in [0,1], \qquad \norm{q(\mu) - q_0} \le 2 \temps \Cst_{-1} \Cst_{0} \exps{(\sum_{i=-1}^2 \cst_i) \norm{q_0}}
\, .
\]
\end{theorem}

The proof of Theorem~\ref{theorem:minimizer.quantitative-appendix} follows the same path as the proof of Theorem~\ref{th:dynamics-interpolation-appendix}, with analogues of Lemmas \ref{lemma:norm_inverse_hessian-appendix} and \ref{lemma:lipschitz-Phi-appendix}. 
The crucial differences come from the fundamental use of A.\ref{assumption:linear-appendix}, which somehow allows us to diagonalize the Hessian $\nabla^2 \Psi$ for all $\temps,q$, and thus allows is to use estimates on the smallest positive eigenvalue of the Hessian. 
In practical cases, this assumption will not allow us to use global-Lipchitz estimates. 
We therefore introduce A.\ref{assumption:local_lipschitz-appendix} to deal with that. 
These two ingredients allow us to use the case $c>0$ in the Gr\"onwall-Bahouri-Bellman type lemma (Theorem~\ref{lemma:CS-appendix}).

The following Lemma gives an improve bounds for the norm of the inverse of the Hessian, using the algebraic requirement on the Hessian. Its proof is similar to the proof of Lemma~\ref{lemma:norm_inverse_hessian-appendix}, and we only point out how to modify it. 

\begin{lemma}[Quantitative norm of inverse Hessian]
\label{lemma:quantitative.norm_inverse_hessian-appendix}
Let $\Psi : [0,1]\times \Reals^d \to \Reals$. 
Assume that A.\ref{assumption:convex-appendix} and A.\ref{assumption:linear-appendix} hold. 
Then 
\begin{equation} \label{eq:bound_restriction-appendix}  \opnorm{
            \restriction{(\nabla^2 \Psi(\temps,q) + \alpha \Identity )^{-1}}{\range{\nabla^2 \Psi(\temps,q)}}
        }
     \le \frac1{\eigen_{1}(\temps,q)}
    \, ,
\end{equation}
where $f|_\subspace$ denotes the restriction of $f$ to the set $\subspace$.
\end{lemma}

\begin{proof}
    Remind that from the proof of Lemma~\ref{lemma:norm_inverse_hessian-appendix}, for all $(q,\temps)\in \RR^d \times [0,1]$, we have
     \[
 P(\temps,q)\left( \nabla^2 \Psi(\temps,q) + \alpha \Identity \right)^{-1} P(\temps,q)^\top = \diag{\frac1\alpha,\ldots,\frac1\alpha,\frac{1}{\eigen_{1}(\temps,q)+\alpha},\ldots,\frac1{\eigen_{\dimim}(\temps,q)+\alpha}}
 \, .
 \]
Assuming that A.\ref{assumption:linear-appendix} holds, we have for all $(\temps,q)\in [0,1]\times \Reals^d$, $\dimim(\temps,q) = \dimim$. 
Restricting to $\subspace$, we see readily that the largest eigenvalue becomes $1/(\alpha+\eigen_1(\temps,q))$.
\end{proof}

Here again, by using the algebraic requirements on $\Psi$ and the local-Lipshcitz bound we are able to derive a local-Lipschitz continuity result for $\Phi$. Here again, the proof is quite similar to the one of Lemma~\ref{lemma:local.lipschitz-Phi-appendix}. 
\begin{lemma}[Local-Lispchitz continuity of $\Phi$]
\label{lemma:local.lipschitz-Phi-appendix}
Let $\Psi$ such that  A.\ref{assumption:convex-appendix}, and A.\ref{assumption:linear-appendix} hold. 
Then $\Phi$ is locally-Lipschitz continuous in $q$ uniformly in $\temps \in[0,1]$. 
More precisely, for all $q,\tilde{q}\in \Reals^d$ and all $\temps \in [0,1]$;
\[
\norm{\Phi(\temps,q) - \Phi(\temps,\tilde{q})} \le \frac{1}{w_{-1}(\norm{\tilde{q}})} \left(\frac{\lipconstant_1(\norm{q} \vee \norm{\tilde{q}} )\supgradient(\norm{q})}{w_{-1}(\norm{q})} + \lipconstant_2(\norm{q} \vee \norm{\tilde{q}} ) \right)\norm{q-\tilde{q}}
\, .
\]
If additionally A.\ref{assumption:local_lipschitz-appendix} holds, we get 
\begin{equation}
\label{eq:condition.gamma-appendix}
\norm{\Phi(\temps,q) - \Phi(\temps,\tilde{q})} \le 2\Cst_{-1} (\Cst_{0}\Cst_{-1}\Cst_1+\Cst_2) \exps{\big((\cst_{-1} + \cst_{0} + \cst_{1})\vee \cst_2\big) \norm{q} \vee \norm{\tilde{q}}} \norm{q-\tilde{q}}
\, ,
\end{equation}
and
\[
\norm{\Phi(\temps,q)}\le \Cst_{-1} \Cst_{0} \exps{(\cst_{-1}+\cst_{0}) \norm{q}}
\, .
\]
\end{lemma}

\begin{proof}
Since $\Psi$ satisfies A.\ref{assumption:linear-appendix}, for all $(\temps,q)\in[0,1]\times \Reals^d$ and all $\tilde{q}\in \Reals^d$ $\partial_\temps \nabla \Psi(\temps,q) \in \range{\nabla^2 \Psi(\temps,\tilde{q})}$, and we can use we can use the second part of Lemma~\ref{lemma:norm_inverse_hessian-appendix}, namely Inequality~\eqref{eq:bound_restriction-appendix}.
Indeed, Eq.~\eqref{eq:local_lip_bound_1-appendix}, in norm, is upper bounded by 
\[
\frac{1}{\lambda_1(\temps,q)+\alpha} \lipconstant_1(\norm{q}\vee\norm{\tilde{q}})\frac{1}{\lambda_1(\temps,\tilde{q})+\alpha}\supgradient(\norm{q}) \norm{q-\tilde{q}}
\, ,
\]
while \eqref{eq:local_lip_bound_2-appendix} is bounded by
\[ 
\frac{1}{\lambda_1(\temps,\tilde{q})+\alpha} \lipconstant_2(\norm{q}\vee\norm{\tilde{q}}) \norm{q-\tilde{q}}
\, .
\]
Summing these last two displays and using the definition of $w_{-1}$ and the bounds of A.\ref{assumption:local_lipschitz-appendix} allows us to conclude. 
\end{proof}

\begin{proof}[Proof of Theorem~\ref{theorem:minimizer.quantitative-appendix}]
Remark that thanks to Lemma~\ref{lemma:local.lipschitz-Phi-appendix}, $\Phi$ satisfies the condition of Theorem~\ref{lemma:CS-appendix} with
\[a= \Cst_{-1}\Cst_0,\quad b=2\Cst_{-1} (\Cst_{0}\Cst_{-1}\Cst_1+\Cst_2)\quad \text{and} \quad c=(\cst_{-1} + \cst_{0} + \cst_{1})\vee \cst_2.\]
Furthermore, Eq.~\eqref{eq:condition.gamma-appendix} can be translated into 
\begin{equation*}   
2 b < \exp{- 2 c \left(\norm{q_0} + a \exps{c \norm{q_0}}\right)}.
\end{equation*}
which is exactly the condition of application of Theorem~\ref{lemma:CS-appendix}. It ensure that there exists a unique solution $\temps\mapsto q(\temps)$ to Eq.~\eqref{eq:ODE-appendix}. Following the proof of Theorem~\ref{th:dynamics-interpolation-appendix} we can conclude easily.
\end{proof}


\section{Specializing our results for \texttt{doc2vec}}
\label{sec:doc2vec-appendix}

In the previous sections, we have seen that, under some technical assumptions on $\Psi$, the mapping $q$ is solution to an ODE, and we proved some bounds on $\norm{q(\mu)-q(0)}$ (by means of Theorem~\ref{theorem:minimizer.quantitative-appendix}). 
In this section, we check that these assumptions are satisfied for the $\Psi$ occurring when considering \texttt{doc2vec} embeddings. 
That is, $\Psi=\Psilinear$, where $\Psilinear$ is defined by Eq.~\eqref{eq:def-interpolation-scheme}. 
This is embodied as Theorem~\ref{theorem:gronwall_doc2vec-appendix}, which is Theorem~\ref{th:bounded-traj} with explicit constants. 
We first prove a useful bound on the norm of $\pi_t$:

\begin{lemma}[Bound on $\pi_t$]
\label{lemma:bound-pit}
Define 
\[
\Pi \defeq 2\winsize \maxsingular{\Red} \cdot \sup_{i} \norm{P_{:,i}}
\, .
\]
Then, for any document $x$ and any position $t\in x$, it holds that
\[
\norm{\pi_t} \leq \Pi
\, .
\]
\end{lemma}

We emphasize that Lemma~\ref{lemma:bound-pit} is true regardless of the model used (PVDMmean, PVDMconcat, PVDBOW), even though this bound can be strengthened for specific models. 
Moreover, it only depends on the $\Proj$ and $\Red$ matrices, which are fixed matrices after training. 

\begin{proof}
Recall that we defined $\pi_t = \Red \Proj h_t$. 
For PVDBOW, $h_t=0$ and there is nothing to prove. 
Otherwise, let us first write
\[
\norm{\pi_t} = \norm{\Red \Proj h_t} \leq \maxsingular{\Red} \cdot \norm{\Proj h_t}
\, 
\]
and focus on $\norm{\Proj h_t}$. 
Let us assume that we work with PVDMconcat. 
Since, in that case, $h_t$ is the concatenation of $2\nu$ arbitrary one-hot vectors, $\Proj h_t$ is the sum of $2\nu$ arbitrary columns of $\Proj$. 
Using the triangle inequality, we deduce that $\norm{\Proj h_t}$ is smaller than $2\winsize$ times the largest norm of a column of $\Proj$. 
When PVDMmean is used, the reasoning is similar. 
Ignoring the $1/(2\winsize)$ factor (which we consider to be part of $\Proj$), the bound is the same. 
\end{proof}

Since the matrix $\Red$ appears in all the definition of the embeddings, one needs some (mild) assumptions on $\Red$. 
The first one ensures that the condition number of $\Red$ is not equal to $+\infty$. 

\begin{assumption}[Condition number of $\Red$]
\label{assumption:R}
Let us $\Red\in \Reals^{D\times d}$.
We assume that $\range{\Red} \subset \Indic^\perp$, 
and further that the smallest singular value of $\Red$ is non-negative, that is,
\[
\minsingular{\Red} > 0
\, .
\]    
\end{assumption}

The requirement for the range of $\Red$ is needed here in order to work in the setting of Lemma~\ref{lemma:softmax-lipschitz-appendix} and~\ref{lemma:gradient_lipschitz-appendix}, and then use the nice bounds for the (local)-Lipschitz constant of the softmax and its Jacobian.

\begin{lemma}
\label{lemma:gradient_psilin}
Suppose that A.\ref{assumption:R} hold. 
Then $\Psilinear$ satisfies A.\ref{assumption:convex-appendix}.
\end{lemma}

\begin{proof}
Recall that $\Indset$ denotes the set of modified words. 
Coming back to the definition of $F$ and $G$, we see that, when forming the difference $F-G$, many cancellations happen. 
To be more precise, replacing a word at position $t$ only modifies $\pi_s$ for $s$ belonging to the neighborhood of $t$.
Thus
\begin{equation}
\label{eq:difference_modified_embeddings-appendix}
G(q) - F(q) = \sum_{t \in \Indsetchange} \Big( \psi_{\xtilde_t}(\pitilde_t + \Red q) - \psi_{x_t}(\pi_t + \Red q) \Big)
\, ,
\end{equation}
where $\Indsetchange \subseteq \{s\in [T], \abs{s-t}\leq \winsize \text{ with } t\in \Indset\}$. 
In particular, there is a numerical constant $\ell>0$ such that $\card{\Indsetchange} \leq \ell \winsize \card{\Indset}$. 
From the definition of $\Psilinear$, Eq.~\eqref{eq:difference_modified_embeddings-appendix}, and Lemma~\ref{lemma:softmax-derivatives}, we deduce that
    \begin{align*}
    \nabla \Psilinear(\temps,q) = & \Red^\top\left(\temps \frac{1}{T}\sum_{t\in \Indsetchange} \big(\nabla \psi_{x_t}(\pi_t + \Red q) - \nabla \psi_{\tilde{x}_t}(\tilde{\pi}_t + \Red q) \big)  + \frac{1}{T}\sum_{t\in x}\nabla\psi_{x_t}(\pi_t + \Red q)\right) \Red \\
    = & \Red ^\top\Bigg(-\temps \frac{1}{T}\sum_{t\in \Indsetchange} \big(\sigma(\pi_t + \Red q) - \sigma(\tilde{\pi}_t + \Red q) \big) + \temps \frac{1}{T}\sum_{t\in \Indsetchange} \big(\Indic_{x_t} - \Indic_{\tilde{x}_t} \big)
     - \frac{1}{T}\sum_{t\in x}\big(\sigma(\pi_t + \Red q) - \Indic_{x_t} \big) \Bigg),
    \end{align*}
    \begin{align*}
    \partial_\temps \nabla\Psilinear(\temps,q) = & \Red ^\top\Bigg(\frac{1}{T}\sum_{t\in \Indsetchange} \big(\Indic_{x_t} - \Indic_{\tilde{x}_t}\big)- \frac{1}{T}\sum_{t\in \Indsetchange} \big(\sigma(\pi_t + \Red q) - \sigma(\tilde{\pi}_t + \Red q) \big) \Bigg) \\
    = & \Red ^\top \left(\frac{1}{T}\sum_{t\in \Indsetchange}\int_0^1 \Big( \big(\Indic_{x_t} - \Indic_{\tilde{x}_t}\big) -  \nabla \sigma(u(\pi_t-\tilde{\pi}_t) + \Red q)(\pi_t - \tilde{\pi}_t)\Big)\dd u\right)
    \end{align*}
    and
    \begin{equation}\label{eq:form_gradient-appendix}
    \nabla^2 \Psilinear(\temps,q)
    = \Red ^\top\Bigg(\temps \frac{1}{T}\sum_{t\in \Indsetchange} \left(\nabla\sigma(\pi_t + \Red q) - \nabla\sigma(\tilde{\pi}_t + \Red q) \right)
     + \frac{1}{T}\sum_{t\in x}\nabla\sigma(\pi_t + \Red q)  \Bigg)\Red,
    \end{equation}
    where we remind that $\nabla\sigma = \Diag(\sigma) - \sigma \sigma^\top$.
    Hence, $\nabla^2 \Psilinear(\temps,\cdot)$ is a symmetric non-negative matrix and $\Psilinear$ satisfies A.\ref{assumption:convex-appendix}.
\end{proof}

Next, we show that $\Psilinear$ satisfies A.\ref{assumption:linear-appendix}. 

\begin{lemma}
\label{lemma:psilin_assumption2-appendix}
Suppose that A.\ref{assumption:R} holds.
For all $\temps \in [0,1]$ and all $q\in \Reals^d$,
\[
\kernel{ \nabla^2 \Psilinear(\temps,q)} = \{0\}
\, ,
\]
and $\Psilinear$ satisfies A.\ref{assumption:linear-appendix} with $N_0 = d$.
Let us recall that we defined $\lambda_1$ the smallest non-zero eigenvalue of the Hessian of $\Psilinear$. 
Then, for all $(\temps,q)\in [0,1]\times\Reals^d$, it holds that 
\[
\eigen_1(\temps,q) \ge  \exps{-2\sqrt{2} \Pi} \frac1D \minsingular{\Red}^2 \exps{-2\sqrt{2}\maxsingular{\Red} \norm{q}}
\, .
\]
\end{lemma}

\begin{proof}
    Let us remind from Lemma~\ref{lemma:smallest-eigenvalue-appendix} the definition of $\lambdamin$, namely for $z \in \Reals^D$, 
    \[
    \lambdamin(z) = \min\left(\spec{\diag{\sigma(z)}-\sigma(z)\sigma(z)^\top}\backslash\{0\} \right)
    \, .
    \]

    For $q,y \in \Reals^d$ and since $ Ry\in \Indic^{\perp}$ (thanks to A.\ref{assumption:R}), the minimax theorem allows us to write (using Eq.~\eqref{eq:form_gradient-appendix})
    \begin{align*}
        \langle \nabla^2 \Psilinear(\temps,q)y,y\rangle = & \temps \frac{1}{T}\sum_{t\in x} \left\langle \left(\nabla \sigma(\tilde{\pi}_t + \Red q)\right) (\Red y),(\Red y)\right\rangle \\
        & + (1-\temps) \frac{1}{T}\sum_{t\in x} \left\langle \left(\nabla \sigma(\pi_t + \Red q)\right) (\Red y),(\Red y)\right\rangle \\
        \ge &   \frac{1}{T}\sum_{t\in x}\left(\temps \lambdamin(\tilde{\pi}_t + \Red q) + (1-\temps)\lambdamin(\pi_t + \Red q) \right)\norm{\Red y}^2
        \, .
    \end{align*}
    Here we have crucialy used A.\ref{assumption:R} and in particular the fact that $\range{R}\subset \Indic^\perp$ and that $\pi_t \in \Indic^\perp$ in order to make $\lambdamin$ appears.
    Let us set
\[\sigma_{(1)}(z) = \min_{i\in[D]} \sigma_i(z)
\, .
\]
Thanks to Lemma~\ref{lemma:smallest-eigenvalue-appendix}, one has
    \[
        \langle \nabla^2 \Psilinear(\temps,q)y,y\rangle \ge   \frac{1}{T}\sum_{t\in x}  \left(\temps D \sigma_{(1)}(\tilde{\pi}_t + \Red q)^2 + (1-\temps) D \sigma_{(1)}(\pi_t + \Red q)^2 \right) D \norm{\Red y}^2
        \, .
        \]
        Furthermore, thanks to Theorem~\ref{theorem:min_softmax-appendix}, 
        \[\sigma_{(1)}(z) \ge \frac1{D} \exps{-\sqrt{2}\norm{q}},\]
        and we have 
\begin{align*}
        \langle \nabla^2 \Psilinear(\temps,q)y,y\rangle \ge &   \frac{1}{T}\sum_{t\in x}  \left(\temps \exp{-2\sqrt{2}\norm{\tilde{\pi}_t + \Red q}} + (1-\temps)\exp{-2\sqrt{2}\norm{\pi_t + \Red q}} \right) \frac1D \norm{\Red y}^2 \\
         \ge &
          \exps{-2\sqrt{2} \Pi} \exps{-2\sqrt{2}\norm{\Red q}}\frac1D \norm{\Red y}^2 \\
         \ge &
          \exps{-2\sqrt{2} \Pi} \exps{-2\sqrt{2}\maxsingular{\Red} \norm{q}}\frac1D \minsingular{\Red}^2 \norm{y}^2,
\end{align*}
where we remind that $\Pi$ is defined in Lemma~\ref{lemma:bound-pit}.
This implies that $\kernel{\nabla^2 \Psilinear(\temps,q)} = \{0\}$, that $\Psilinear$ fulfills A.\ref{assumption:linear-appendix} with $N_0 = d$, and that
\begin{equation}
\label{eq:bound_eigenvalue_lin-appendix}
\eigen_1(\temps,q) \ge  \exps{-2\sqrt{2} \Pi} \frac1D \minsingular{\Red}^2 \exps{-2\sqrt{2}\maxsingular{\Red} \norm{q}}.
\end{equation}
\end{proof}

Next, we show that $\Psilinear$ satisfies A.\ref{assumption:local_lipschitz-appendix}. 

\begin{lemma}[Local Lipschitz continuity of $\Psilinear$]
\label{lemma:loc_lip_philin-appendix}
Suppose that A.\ref{assumption:R} holds. 
Then $\Psilinear$ satisfies A.\ref{assumption:local_lipschitz-appendix} with
\[
\Cst_{-1}= D \exps{2\sqrt{2} \Pi} \frac1{\minsingular{\Red}^2},\quad \Cst_{0}= 4 \ell \winsize \maxsingular{R}\frac{\card{\Indset}}{T}
\, ,
\quad \Cst_1=\frac{8 \exps{6 \sqrt{2} \Pi}}{(D-1)} \maxsingular{\Red}^3 
\, ,
\quad \Cst_2=\frac{4 \ell \winsize \Pi \exps{4\sqrt{2}\Pi} }{D-1} \maxsingular{\Red}^2 \frac{\card{\Indset}}{T}
\, ,
\]
and
\[
\cst_{-1} = 2\sqrt{2}\maxsingular{\Red}\, ,
\quad \cst_{0}= 0 \, ,
\quad \cst_1=3\sqrt{2}\maxsingular{\Red} \, ,
\quad \text{and} \quad 
\cst_2=2\sqrt{2}\maxsingular{\Red}
\, .
\]
\end{lemma}

\begin{proof}
We  have for all $\temps \in [0,1]$ and all $q\in \Reals^d$,
\begin{align*}
\norm{\partial_\temps \nabla \Psilinear(\temps,q)} \le & \norm{\Red^\top \Bigg(\frac{1}{T}\sum_{t\in \Indsetchange} \big(\Indic_{x_t} - \Indic_{\tilde{x}_t} \big)- \frac{1}{T}\sum_{t\in \Indsetchange} \big(\sigma(\pi_t + \Red q) - \sigma(\tilde{\pi}_t + \Red q) \big)   \Bigg) } \\
\le &  4 \maxsingular{\Red} \frac{|\Indsetchange|}{T} \\
\le & 4 \ell \winsize \maxsingular{R}\frac{\card{\Indset}}{T},
\end{align*}
where we have used the fact that $\norm{\sigma} \le 1$ and the previous bound gives the value of $\Cst_0$ and $\cst_0$.
Thanks to Lemma~\ref{lemma:softmax-lipschitz-appendix}  $\sigma$ is locally-Lipschitz continuous and thanks to Lemma~\ref{lemma:gradient_lipschitz-appendix}, $\nabla \sigma$ is also locally-Lipschitz continuous, hence for $q,\tilde{q}\in \Reals^d$
\begin{align*}
\norm{\partial_\temps \nabla \Psilinear(\temps,q)-\partial_\temps \nabla \Psilinear(\temps,\tilde{q})}
\le 
& \norm{\Red^\top \frac{1}{T}\sum_{t\in\Indsetchange}\left( \int_0^1 \left( \sigma(u(\pi_t-\tilde \pi_t) + \Red q) -  \sigma(u(\pi_t-\tilde \pi_t) + \Red \tilde{q}) \right)(\pi_t -\tilde{\pi}_t)\dd u\right)}\\
\le 
& 4 \frac{1}{D-1} \exps{2\sqrt{2}(2 \Pi + \maxsingular{\Red}(\norm{q}\vee\norm{\tilde{q}}))} \ell \winsize \maxsingular{\Red}^2 \frac{\card{\Indset}}{T} \Pi \norm{q-\tilde{q}}\\
\le &
\frac{4 \ell \winsize \Pi \exps{4\sqrt{2}\Pi} }{D-1} \frac{\card{\Indset}}{T}\exps{2\sqrt{2}\maxsingular{\Red}(\norm{q}\vee\norm{\tilde{q}})}\maxsingular{R}^2\norm{q-\tilde{q}}  
\, ,
\end{align*}
where we have used Lemma~\ref{lemma:softmax-lipschitz-appendix} and the bound $\frac{D}{D-1} \le 2$, which gives the value of $\Cst_2$ and $\cst_2$.
Finally, let us remark that 
\begin{align*}
    \opnorm{\nabla^2 \Psilinear(\mu,q)-\nabla^2 \Psilinear(\mu,q)} \le &
        \mu \frac{1}{T}\sum_{t \in x } \norm{\Red^\top\left( \nabla \sigma(\tilde{\pi}_t + \Red q)  - \nabla \sigma(\tilde{\pi}_t + \Red \tilde{q})  \right) \Red} \\
        & + (1-\mu) \frac{1}{T}\sum_{t \in x } \norm{\Red^\top\left( \nabla \sigma({\pi}_t + \Red q)  - \nabla \sigma({\pi}_t + \Red\tilde{q})  \right) \Red} \\
        & \le 
         \frac{8 \exps{6 \sqrt{2} \Pi}}{(D-1)} \maxsingular{\Red}^3 \exps{3\sqrt{2}\maxsingular{\Red}(\norm{q}\vee\norm{\tilde{q}})} \norm{q-\tilde{q}}
     \,  ,
\end{align*}
which gives the value of $\Cst_1$ and $\cst_1$.
Finally, Eq.~\ref{eq:bound_eigenvalue_lin-appendix} gives directly that
\[
w_{-1}(\radius) =  \exps{-2\sqrt{2} \Pi} \frac1D \minsingular{R}^2 \exps{-2\sqrt{2}\maxsingular{R} \norm{q}}
\, ,
\]
which concludes the proof. 
\end{proof}

Next, we show that $\norm{q_0}$ is not too large. 

\begin{lemma}[Bound on $\norm{q_0}$]
\label{lemma:bound_q0-appendix}
Suppose that A.\ref{assumption:R} holds.
Then 
\[
\norm{q_0} \le \frac{\sqrt{2}\maxsingular{\Red}}{\alpha}
\, .
\]
\end{lemma}

We demonstrate Lemma~\ref{lemma:bound_q0-appendix} in practice in Section~\ref{sec:bound-q0-appendix}. 
The key idea behind the proof is that the regularization term $\frac{\alpha}{2}\norm{q}^2$ prevents $q$ from escaping to infinity. 

\begin{proof}
Let us recall that 
\[
q_0  = \Argmin_{q\in\Reals^d} \left\{ \frac{1}{T}\sum_{t\in x} \psi_{x_t}(\pi_t + Rq) + \frac{\alpha}2 \norm{q}^2\right\}
\, .
\]
In view of Lemma~\ref{lemma:gradient_psilin}, $q_0$ is the unique solution of the following equation:
\begin{equation}
\label{eq:singular_value-appendix}
\Red^\top \left( \frac{1}{T}\sum_{t \in x } \Big(\sigma(\pi_t+\Red q) - \Indic_{x_t} \Big) \right) + \alpha q    = 0
\, .
\end{equation}
From Eq.~\eqref{eq:singular_value-appendix}, we deduce that 
\[
\norm{q_0} = \frac{1}{T\alpha} \norm{\Red^\top \left( \sum_{t \in x } \Big(\sigma(\pi_t+\Red q_0) - \Indic_{x_t} \Big) \right)}
\, .
\]
By definition of $\maxsingular{\Red}$ and the triangle inequality, this is upper bounded by 
\begin{equation}
\label{eq:aux-bound-q0}
\frac{\maxsingular{\Red}}{T\alpha} \sum_{t\in x} \norm{\sigma(\pi_t + \Red q_0) - \indic{x_t}}
\, .
\end{equation}
But we notice that, for any $q\in \RR^d$ and $t\in x$,
\begin{align*}
\norm{\sigma(\pi_t+\Red q) - \Indic_{x_t} }^2
= & \sum_{j \neq x_t} \sigma_j(\pi_t+\Red q)^2 + (\sigma_{x_t}(\pi_t+\Red q)-1)^2 \\
= & \sum_{j } \sigma_j(\pi_t+\Red q)^2 + 1 - 2 \sigma_{x_t}(\pi_t+\Red q) \\
\norm{\sigma(\pi_t+\Red q) - \Indic_{x_t} }^2 \le & 2
\, ,
\end{align*}
where we have used the fact that $\norm{\sigma} \le 1$ and $\sigma_i \ge 0$. 
Hence each term in Eq.~\eqref{eq:aux-bound-q0} is upper bounded by $\sqrt{2}$.  
Keeping in mind that the summation over $t\in x$ has at most $T$ terms, we deduce the result.
\end{proof}

We are now ready to apply case $c >0$ of Theorem~\ref{lemma:CS-appendix} to obtain the promised quantitative bounds. 

\begin{theorem}[Quantitative bounds for \texttt{doc2vec} embeddings]
\label{theorem:gronwall_doc2vec-appendix}
Let $\Psilinear$ defined in Eq.~\eqref{eq:def-interpolation-scheme}, and suppose A.\ref{assumption:R} holds. 
Let us define
\[
A \defeq 4 \ell \winsize D \exps{2 \sqrt{2} \Pi} \frac{\maxsingular{R}}{\minsingular{R}^2}
\, ,
\]
\[
B \defeq 64  \ell \winsize D\frac{\maxsingular{R}^2}{\minsingular{R}^2}\exps{10\sqrt{2}\Pi}\left(\frac{\maxsingular{R}^2}{\minsingular{R}^2} + \frac{\Pi}{D-1}\right)
\, ,
\]
and
\[
C \defeq 5 \sqrt{2} \maxsingular{R}
\, .
\]
Suppose that 
\begin{equation}
\label{eq:condition_q0-appendix}
\frac{\card{\Indset}}{T} \le \frac{1}{2B} \exps{-2 \left(AC+1\right) \exps{C\frac{\sqrt{2}\maxsingular{\Red}}{\alpha}}}
\, ,
\end{equation}
Then
\begin{equation}
\sup_{\temps\in[0,1]} \norm{\sol(\temps) - q_0} \le 2 A  \exps{C \norm{q_0}} \frac{\card{\Indset}}{T}
\, .
\end{equation}
\end{theorem}

\begin{proof}
Remark that $\Psilinear$ satisfies A.\eqref{assumption:convex-appendix} (Lemma~\ref{lemma:gradient_psilin}), A.\eqref{assumption:linear-appendix} (Lemma~\ref{lemma:psilin_assumption2-appendix}) and A.\eqref{assumption:local_lipschitz-appendix} (Lemma~\ref{lemma:loc_lip_philin-appendix}). 
Therefore the assumptions of Theorem~\ref{theorem:minimizer.quantitative-appendix} are satisfied. 
Let us note that 
\[
\cst_0 \cst_{-1} = A \frac{\card{\Indset}}{T}
\, ,
\]
\[
2\Cst_{-1}(\Cst_{-1}\Cst_1\Cst_0 + \Cst_2) \le B \frac{\card{\Indset}}{T}
\]
and
\[
(\cst_{-1}+\cst_1)\vee \cst_2 = C
\, .
\]
Remark that in that setting, thanks to Lemma~\ref{lemma:bound_q0-appendix}, $\norm{q_0} \le \frac{\sqrt{2}\maxsingular{\Red}}{\alpha}$. 
We also have the following straightforward bounds:
\[ 
\frac{\card{\Indset}}{T} \le 1 \quad \text{and} \quad  C \norm{q_0}\le \exps{C\norm{q_0}} \le \exps{C \frac{\sqrt{2}\maxsingular{\Red}}{\alpha}}
\, .
\]
Using Eq.~\eqref{eq:condition_q0-appendix},  one necessarily have 
\begin{align*}
2B \frac{\card{\Indset}}{T} \le &
\exps{-2(AC + 1 ) \exps{C\frac{\sqrt{2}\maxsingular{\Red}}{\alpha}}} \\
\le &
\exps{-2C\left(\norm{q_0} + A  \frac{\card{\Indset}}{T} \exps{C \norm{q_0}} \right)}
\, .
\end{align*}
This guarantees that Eq.~\eqref{eq:final-appendix} and one can apply Theorem~\ref{theorem:minimizer.quantitative-appendix}, which yields the desired result. 
\end{proof}


\section{Gr\"onwall-Bahouri-Bellman type result}
\label{sec:gronwall-appendix}

In this section, we collect all results related to ODEs. 
In our setting, as seen in Lemma~\ref{lemma:lipschitz-Phi-appendix}, and in view of A.\ref{assumption:local_lipschitz-appendix}, the coefficients of Eq.~\ref{eq:ODE-appendix} are not globally Lispchitz (although locally-Lipschitz). 
Thus, while local existence and uniqueness of solutions to Eq.~\eqref{eq:ODE-appendix} is a given (small $\temps$ regime),  existence up to time $1$ and non-explosion of the solutions is much more challenging to achieve (large $\temps$ regime). 
Unfortunately, this is the regime that we are interested into: the local behavior of the ODE at $\mu=0$ does not tell us anything interesting, since what we aim at is the comparison between the starting point ($\mu=0$) and final point ($\mu=1$) of the dynamic. 
Our strategy is to use an \emph{ad hoc} extension of the Gr\"onwall-Bahouri-Bellman lemma to deal with our specific setting. 

Our approach is inspired by proofs of Gr\"onwall-Bellman-Bahouri type lemmas, see for example \citet{dannanIntegralInequalitiesGronwallBellmanBihari1985,agarwalGeneralizationRetardedGronwalllike2005,kimGronwallBellmanPachpatte2009,pachpatteJournalInequalitiesPure2004}.
It relies on an explicit integration of the integral inequality which will pop up in the computations. 
Note that, instead of generic local constants $\lipconstant$, $\supgradient$, and $w_{-1}$, and in view of Section~\ref{sec:doc2vec-appendix}, we will suppose that all those quantity are locally bounded by some exponential functions. 
Our derivation is very close to that of Pachpatte inequality \citep{pachpatteJournalInequalitiesPure2004}, but here we keep track of the constants. 
In doing, so \textbf{we gain an explicit criteria for non-explosion of the solutions} up to time $\temps =1$. 
To view other applications of non-explosion on the time-one map, one could also consult   \cite{bailleulNonexplosionCriteriaRough2020} and the references therein. 

\begin{theorem}[Gr\"onwall-Bahouri-Bellman type inequality]
\label{lemma:CS-appendix}
Let $\funclip : [0,1] \times \Reals^d \to \Reals^d $ be a continuous function and $a,b,c>0$ be numerical constants such that, for all $q,\tilde{q} \in \Reals^d$,
\begin{equation}
\label{eq:gronwall-hyp-simple}
\sup_{\temps \in [0,1]} \norm{\funclip(\temps,q)} \le a \exps{c \norm{q}}
\, ,
\end{equation}
and
\begin{equation}
\label{eq:gronwall-hyp-diff}
\sup_{\temps \in [0,1]}\norm{\funclip(\temps,q) - \funclip(\temps,\tilde{q})} \le b \exps{c(\norm{q}\vee\norm{\tilde{q}})} \norm{q-\tilde{q}}
\, .    
\end{equation}
Let $q_0 \in \Reals^d$ such that either $c=0$ or
    \begin{equation}
    \label{eq:condition_gronwall_lemma} 
2 b < \exp{- 2 c \left(\norm{q_0} + a \exps{c \norm{q_0}}\right)}.
\end{equation}
Then, there exists a unique function $\sol:[0,1] \to \Reals^d$ such that $\sol(0)=0$ and for all $\temps\in[0,1]$
\begin{equation}
\label{eq:general_ODE}
\sol'(\temps) = \funclip(\temps,\sol(\temps))
\, .
\end{equation}
Furthermore, for all $\temps \in [0,1]$,
\[
\norm{\sol(\temps) -  q_0} \le \begin{cases}
    2 \temps a \exps{c \norm{q_0}}  & \text{ if } c >0, \\
    \temps a \exps{b} & \text{ if } c =0\, . \\
\end{cases}
\]
\end{theorem}

\begin{proof}
\textbf{Step 1: Existence of the map satisfying \eqref{eq:general_ODE}.}
Note that since $\funclip$ is locally Lipschitz continuous, thanks to the Cauchy-Lipschitz/Picard-Lindel\"of theorem (see \citet[Chapter 2]{arnoldOrdinaryDifferentialEquations1978}), there exists an open interval $I^\star$ of $[0,1]$ and a unique function $\sol : I^\star \to \Reals^d$ such that $\sol$ is the unique solution to Eq.~\eqref{eq:general_ODE}. 
Note that an open interval of $[0,1]$ which contains $0$ is necessarily of the form $[0,\tau)$ with $0<\tau<1$ or $[0,1]$. 
Remark also that for all $\temps \in I^\star$, thanks to the regularity assumption on $\Lambda$, on $I^\star$, $\temps \mapsto \Lambda(\temps,q(\temps))$ is continuous and for $\temps \in I^\star$ the following integral equation  is satisfied:
\[
q(\temps) = q_0 + \int_0^{\temps} \Lambda(\tilde{\temps},q(\tilde{\temps}))\dd \tilde{\temps}
\, .
\]

\textbf{Step 2: .}
Taking the norm in the previous display and using the triangle inequality, we see that 
\[
\norm{q(\temps) - q_0} \leq \int_0^{\temps} \norm{\funclip(\tilde{\temps},q_0)} \dd \tilde{\temps} + \int_0^\temps \norm{\funclip(\tilde{\temps},\sol(\temps))-\funclip(\tilde{\temps},q_0)} \dd \tilde{\temps}
\, .
\]
Using our assumptions on $\Lambda$, namely Eqs.~\eqref{eq:gronwall-hyp-simple} and~\eqref{eq:gronwall-hyp-diff}, we obtain
\[
\norm{q(\temps) - q_0} \leq \temps a \exps{c\norm{q_0}} + b \int_0^\temps \exps{c (\norm{\sol(\tilde{\temps})}+\norm{q_0})}   \norm{\sol(\tilde{\temps})-q_0}\dd \tilde{\temps}
\, .
\]
Since $\norm{\sol(\mutilde)}-\norm{\sol(q_0)}\leq \norm{\sol(\mutilde)-q_0}$, we deduce that 
\begin{align*}
\norm{q(\temps) - q_0} \le  \temps a \exps{c\norm{q_0}} + b \exps{2c\norm{q_0}} \int_0^\temps \exps{c \norm{\sol(\tilde{\temps})-q_0}}\norm{\sol(\tilde{\temps})-q_0}\dd \tilde{\temps}
\, .
\end{align*}
Let us define for all $\temps\in I^\star$,
\[
\cQ(\temps) =  
\begin{cases}
a \exps{c\norm{q_0}} + b \exps{2c\norm{q_0}} \frac{1}{\temps} \int_0^\temps \exps{c \norm{\sol(\tilde{\temps})-q_0}}\norm{\sol(\tilde{\temps})-q_0}\dd \tilde{\temps} & \text{ if } \temps>0 \, ,\\
a \exps{c\norm{q_0}} & \text{ if } \temps = 0\, .
\end{cases}
\]
Note that $\frac{1}{\temps} \int_0^\temps \exps{c \norm{q(\tilde{\temps})-q_0}} \norm{q(\tilde{\temps}) - q_0} \dd \tilde{\temps} \to_{\temps \to 0} \exps{c \norm{q(0)-q_0}} \norm{q(0) - q_0} = 0$ and $\cQ$ is continuous in $\temps = 0$. 
Furthermore, for $\temps \in I^\star \backslash\{0\}$,
\begin{equation}
\label{eq:derivatives_cQ-appendix}
\cQ'(\mu) = -\frac{1}{\temps^2 } b \exps{2c\norm{q_0}} \int_0^\temps \exps{c \norm{q(\tilde{\temps})-q_0}}  \norm{q(\tilde{\temps}) - q_0}  \dd \tilde{\temps} + \frac{1}{\mu}b \exps{c \norm{q(\temps)-q_0}} \norm{q(\temps) - q_0}
\, . 
\end{equation}

With this notation in hand, for any $\temps \in I^\star\backslash\{0\}$, $\norm{\sol(\temps)-q_0} \le \temps \cQ(\temps)$. 
Since we restrict our attention to $\temps \leq 1$, we have 
\begin{equation}
\label{eq:gronwall-first-possible-improvement}
\norm{\sol(\temps)-q_0} \le \cQ(\temps)  
\, .
\end{equation}

\textbf{Step 3: Differential inequality.}
Since $x\mapsto x \exps{cx}$ is a non-decreasing function, we have for $ \temps \in I^\star\backslash\{0\}$
\begin{align}
\cQ'(\temps) & \le b \exps{2c\norm{q_0}} \frac{\norm{\sol(\temps)-q_0}}{\temps} \exps{c \temps \frac{\norm{\sol(\temps)-q_0}}{\temps}} \nonumber \\
 & \le b \exps{2c\norm{q_0}} \cQ(\temps) \exps{c \cQ(\temps)}
\, ,
\label{eq:ineq_CQ_proof_CS-appendix} 
\end{align}
where we used Eq.~\eqref{eq:derivatives_cQ-appendix} for a direct bound one the derivative and Eq.~\eqref{eq:gronwall-first-possible-improvement} to obtain the last display. 

\textbf{Step 4: Cauchy-Lipschitz setting ($c=0$).}
Suppose for a moment that $c=0$ so that we are in the standard setting of global Cauchy-Lipschitz/Picard Lindel\"of Theorem and standard Gr\"onwall Lemma. 
We have for $\temps\in I^\star$
\[
\log\left(\frac{\cQ(\temps)}{a}\right) \le b \temps
\, ,
\]
 $I^\star=[0,1]$ and 
\[\norm{\sol(\temps)-q_0} \le \temps \cQ(\temps) \le a \temps \exps{b \temps}
\, .
\]

\textbf{Step 5: Gr\"onwall-Bahouri-Bellman integration ($c>0$).}
Suppose now that $c>0$. 
Let $\beta>0$. 
Let us remark that for all $x\ge 0$, $x \exps{cx} \le  \frac1{c}\exps{2cx}$, and we have 
\begin{equation}
\label{eq:gronwall-second-possible-improvement}
\cQ'(\temps) \le \frac{b}{c} \exps{2c\norm{q_0}} \exps{2c \cQ(\temps)}
\, .
\end{equation}
Multiplying both sides of Eq.~\eqref{eq:gronwall-second-possible-improvement} by $\exps{-2c \cQ(\temps)}$, one recognize (up to constants) the derivative of $\exps{-2c\cQ}$. Integrating from $0$ to $\temps$, we have proved that  
\begin{equation}
\label{eq:gronwall-after-integration}
\frac{\exps{-2 c a \exps{ c \norm{q_0}}} - \exps{-2c \cQ(\temps) }}{2} \le b \exps{2c\norm{q_0}} \temps 
\, .
\end{equation}
When 
\begin{equation}\label{eq:condition:gronwall} 
\exps{- 2 c a \exps{c \norm{q_0}}} - 2 b \exps{2c \norm{q_0}} \temps >0
\, ,
\end{equation}
we have
\begin{equation}\label{eq:first_bound_cQ-appendix}
\exps{c \cQ(\temps) } \le\left(\exps{-2 c a \exps{c \norm{q_0}}} - 2 b \exps{2c \norm{q_0}} \temps\right)^{-\frac{1}{2}}
\, .
\end{equation}
Furthermore, whenever Eq~\eqref{eq:condition_gronwall_lemma} holds (namely Eq.~\eqref{eq:condition:gronwall} is true for all $\temps \in[0,1]$)
we can take $I^\star = [0,1]$, since Eq.~\eqref{eq:first_bound_cQ-appendix} guaranty that $\cQ$ does not explode.

For $\temps \in I^\star\backslash\{0\}$ which satisfies Eq.~\eqref{eq:condition:gronwall}, when using the previous bound and Eq.~\eqref{eq:ineq_CQ_proof_CS-appendix}, we have the following inequality :
\[
\cQ'(\temps) \le b \exps{2 c \norm{q_0}} \left(\exps{- 2 c a \exps{ c \norm{q_0}}} - 2 b \exps{2c \norm{q_0}} \temps\right)^{-\frac{1}{2}} \cQ(\temps)
\, .
\]
When dividing by $\cQ(\temps)$ and integrating, we get
\[
\log(\cQ(\temps))-\log(\cQ(0)) 
\le 
\exp{b \exps{2 c \norm{q_0}} \int_0^\temps\left(\exps{- 2 c a \exps{ c \norm{q_0}}} - 2 b \exps{2c \norm{q_0}} \tilde{\temps}\right)^{-\frac{1}{2}} \dd \tilde{\temps}}
\, .
\]
Therefore, for all $\mu$ which satisfies Eq.~\eqref{eq:condition:gronwall},
\begin{align}
\norm{\sol(\temps) - q_0} \le  \temps \cQ(\temps) \le & \temps a \exps{c \norm{q_0}} \exp{ \int_0^\temps b \exps{2c\norm{q_0}}\left(\exps{-2 c a \exps{ c \norm{q_0}}} - 2 b \exps{2c \norm{q_0}} \tilde\temps\right)^{-\frac{1}{2}} \dd  \tilde{\temps} }\nonumber\\
\le & \temps a \exps{c \norm{q_0}} \exp{ \left(\exps{- 2 c a \exps{ c \norm{q_0}}}\right)^{\frac{1}{2}} - \left(\exps{-2 c a \exps{c \norm{q_0}}}-2 b \exps{2c \norm{q_0}} \temps\right)^{\frac{1}{2}}} \label{eq:condition_gronwall_explanation-appendix}\\
\le & \temps a \exps{ c \norm{q_0}} \exp{ \left(2 b  \exps{2 c \norm{q_0}} \temps \right)^{\frac12}  }
\, , \nonumber
\end{align}
where we have use that for $0\le v < \tilde{v} $, $\sqrt{\tilde{v}} - \sqrt{v} \le \sqrt{\tilde{v} - v}$.
Note that Eq.~\eqref{eq:condition_gronwall_explanation-appendix}  makes sense since Eq.~\eqref{eq:condition:gronwall} is satisfied. 
Finally, one can use Eq.~\eqref{eq:condition:gronwall} and write
\[
2 b  \exps{2 c \norm{q_0}} \temps \le 2 b  \exps{2 c \norm{q_0}} \le \exps{-2ac \exps{2\norm{q_0}}} \le 1
\, ,
\]
which gives 
\[
\norm{\sol(\temps) - q_0} \le \temps 2 a \exps{c \norm{q_0}}
\, ,
\]
which is the wanted result. 
\end{proof}

\begin{remark}[Improving Theorem~\ref{theorem:gronwall_doc2vec-appendix}]
There are several open avenues to improve Theorem~\ref{theorem:gronwall_doc2vec-appendix}. One possibility is to keep a finer the dependency in $\temps$ when bounding $\norm{q(\temps)-q_0}$ (namely keeping the $\mu$ factor when deriving Eq.~\eqref{eq:gronwall-first-possible-improvement}). 
A second possible improvement is to use a finer inequality than $y\leq \exps{y}$ when deriving Eq.~\eqref{eq:gronwall-second-possible-improvement}. 
Unfortunately, in both cases, we were unsuccessful in integrating these more complicated expressions in a tractable form (derivation leading to Eq.~\eqref{eq:gronwall-after-integration}). 
\end{remark}



\section{Technical results related to the softmax function}
\label{sec:softmax-appendix}


In this section, we collect all technical facts related to the softmax function used throughout the proofs. 
Let us recall that we defined the softmax function from $\Reals^D$ to
$\Reals^D$ as $\sigma (x) = (\sigma_1 (x), \ldots, \sigma_D (x))^\top$, where
for all $i \in [D]$,
\[ 
\sigma_i (x) = \frac{\exps{x_i}}{\sum_{j = 1}^D \exps{x_j}} . \]
We also defined, for all $x\in \Reals^D$ and all $i\in[D]$,
\[
\psi_i(x) = -\log(\sigma_i(x))
\, .
\]


\subsection{Basics on the softmax function}

We start by recalling elementary properties of the softmax function. 

\begin{lemma}[Softmax derivatives]
\label{lemma:softmax-derivatives}
We have
\[
\frac{\partial }{\partial x_k} \Softmax_\ell(x) = 
\begin{cases}
\Softmax_k(x)(1-\Softmax_k(x)) &\text{ if } k=\ell \\
-\Softmax_k(x)\Softmax_\ell(x) & \text{ otherwise.}
\end{cases}
\]
In a more concise way,
\[
\nabla \sigma (x) = \Diag(\sigma(x))  - \sigma(x) \sigma(x)^\top
\, .
\]
\end{lemma}

A straightforward consequence of Lemma~\ref{lemma:softmax-derivatives} is the computation of the first derivatives of $\psi_i$ (these are very standard computations, see for instance Proposition~1 and~2 in \citet{gao_pavel_2017}).

\begin{lemma}[Gradient of $\psi_i$]
\label{lemma:psi-gradient}
	We have
	\[
	\frac{\partial }{\partial x_k} \psi_i(x) = 
	\begin{cases}
	-1 + \Softmax_k(x) & \text{ if } k=i \\
	\Softmax_k(x) & \text{ otherwise.}
	\end{cases}
	\]
 In more concise notation, $\nabla \psi_i = \Softmax - \indic{i}$. 
 \end{lemma}

Similarly, we have:

\begin{lemma}[Hessian of $\psi_i$]
\label{lemma:psi-hessian}
We have 
\[
\frac{\partial^2}{\partial x_k\partial x_\ell} \psi_i(x) = 
\begin{cases}
\softmax{x}_k(1-\softmax{x}_k) \text{ if } k=\ell \\
-\softmax{x}_k\softmax{x}_\ell \text{ otherwise.}
\end{cases}
\]
In more concise notation,  
\begin{equation}
\label{eq:hessian-psi-ind}
\nabla^2 \psi_i = \nabla \sigma = \Diag(\sigma) - \sigma \sigma^\top.
\end{equation}
\end{lemma} 

\begin{corollary}[Convexity of log-softmax]
\label{lemma:convexity_logsoftmax-appendix}
For any $i\in [D]$, the function $\psi_i$ is convex. 
\end{corollary}

The proof of the previous fact relies on the Courant minimax theorem, which gives the value of the eigenvalue of a real symmetric matrix. Furthermore, we also use that fact that a function such that its Hessian is a non-negative symmetric matrix is convex. 

\begin{proof}
Let $x,v\in \Reals^D$. Since $\sum_{i} \sigma_i(x) = 1$, we have 
\begin{align*}
    \scalar{\nabla^2 \psi_i(x) v }{v} &= \sum_k \sigma_k(x) v_k^2  - \sum_k\sigma_k(x) \sum_{\ell} \sigma_\ell(x) v_\ell v_k  \\ 
    &= \sum_k \sigma_k(x) v_k^2  - \left(\sum_k\sigma_k(x)  v_k\right)^2 = \sum_{k}\sigma_k(x)\left(v_k - \sum_{\ell} \sigma_\ell(x) v_\ell\right)^2 \ge 0
    \, ,
\end{align*}
Hence, thanks to the Courant minimax principle, all the eigenvalues of $\nabla^2 \psi_i$ are non-negative. Hence $\nabla^2 \psi_i$ is a non-negative symmetric matrix, and $\psi_i$ is a convex function.  
\end{proof}

The following proposition controls the spectrum of the Hessian of $\psi_i$, that is the gradient of the softmax, in function of the minimal and maximal values of the softmax function. 

\begin{lemma}[Spectrum of the softmax Jacobian]
\label{lemma:smallest-eigenvalue-appendix}
Let $\radius>0$. For $x\in \Reals^D$, let us define 
\[
\sigma_{(1)}(x) \defeq \min_{i \in [D]} \sigma_i (x)
\, ,\]
and 
\[
\sigma_{(D)}(x) \defeq \max_{i \in [D]} \sigma_i (x)
\, .
\]
Let us define 
\[
\lambdamin(x) \defeq \min \left\{\spec{\nabla \sigma(x)} \backslash\{0\}\right\}
\, ,
\]
and
\[
\lambdamax(x) \defeq \max \left\{\spec{\nabla \sigma(x)}\right\} 
\, .
\]
Then
\[
D \sigma_{(1)}^2(x) \le \lambdamin(x) \le \lambdamax(x) \le D \sigma_{(D)}^2(x)
\, .
\]
\end{lemma}

\begin{proof}
According to Lemma~\ref{lemma:psi-hessian}, 
\[
\nabla \sigma (x) = \Diag(\sigma(x)) - \sigma(x) \sigma(x)^\top
\, .
\]
This matrix is symmetric, and according to Corollary~\ref{lemma:convexity_logsoftmax-appendix}, its eigenvalues are non-negative real numbers. 
Since $\sum_i \sigma_i(x) = 1$, one has
\[
\nabla \sigma (x) \Indic = 0
\, ,
\]
where, as before, $\Indic = (1,\ldots,1)^\top$. 
Since for all $i\in[D]$, $\sigma_i(x)\neq 0$, if 
$v\in \kernel{\nabla \sigma (x)}$ then necessarily for all 
$i\in[D]$, $v_i \sigma_i(x) - \sigma_i(x) \sum_j \sigma_j(x) v_j = 0 $, and 
$v = v_1 \Indic$. 
Hence, $\range{\nabla \sigma(x)} = \Indic^{\perp}$ and $\kernel{\nabla \sigma(x)}=\vecspan{\Indic}$. 
Using the Courant minimax characterization of eigenvalues, we have 
\begin{equation}
\label{eq:minimax-appendix}
\lambdamin(x) = \min_{\substack{v\in \Indic^{\perp} \\ \norm{v} = 1 }} \scalar{\nabla \sigma(x) v}{v} = \min_{\substack{v\in \Indic^{\perp} \\ \norm{v} = 1 }} v^{\top}\left(\nabla \sigma(x)\right) v \quad \text{and} \quad \lambdamax(x) = \max_{\substack{v\in \Indic^{\perp} \\ \norm{v} = 1 }} \scalar{\nabla \sigma(x) v}{v} =\max_{\substack{v\in \Indic^{\perp} \\ \norm{v} = 1 }} v^{\top}\left(\nabla \sigma(x)\right) v
\, .
\end{equation}
Note then that for $v\in \Indic^\perp$ (and dropping the $x$ dependency), 
\[
v^{\top}\left(\nabla \sigma(x)\right) v = \sum_{i=1}^D \sigma_i v_i^2 - \left(\sum_{i=1}^D \sigma_j v_j\right)^2
\, .
\]

Now, the Cauchy-Schwarz inequality guarantees that the previous display is non-negative, but this is not enough to conclude. 
We resort to the \emph{four-letter identity} (\citet[Exercise~3.7]{steele_2004}, see also \citet[Proposition~13]{garreau_mardaoui_2021}) to write
\begin{equation}
\label{eq:aux-lower-bound-2}
v^{\top}\left(\nabla \sigma(x)\right) v = 
\sum_{i=1}^D \sigma_iv_i^2 - \left(\sum_{i=1}^D \sigma_iv_i\right)^2 = \sum_{j < k} \sigma_j\sigma_k (v_k-v_j)^2
\, .
\end{equation}
Keeping in mind that the $\sigma_i$s are non-negative, this last identity  gives 
\[
\sigma_{(1)}^2 \sum_{j < k} \sigma_j\sigma_k (v_k-v_j)^2 \le v^{\top}\left(\nabla \sigma(x)\right) v \le \sigma_{(D)}^2 \sum_{j < k} \sigma_j\sigma_k (v_k-v_j)^2
\, .
\]
In the term
\[
\left(\sigma_{(1)}\right)^2 \cdot \sum_{j<k} (v_k-v_j)^2
\, .
\]
we recognize ($D^2$ times) the \emph{variance} of the $v_i$s. 
More precisely, 
\[
\frac{1}{D^2}\sum_{j<k} (v_k-v_j)^2 = \frac{1}{D} \sum_{i=1}^D \left(v_i - \frac{1}{D}\sum_j v_j\right)^2
\, .
\]
Since $v\in \vecspan{\Indic}^\bot$, we know that $\sum_j v_j=0$, and the previous display reduces to ($1/D$ times) the norm of $v$. Whenever $\norm{v} = 1$ and $v \in \Indic^\perp$ we have shown 
\[D \sigma_{(1)}(x)^2 \le v^{\top}\left(\nabla \sigma(x)\right) v \le D \sigma_{(D)}(x)^2.\]
Coming back to the characterization of the eigenvalues given by Eq.~\eqref{eq:minimax-appendix}, we deduce the result. 
\end{proof}

The previous bound, associated with estimates on the infimum and supremum of the softmax function on balls gives estimates on the (local)-Lipschitz constant of the softmax.

\begin{lemma}[local-Lipschitz continuity of the softmax]
\label{lemma:softmax-lipschitz-appendix}
For all $x,y\in \Reals^D$ such that $x,y \in \Indic^\perp$,
\[
\norm{\sigma(x) - \sigma(y)} \le \frac{D}{(D-1)^2} \exp{2\sqrt{\frac{D}{D-1}}(\norm{x}\vee\norm{y})}\norm{x-y}
\, .
\]
\end{lemma}

In order to prove the previous lemma, one only has to remember that the operator norm for real symmetric matrices is the greatest eigenvalue, and use the fundamental theorem of analysis.

\begin{proof}
Let $x,y\in \Indic^\perp$.  
We write
\begin{align*}
\norm{\sigma(x) - \sigma(y)} = & \norm{\int_0^1 \nabla \sigma(u(x-y)+ y)(x-y) \dd u} \\
\le & \int_0^1 \opnorm{\nabla \sigma (u(x-y) + y)} \norm{x-y}\dd u.
\end{align*}
One can then use Theorem~\ref{theorem:min_softmax-appendix} and Lemma~\ref{lemma:smallest-eigenvalue-appendix}, and we have for all $u\in[0,1]$,
\begin{align*}
\opnorm{\nabla \sigma (u(x-y) + y)} &= \lambdamax(u(x-y) + y) \\
& \le D \sigma_{(D)}(u(x-y) + y)^2 \\
&\le D \left(\frac{1}{1 + (D-1) \exps{-\sqrt{\frac{D}{D-1}} \norm{u(x-y)+y}}}\right)^2 \\
&\le \frac{D \exps{2\sqrt{\frac{D}{D-1}} \norm{u(x-y)+y}}}{(D-1)^2} \\
&\le \frac{D}{(D-1)^2}  \exps{2\sqrt{\frac{D}{D-1}} (\norm{x}\vee\norm{y})}
\, .
\end{align*}
Putting everything together, we have 
\[
\norm{\sigma(x) - \sigma(y)} \le \frac{D}{(D-1)^2} \exps{2\sqrt{\frac{D}{D-1}} \norm{x}\vee\norm{y}}
\norm{x-y}
\, ,
\]
which is the desired result.
\end{proof}

\begin{remark}[Lipschitz continuity of the softmax]
Note that usually, the Lipschitz continuity of the softmax is considered, but with respect to the Frobenius norm. 
One can obtain a crude bound starting from the squared Frobenius norm of the Jacobian, namely 
\begin{equation} 
\label{eq:frobenius-norm-jacobian-softmax}
\sum_i \Softmax_i^2(1-\Softmax_i)^2 + \sum_{i\neq j} \Softmax_i^2\Softmax_j^2
\, .
\end{equation}
Since the Frobenius norm is always greater than the operator norm, this implies the result for a (global) Lispchitz constant equal to $1$.
A finer study of Eq.~\eqref{eq:frobenius-norm-jacobian-softmax} yields a better Lipschitz constant for $\Softmax$. 
This is what \citet{alghamdi_et_al_2022} do, proving $1/2$-Lipschitz continuity for the softmax function (Proposition~1 in Appendix A.4). 
\end{remark}

In view of the specific form of the gradient of the softmax, this implies that we have (almost) the same local-Lipschitz constant for the gradient of the softmax. 

\begin{corollary}[local-Lispchitz continuity of the softmax Jacobian]
\label{lemma:gradient_lipschitz-appendix}
For all $x,y \in \Indic^{\perp}$,
\[
\opnorm{\nabla\sigma(x) - \nabla\sigma(y)} \le  \frac{2D^2}{(D-1)^3 }\exps{3\sqrt{\frac{D}{D-1}}(\norm{x}\vee\norm{y})}\norm{x-y}
\, .
\]
\end{corollary}

The proof is a direct consequence of the particular form of the Jacobian (see Lemma~\ref{lemma:softmax-derivatives}) and of the fact that 
\[
\abs{\sigma_i(x) - \sigma_i(y)} \le \norm{\sigma(x) - \sigma(y)}
\, .
\]

\begin{proof}
Let $x,y \in \Indic^\perp$. We have 
\[
\opnorm{\nabla \sigma (x) - \nabla \sigma(y)} = \sup_{\substack{v \in \Reals^D\\ \norm{v} = 1}} v^\top \left(\nabla \sigma (x) - \nabla \sigma(y)\right) v
\, .
\]
Furthermore, using the same argument as in the proof of Lemma~\ref{lemma:smallest-eigenvalue-appendix}, one can only consider $v\in \Indic^\perp$ with $\norm{v}=1$.
Applying Eq.~\eqref{eq:aux-lower-bound-2} to $x$ and $y$ and forming the difference, we obtain  
\begin{align*}
v^\top \left(\nabla \sigma (x) - \nabla \sigma(y)\right) v = &\sum_{i<k} \Big(\sigma_i(x)\sigma_k(x) - \sigma_i(y)\sigma_k(y)\Big)(v_i - v_k)^2 \\
= & 
\sum_{i<k} \Big(\sigma_i(x) - \sigma_i(y)\Big)\sigma_k(x)(v_i - v_k)^2
+
\sum_{i<k} \sigma_i(y) \Big(\sigma_k(x) - \sigma_k(y)\Big)(v_i - v_k)^2
\end{align*}
Each of these terms can be bounded, using successively the local Lipschitz continuity of the softmax (Lemma~\ref{lemma:softmax-lipschitz-appendix}) and the definition of $\sigma_{(D)}$.
The last display is upper bounded by
\begin{align*}
\left( \frac{D}{(D-1)^2} \exps{2 \sqrt{\frac{D}{D-1}}(\norm{x} \vee \norm{y}) } \sigma_{(D)}(x) + \frac{D}{(D-1)^2} \exps{2 \sqrt{\frac{D}{D-1}}(\norm{x} \vee \norm{y}) } \sigma_{(D)}(y) \right) \sum_{i<k} (v_i -v_k)^2 \norm{x-y}
\, ,
\end{align*}
which, in turn, is smaller than
\[
(\sigma_{(D)}(x) + \sigma_{(D)}(y))\frac{D}{(D-1)^2} \exps{2\sqrt{\frac{D}{D-1}}(\norm{x}\vee\norm{y}) }\sum_{i<k} (v_i -v_k)^2 \norm{x-y}
\, .
\]
Using the bound on $\sigma_{(D)}$ given by Theorem~\ref{theorem:min_softmax-appendix}, we have
\begin{align*}
v^\top \left(\nabla \sigma (x) - \nabla \sigma(y)\right) v \leq \frac{2D}{(D-1)^3} \exps{3 \sqrt{\frac{D}{D-1}}(\norm{x} \vee \norm{y}) } \norm{x-y} \sum_{i<k} (v_i -v_k)^2.
\end{align*}
Using again the same argument as in the proof of Lemma~\ref{lemma:softmax-lipschitz-appendix}, we have 
$\sum_{i<k} (v_i -v_k)^2 = D$, and finally for $v\in \Indic^\perp$ with $\norm{v} = 1$, we have 
\[
v^\top (\nabla \sigma(x) - \nabla \sigma (y) ) v \le \frac{2D^2}{(D-1)^3} \exps{2 \sqrt{\frac{D}{D-1}} (\norm{x} \vee \norm{y}) } \norm{x-y}
\, , 
\]
which gives the wanted result by taking the supremum on $v$.
\end{proof}


\subsection{Minimization of the softmax function}

In this section, we study the extremal values of the softmax function. 
The reason of this study is the close connection of these extremal values with the spectrum of the softmax and log-softmax function. 
Intuitively, the trivial bound $\sigma_i(x)\leq 1$ can be greatly strengthened when the norm of $x$ is constrained: $\sigma_i(x)=1$ is achieved only when $x_i \to +\infty$, which can not be if $x$ lives in a ball of radius $\rho$. 

\begin{theorem}[Bounding the softmax function]
\label{theorem:min_softmax-appendix}
Let $\radius > 0$ and $D \ge  2$. We have
\[ 
\min_{i \in [D]} \inf_{\substack{
\norm{x} \le \radius\\
\sum_j x_j = 0
     }} \sigma_i (x) = \frac{1}{1 + (D - 1) \exps{ \sqrt{\frac{D}{D -
     1}} \radius}}
     \, . 
     \]
     and 
\[ \max_{i \in [D]} \sup_{\substack{
\norm{x} \le \radius\\
\sum_j x_j = 0}} \sigma_i (x) = \frac{1}{1 + (D - 1) \exps{ -\sqrt{\frac{D}{D -
     1}} \radius}} 
     \, . 
     \]
\end{theorem}

\begin{remark}[Bounding the softmax]
The softmax function is ubiquitous in machine learning, and many bounds can be found in the literature~\citep{wei_et_al_2023}. 
Generally, these bounds are pointwise, and not applicable in our case since we need a global bound on the ball of radius $\rho$ (with the additional constraint $\sum_j x_j=0$ coming from our algebraic assumption).  
\end{remark}

\begin{proof}
\textbf{Step 1: the infimum is achieved and invariant by permutation.}
For any $x\in\Reals^D$ such that $\norm{x}\leq \radius$,  $\sigma_i (x) \in (0,1)$ for all $i\in [D]$.
Furthermore,
  \[ 
  \nabla \sigma_i (x) = \sigma_i (x) \Indic_i - \sigma_i (x) \sigma (x)
  \, , 
  \]
  where we remind that $(\Indic_1, \ldots, \Indic_D)$ is the canonical basis of $\Reals^D$. 
  Hence   $\nabla \sigma_i (x) \neq 0$ and the supremum is achieved 
  on the sphere. 
  Note that $B^0(\radius) \defeq  \left\{ x \in \Reals^D : \norm{x} = \radius, \sum_j x_j = 0 \right\}$ is a compact set, and the infimum is  a
  minimum.
  Consider $i_0 \in [D]$ and $y \in B^0(\radius)$ a joint minimizer such that
  \begin{equation}
  \label{eq:def-i0-softmax}
    \sigma_{i_0} (y) = \min_{i \in [D]}
    \min_{x \in B^0(\radius)} \sigma_i (x)
    \, .
  \end{equation}
  Remark that Eq.~\eqref{eq:def-i0-softmax} is invariant by permutation, \ie for any permutation $\tau: [D]\to [D]$, we have
  \[
    \sigma_{\tau(i_0)} (\tau \cdot y) = \sigma_{i_0} (y) = \min_{i \in [D]} \min_{x \in B^0(\radius)}
    \sigma_i (x) 
    \, ,
  \]
  where $\tau \cdot y = (y_{\tau(i)})_{i \in \{ 1, \ldots, D\}}$.
  Hence, one can suppose without loss of generality that $i_0 = 1$.

\textbf{Step 2: the coordinates of a minimizer are equal under $z \mapsto z \exps{-z}$ except at $i_0$.}
  In this setting, since we have for all $i \in \{ 2, \ldots, D \}$, $\sigma_1
  (y) \le \sigma_i (y)$ this implies that $y_1 \le y_i$. 
  Using the
  fact that $\sum_j y_j = 0$, when summing the previous inequality for all $i
  \in [D]$, one gets $y_1 \le 0$. Note that in fact $y_1
  < 0$. 
  Indeed, if $y_1 = 0$, we have $y_i = 0$ for all $i \in [D]$ and $\norm{y} = 0 \neq \radius$.

We are in the setting of a minimization problem under constrains, namely $y$ solves
\begin{equation*}
\text{minimize } \sigma(x)
\qquad 
\text{subject to }
\quad \norm{x}^2 = \rho^2,\quad \scalar{x}{\Indic} = 0
\, .
\end{equation*}
  Using the Lagrange-Multiplier Theorem, there exist $\alpha, \beta \in
  \Reals$ such that for the aforementioned solution $y$ we have
  \[\nabla \sigma(y) + \alpha \nabla\left(\norm{\cdot}^2-\rho^2\right)(y) + \beta \nabla \left(\scalar{\cdot}{\Indic}\right)(y) = 0
  \, ,
  \]
  which translate into
  \begin{align*}
    \sigma_1 (y) - \sigma_1 (y)^2 + 2 \alpha y_1 + \beta &= 0 \\
    - \sigma_1 (y) \sigma_i (y) + 2 \alpha y_i + \beta &= 0 & \text{for } i \in \{ 2, \ldots, D \} \, .
  \end{align*}
  Remark that $\beta = 0$ and $\alpha \neq 0$.
  Indeed, by summing all these previous equality, and using that $\sum y_i = 0$ and
  $\sum \sigma_i = 1$, one gets $D \beta = 0$ and $\beta = 0$. Remind that
  $y_1 < 0$, and since
  \[ 
  \sigma_1 (y) - \sigma_1 (y)^2 + 2 \alpha y_1 = 0
  \, , 
  \]
  if $\alpha = 0$ then $\sigma_1 (y) (1 - \sigma_1 (y)) = 1$, which is not
  possible.
  Hence $\alpha \neq 0$.
  
  We also have that for all $i, j \in \{ 2, \ldots, D \}$,
  \[ 
  \sigma_1 (y) = \frac{2 \alpha y_i}{\sigma_i (y)} = \frac{2 \alpha
     y_j}{\sigma_j (y)} 
     \, . 
     \]
  Using that fact that $\alpha \neq 0$, this implies that $\frac{y_i}{\exps{y_i}}
  = \frac{y_2}{\exps{y_2}}$ for all $i \in \{ 2, \ldots, D \}$ and that
  \[ 
  0 = y_1 + \sum_{i = 2}^D y_i = y_1 + \left( \sum_{i = 2}^D y_2 \exps{- y_2}
     \exps{y_i} \right) = y_1 + \left( \sum_{i = 1}^D \exps{y_i} - \exps{y_1} \right)
     y_2 \exps{- y_2} = y_1 + \exps{y_1} \frac{1 - \sigma_1 (y)}{\sigma_1 (y)} y_2
     \exps{- y_2} 
     \, .\]
  As a consequence, for all $i \in \{ 2, \ldots, D \}$,
  \begin{equation}\label{eq:ratio-softmax-constant}
    y_i \exps{- y_i} = y_2 \exps{- y_2} = - y_1 \exps{- y_1} \frac{\sigma_1 (y)}{1 - \sigma_1 (y)} 
    \, . 
  \end{equation}

\textbf{Step 3: expression of the minimum in function of the solution of $z \exps{-z} = c$.}
  Since $y_1 < 0$, the previous equality~\eqref{eq:ratio-softmax-constant} implies that $y_i > 0$ for $i \in \{
  2, \ldots, D \}$.
  For any $0 < c < \exps{- 1}$, the equation $x \exps{- x} = c$
  has exactly two solutions, which we call $0 < y_- (c) < 1 < y_+ (c)$. Let us
  define
  \[ 
  n =\card{\left\{ 2 \le i \le D, y_i = y_- \left( - y_1 \exps{-
     y_1} \frac{\sigma_1 (y)}{1 - \sigma_1 (y)} \right) \right\}}
     \, 
     \]
the number of ``negative'' solutions. 
By definition of $n$, we necessarily have
  \begin{equation}\label{eq:form-softmax-argmin1}
    \sigma_1 (y) = \frac{\exps{y_1}}{\exps{y_1} + n \exps{y_-} + (D - 1 - n) \exps{y_+}}
    \, .
  \end{equation}
Recall that $\sum_j y_j = 0$ and $\norm{y}=\rho$, hence
  \begin{align}
    y_1 + n y_- + (D - 1 - n) y_+ &= 0 \label{eq:system-softmax-nosquare} \\
    y_1^2 + n y_-^2 + (D - 1 - n) y_+^2 &= \radius^2
    \, . \label{eq:system-softmax-square} 
  \end{align}
  When $n=D-1$, one can solve the previous equations and we have $y_1 = \-\rho \sqrt{\frac{D}{D-1}}$ and $y_j = \rho \sqrt{\frac{1}{D(D-1)}}$ for all $j\in\{2,\cdots,D\}$, and $\sigma_1(y) = \frac{1}{1+(D-1)\exps{\sqrt{\frac{D}{D-1}}\rho}}$.
  
  Since the problem here is symmetric in $y_-$ and $y_+$, one can suppose that $1\le n \le  D-2$. Hence rewriting Eq.~\eqref{eq:system-softmax-nosquare}, we obtain
  \[ 
  y_+ = - \left( \frac{n}{D - 1 - n} y_- + \frac{1}{D - 1 - n} y_1
     \right) 
     \, . 
     \]
  Replacing the value of $y_+$ by the right-hand side of the previous display in Eq.~\eqref{eq:system-softmax-square}, we obtain
  \[ 
  \left( n + \frac{n^2}{D - 1 - n} \right) y_-^2 + 2 \frac{n}{D - 1 - n}
     y_1 y_- - \left( \radius^2 - y_1^2 \left( 1 + \frac{1}{D - 1 - n} \right)
     \right) = 0
     \, . 
     \]
  Dividing by $\left( n + \frac{n^2}{D - 1 - n} \right)$, we get
  \[ y_-^2 + \frac{2}{D - 1} y_1 y_- - \left( \frac{D - 1 - n}{n (D - 1)} \radius^2
     - \frac{D - n}{n (D - 1)} y_1^2 \right) = 0 
     \, . \]
We can see the previous display as a quadratic equation in $y_-$, which we now solve.
There exists $\varepsilon \in \{ - 1, 1 \}$ such that
  \[ y_- = \frac{- y_1 - \varepsilon \sqrt{\frac{D-1-n}{n}} \Delta(y_1)}{D - 1},\]
  where
  \[\Delta(y_1) = \sqrt{(D - 1)
     \radius^2 -D y_1^2}
     \, .
     \]
     Note that in this setting one necessarily have 
     \begin{equation}
     \label{eq:domainy1-appendix}
     -\sqrt{\frac{D}{D-1}} \radius\le y_1 \le 0
     \, ,
     \end{equation}
     since we have already seen that the minimization problem under constrains has a solution,  $y_-$ and $y_+$ exist and $1\le n\le D-2$. When the previous condition is not satisfied, necessarily in the case $n=D-1$ or $n=0$ holds which has already been treated. 
     
Finally, when using the fact that $y_+ = -\frac{1}{D-1-n}(y_1 + n y_-)$,
  \[ 
  y_+ = \frac{- y_1 + \varepsilon \sqrt{\frac{n}{D - 1 - n}} \Delta(y_1)}{D - 1}
  \, . 
  \]
  And since $y_+ > y_-$, we have
  \[ 
  \varepsilon \sqrt{\frac{n}{D - 1 - n}} > - \varepsilon \sqrt{\frac{D - 1
     - n}{n}} 
     \, ,
     \]
  and we conclude that $\varepsilon = 1$, \ie
  \begin{align}
    y_- &= \frac{- y_1 - \sqrt{\frac{D-1-n}{n}} \Delta(y_1)}{D - 1} \label{eq:form-softmax-ymoins} \\
    y_+ &= \frac{- y_1 + \sqrt{\frac{n}{D - 1 - n}} \Delta(y_1)}{D - 1}
    \, . \label{eq:form-softmax-yplus}
  \end{align}
Taking a step back, we have managed to express all coordinates as an explicit function of $y_1$. 

\textbf{Step 4: closed-form expression of the minimum.}
  Replacing $y_{-}$ and $y_+$ in Eq.~\eqref{eq:form-softmax-argmin1} by the expression obtained in Eqs.~\eqref{eq:form-softmax-ymoins} and~\eqref{eq:form-softmax-yplus}, we have to minimize the function of $y_1$ defined
  \begin{align*}
    g (y_1)  = & \frac{\exps{y_1}}{\exps{y_1} {+ n \exps{\frac{- y_1 - \sqrt{\frac{D-1-n}{n}} \Delta(y_1)}{D - 1}}}  + (D - 1 - n)
    {\exps{\frac{- y_1 + \sqrt{\frac{n}{D-1-n}} \Delta(y_1)}{D - 1}}} }\\
       = & \left( 1 + \exps{- \frac{D}{D - 1} y_1} \left( n \exps{- \frac{\sqrt{\frac{D
    - 1 - n}{n}} \Delta(y_1)}{D - 1}} + (D - 1 - n)
    \exps{\frac{\sqrt{\frac{n}{D - 1 - n}} \Delta(y_1)}{D - 1}}
    \right) \right)^{- 1}\\
      = & \Bigg( 
      1 + \exps{- \frac{D}{D - 1} y_1} \sqrt{n (D - 1 - n)}\\
      &\hspace{2cm}\times \left( \sqrt{\frac{n}{D - 1 - n}} \exps{- \frac{\sqrt{\frac{D - 1 -
      n}{n}} \Delta(y_1)}{D - 1}} + \sqrt{\frac{D - 1 -
      n}{n}} \exps{\frac{\sqrt{\frac{n}{D - 1 - n}} \Delta(y_1)}{D - 1}} \right)
    \Bigg)^{- 1}\\
  \end{align*}
  Note that for $y_1$ satisfying Eq.~\eqref{eq:domainy1-appendix} $y_1$ is non-positive. 
  It is elementary to show that $y \mapsto
  \Delta(y)$ is an increasing function on $\Reals_-$. 
  Moreover, for  all $a > 0$, $h : x \mapsto a \exps{- \frac{x}{a}} + \frac{1}{a} \exps{a x}$ is an increasing function on $\Reals_+$. 
  Thus $y\mapsto h(\Delta(y)/(D-1))$ is a decreasing mapping on $\Reals_-$. 
  Hence, by taking $a=\sqrt{\frac{D-1-n}{n}}$, we have 
  \[
  h\left(\frac{\Delta(y_1)}{D-1}\right) \le h(0) = \sqrt{\frac{D-1-n}{n}} + \sqrt{{\frac{n}{D-1-n}}}
  \, ,
  \]
  and  
  \[
  \sqrt{(D-1-n)n} h\left(\frac{\Delta(y)}{D-1}\right) \le \sqrt{(D-1-n)n} \left(\sqrt{\frac{D-1-n}{n}} + \sqrt{{\frac{n}{D-1-n}}}\right) = D-1
  \, . 
  \]
Using this last display, we write
  \[ 
  g (y_1) \ge  \frac{1}{1+(D-1)\exps{-\frac{D}{D-1} y_1 }}
     \, . 
     \]
     The right-hand side is an increasing function of $y_1$, whose minimum value is $-\sqrt{\frac{D-1}{D}}\radius$, and this gives
     \begin{equation}
    \label{eq:equality-bound}
     g (y_1) = \frac{1}{1+(D-1)\exps{\sqrt{\frac{D}{D-1}}\radius}}
     .
     \end{equation}
Thus equality in the key bound is reached for
  \[ 
  y = \left( - \sqrt{\frac{D - 1}{D}} \radius, \sqrt{\frac{1}{D (D - 1)}} \radius,
     \ldots, \sqrt{\frac{1}{D (D - 1)}} \radius \right)^\top 
\, , \]
with value given by Eq.~\eqref{eq:equality-bound}. 
    
\textbf{Step 5: Proof for the maximum.}
     Following the same reasoning as in the proof of Theorem~\ref{theorem:min_softmax-appendix}, we show that the maximum is reached for the point
\[
\left( \sqrt{\frac{D - 1}{D}} \radius, -\sqrt{\frac{1}{D (D - 1)}} \radius,
     \ldots, -\sqrt{\frac{1}{D (D - 1)}} \radius\right)^\top
     \, ,
     \]
and the coordinate $\sigma_1$, and we get the wanted result.
\end{proof}


\section{Additional experimental results}
\label{sec:more-results-appendix}

In this section we collect additional experimental results. 

\subsection{Illustration of Theorem~\ref{th:bounded-traj} with another implementation}

In Figure~\ref{fig:influence-length-document-gensim} and~\ref{fig:influence-number_replacements-gensim}, we present a replication of the experiment presented in Section~\ref{sec:doc2vec-robustness} of the main paper. 
This time, we used the \texttt{gensim} implementation of the \texttt{doc2vec} model. 
The main difference is the use of \emph{hierarchical softmax} instead of softmax. 
Despite this difference, the empirical results remain consistent with our theoretical claims and experimental results with an \emph{ad hoc} implementation. 
We conjecture that the hierarchical softmax has similar algebraic properties to the softmax, in particular kernel stability, which would justify conducting the same analysis. 

\begin{figure}
    \centering
\includegraphics[scale=0.22]{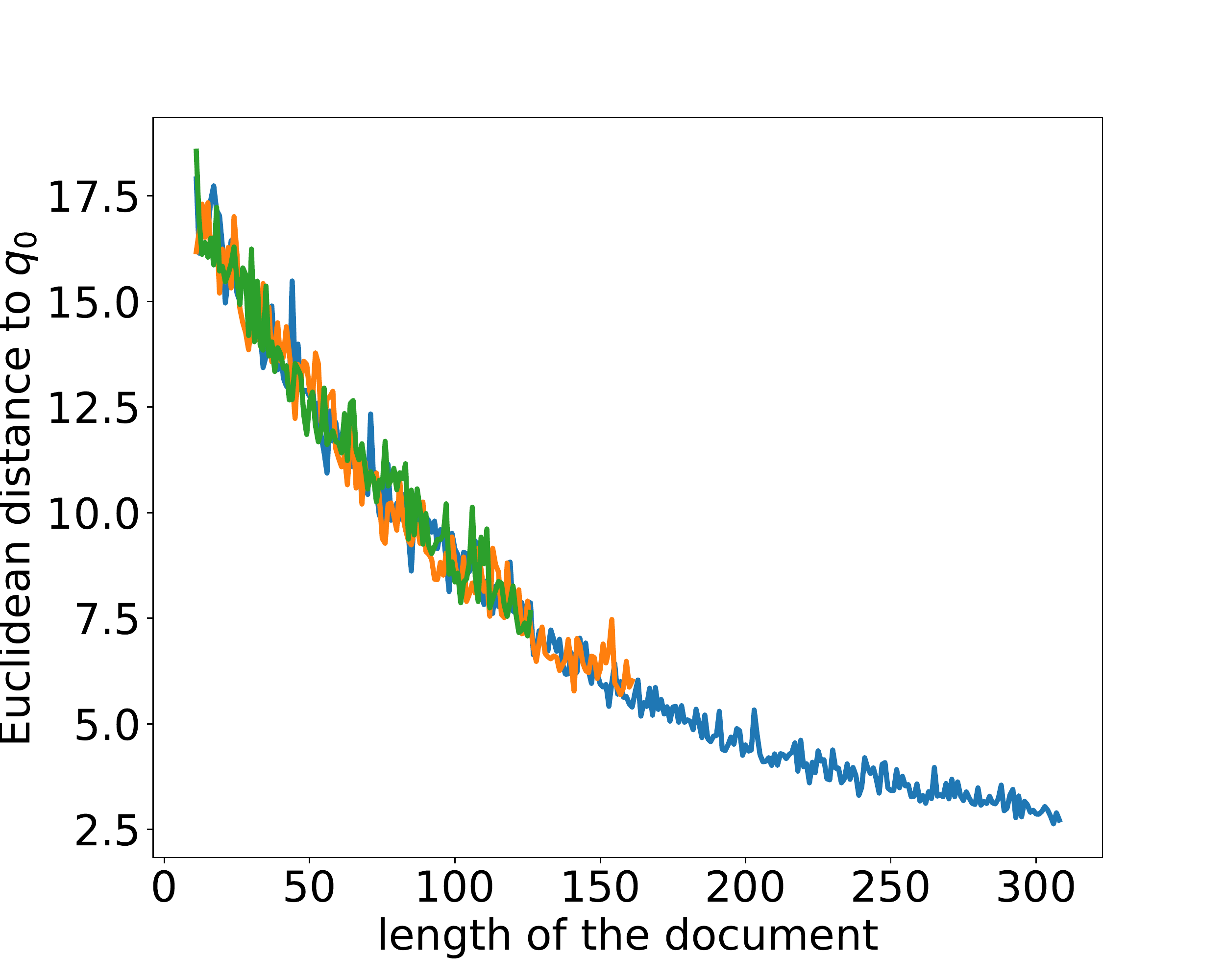}
    \hfill 
    \includegraphics[scale=0.22]{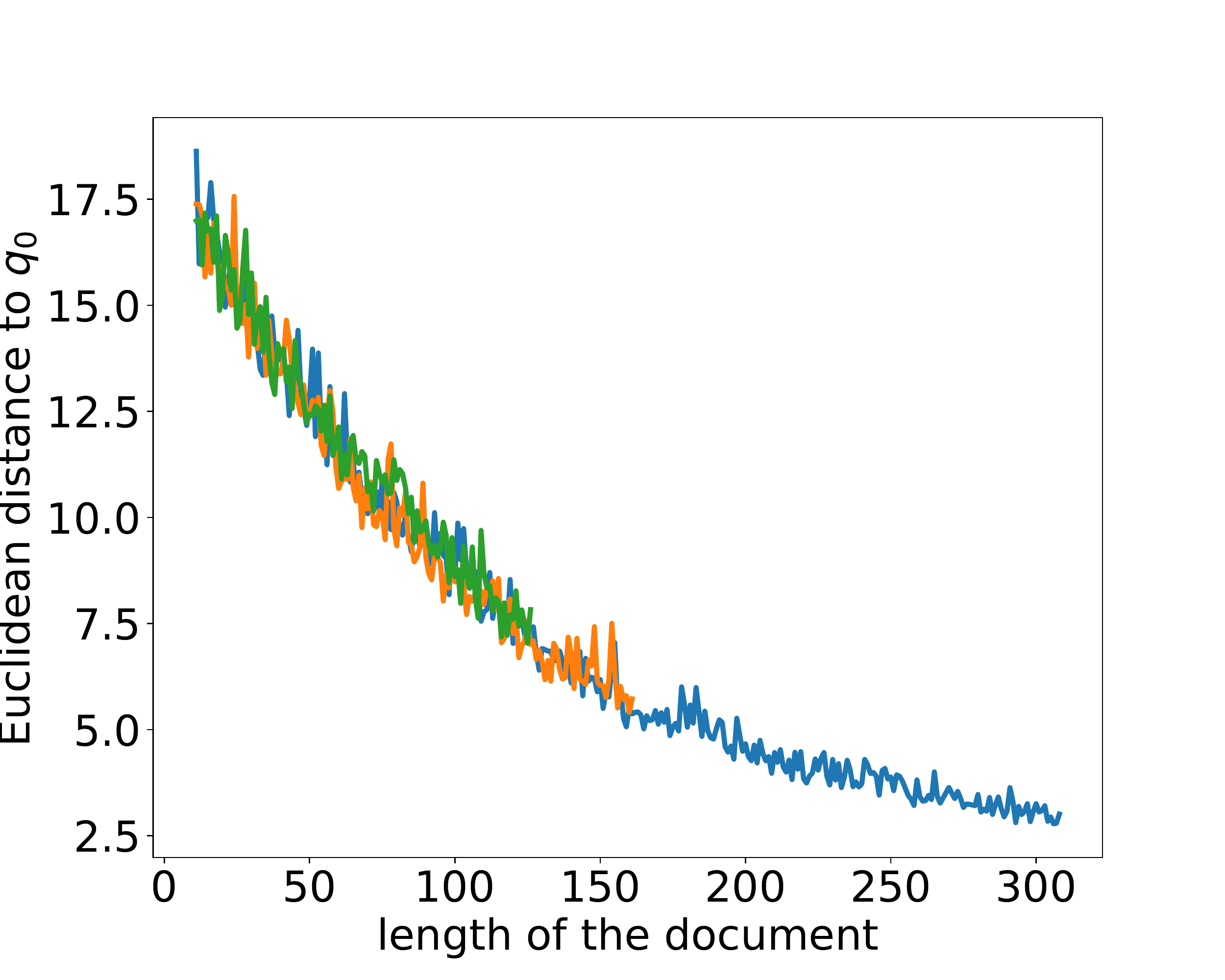}
    \hfill
    \includegraphics[scale=0.22]{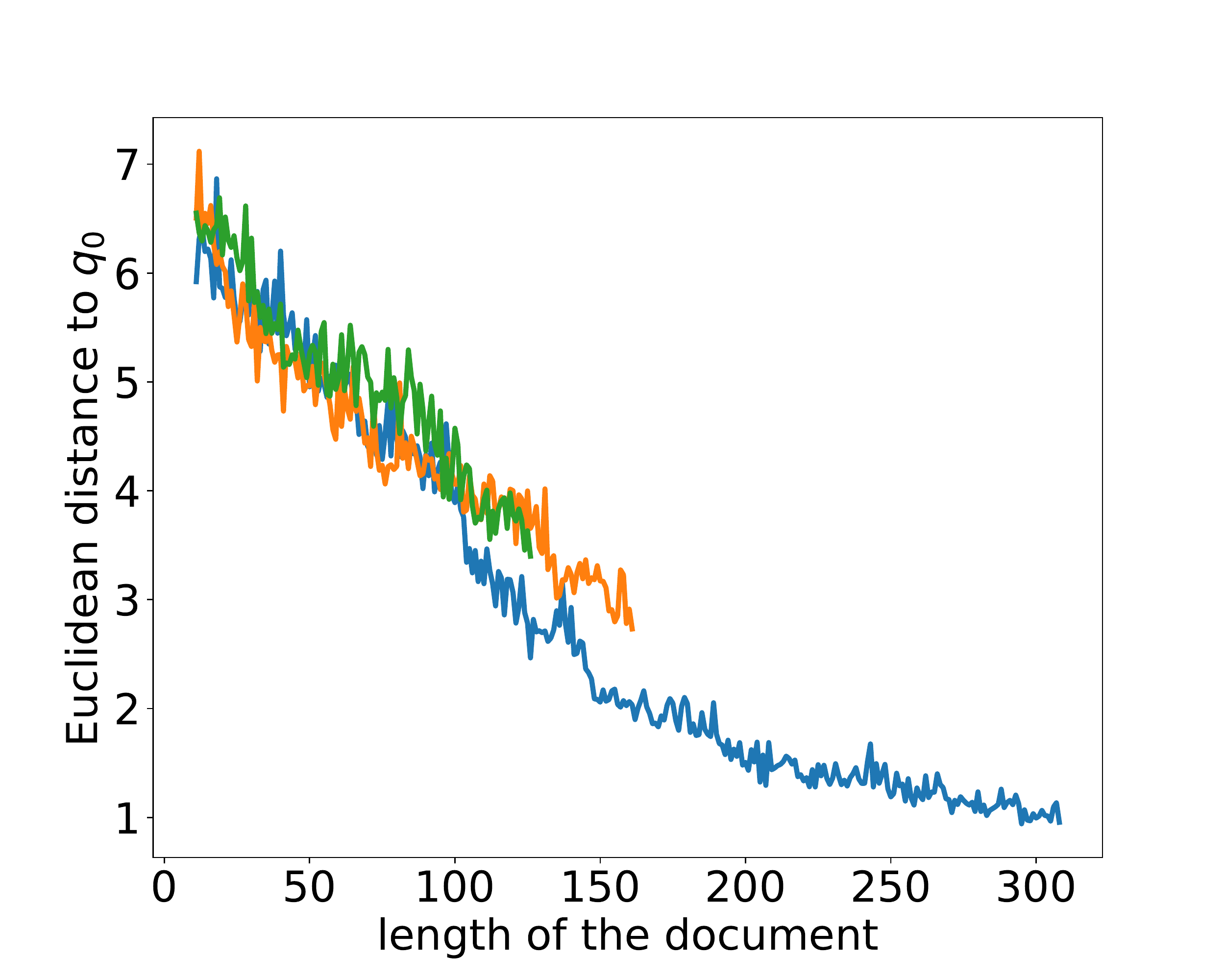}
    \caption{\label{fig:influence-length-document-gensim}Influence of the length of the document with \texttt{gensim} implementation of \texttt{doc2vec}. Increasing the length of a document by considering the first words of $3$ IMDB examples and replacing $5$ words at random several times for each document lengTheorem Dimension of the embedding is $d=50$, size of the dictionary is $D=23,048$.}
\end{figure}

\begin{figure}
    \centering
\includegraphics[scale=0.22]{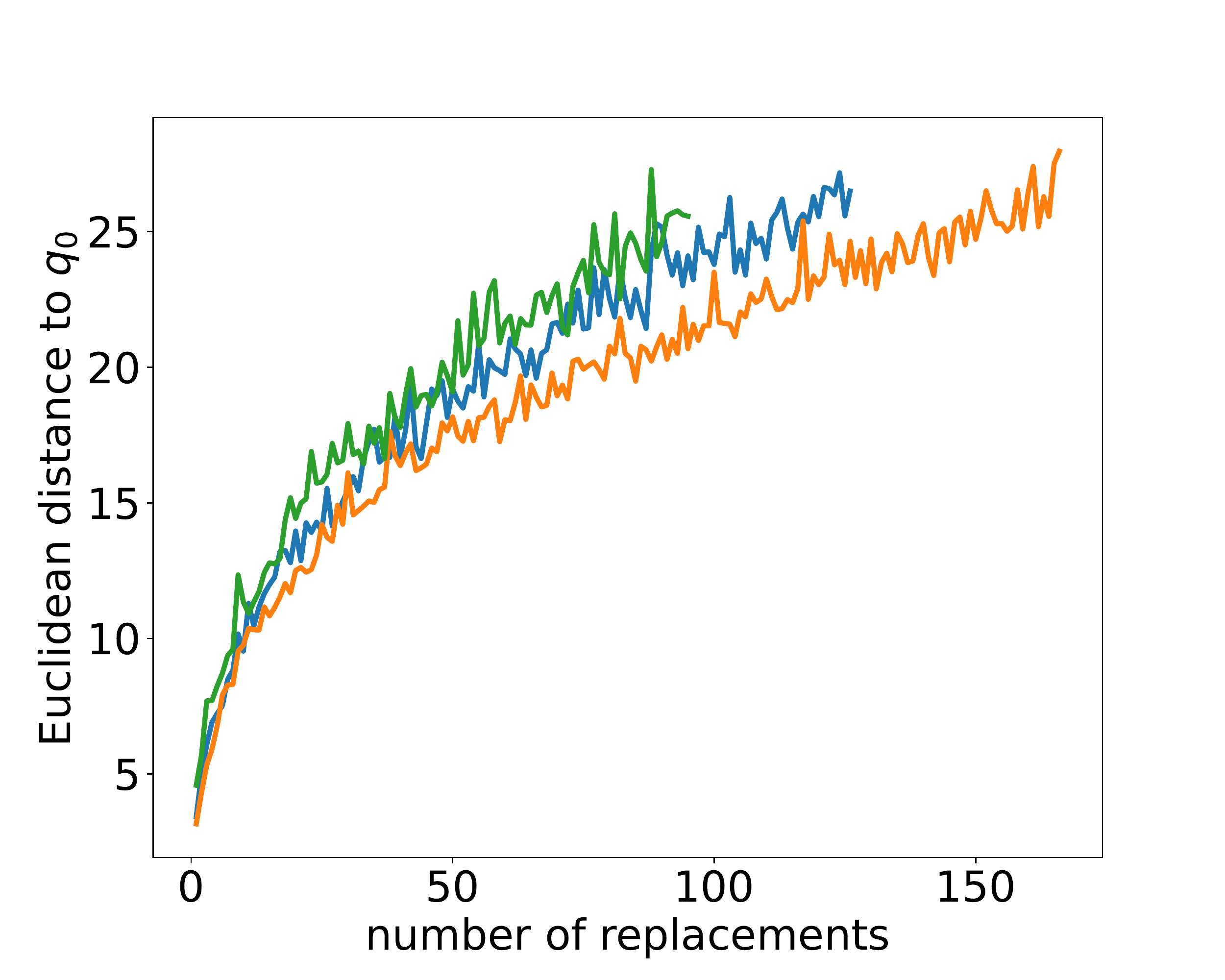}
    \hfill 
    \includegraphics[scale=0.22]{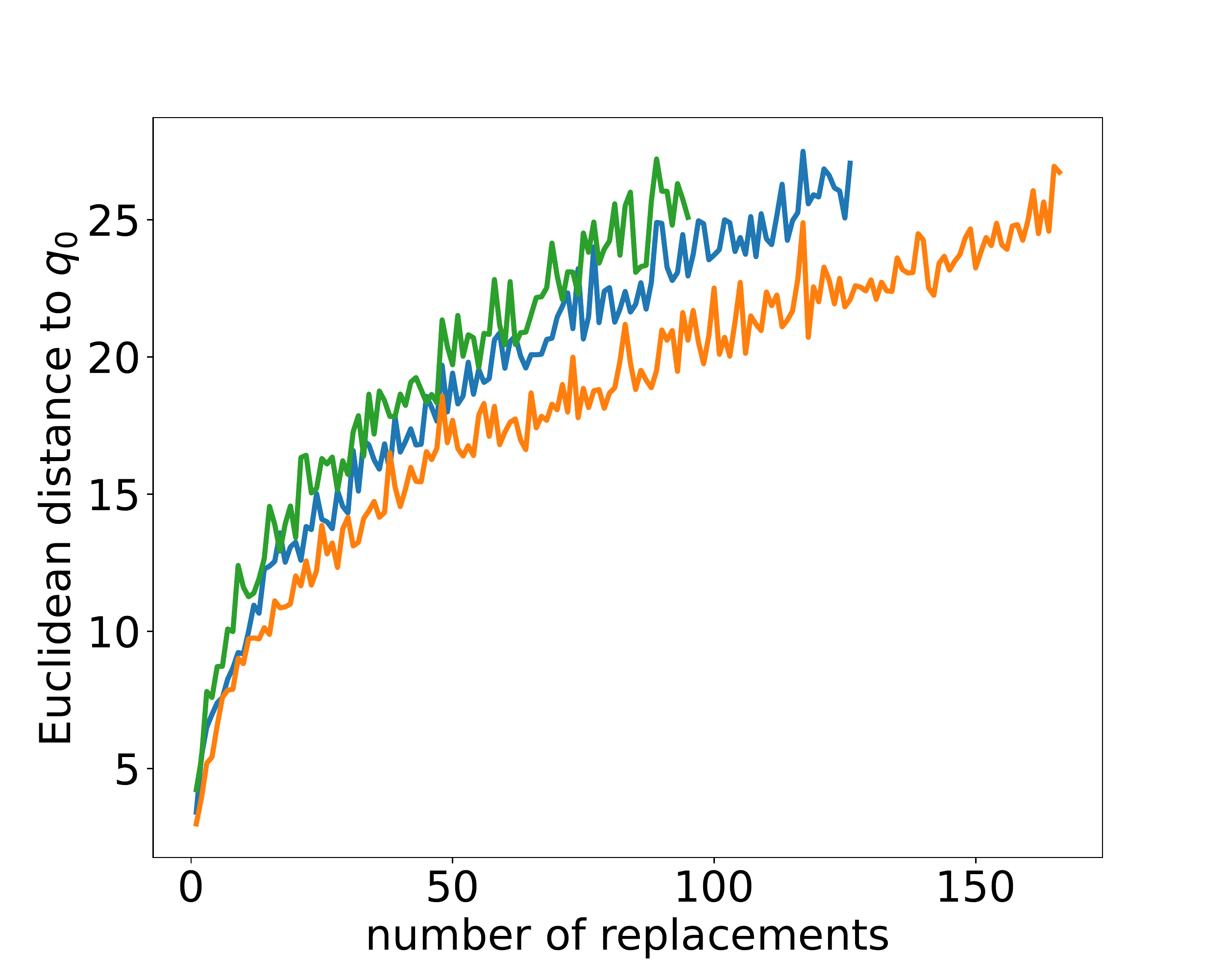}
    \hfill
    \includegraphics[scale=0.22]{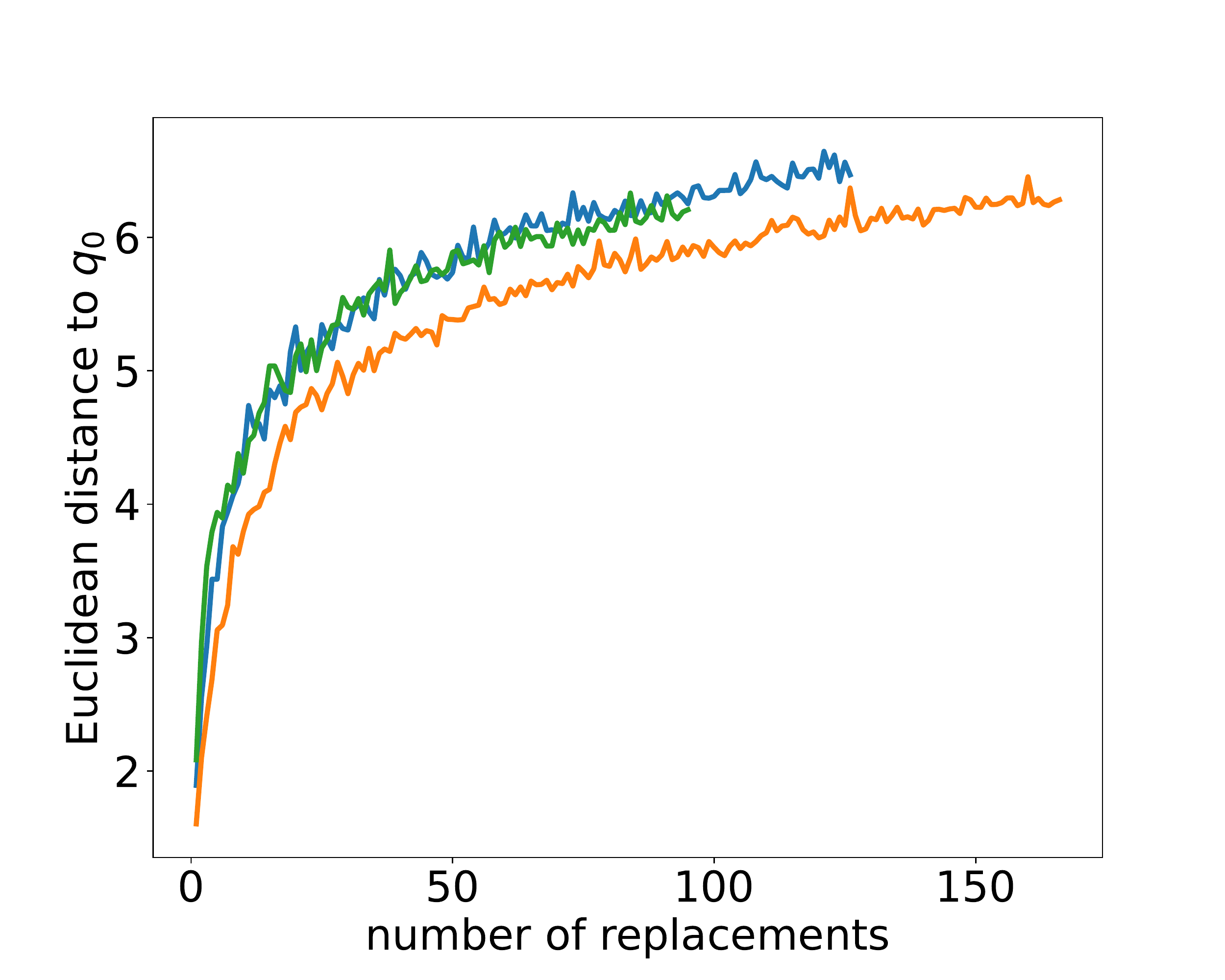}
    \caption{\label{fig:influence-number_replacements-gensim}Influence of number of replacements with \texttt{gensim} implementation of \texttt{doc2vec}.  Considering $3$ examples from the IMDB dataset. Dimension of the embedding is $d=50$, size of the dictionary is $D=23,048$.}
\end{figure}


\subsection{Illustration of Lemma~\ref{lemma:bound_q0-appendix}}
\label{sec:bound-q0-appendix}

In Figure~\ref{fig:norm-original-embedding-gensim}, we illustrate the bound provided by Lemma~\ref{lemma:bound_q0-appendix}. 
We consider the $5$ longest examples of the IMDB dataset and create artificial documents of increasing length as before. 
We observe no asymptotic dependency in $T$, as predicted by the theoretical result. 

\begin{figure}
    \centering
\includegraphics[scale=0.22]{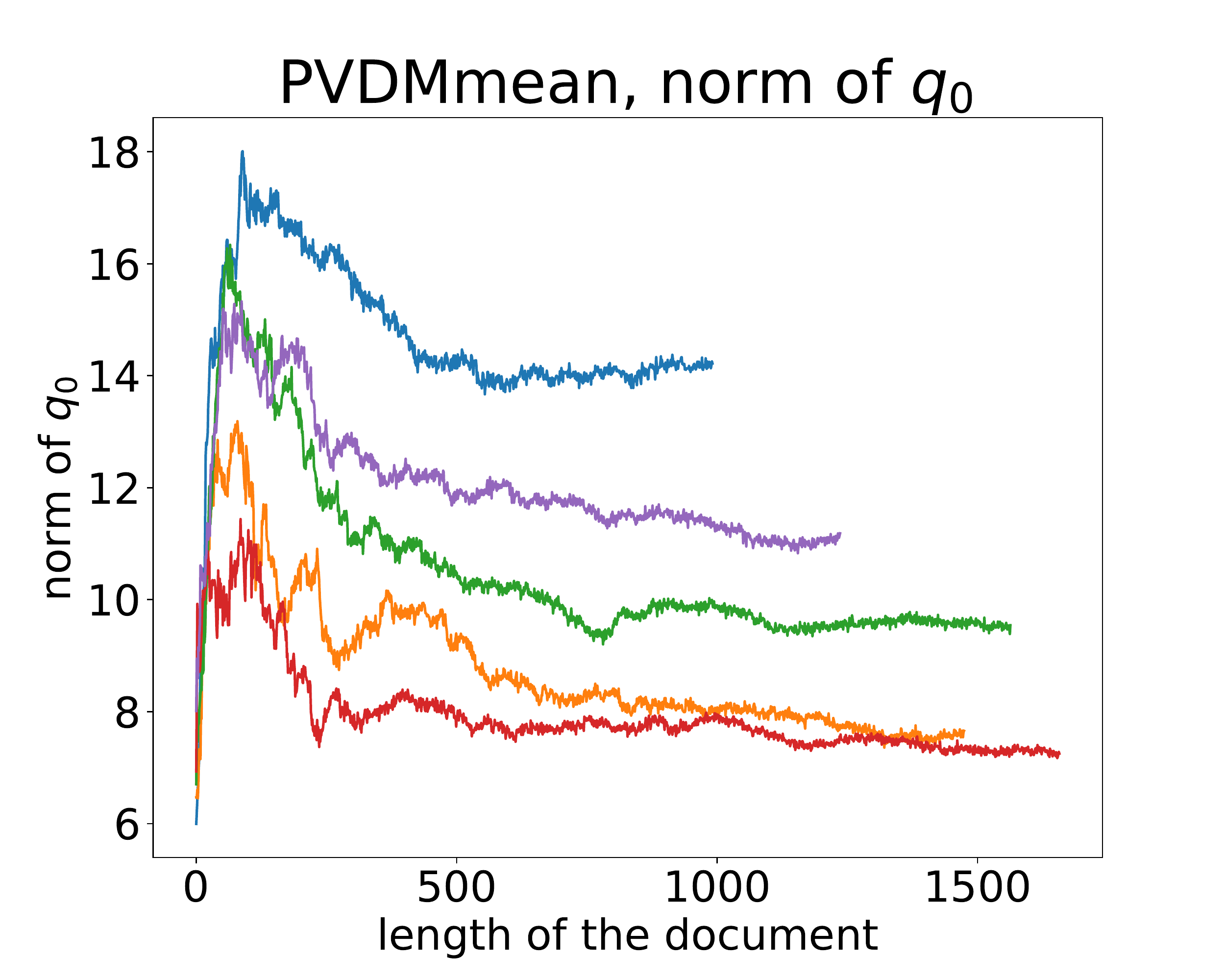}
    \hfill 
    \includegraphics[scale=0.22]{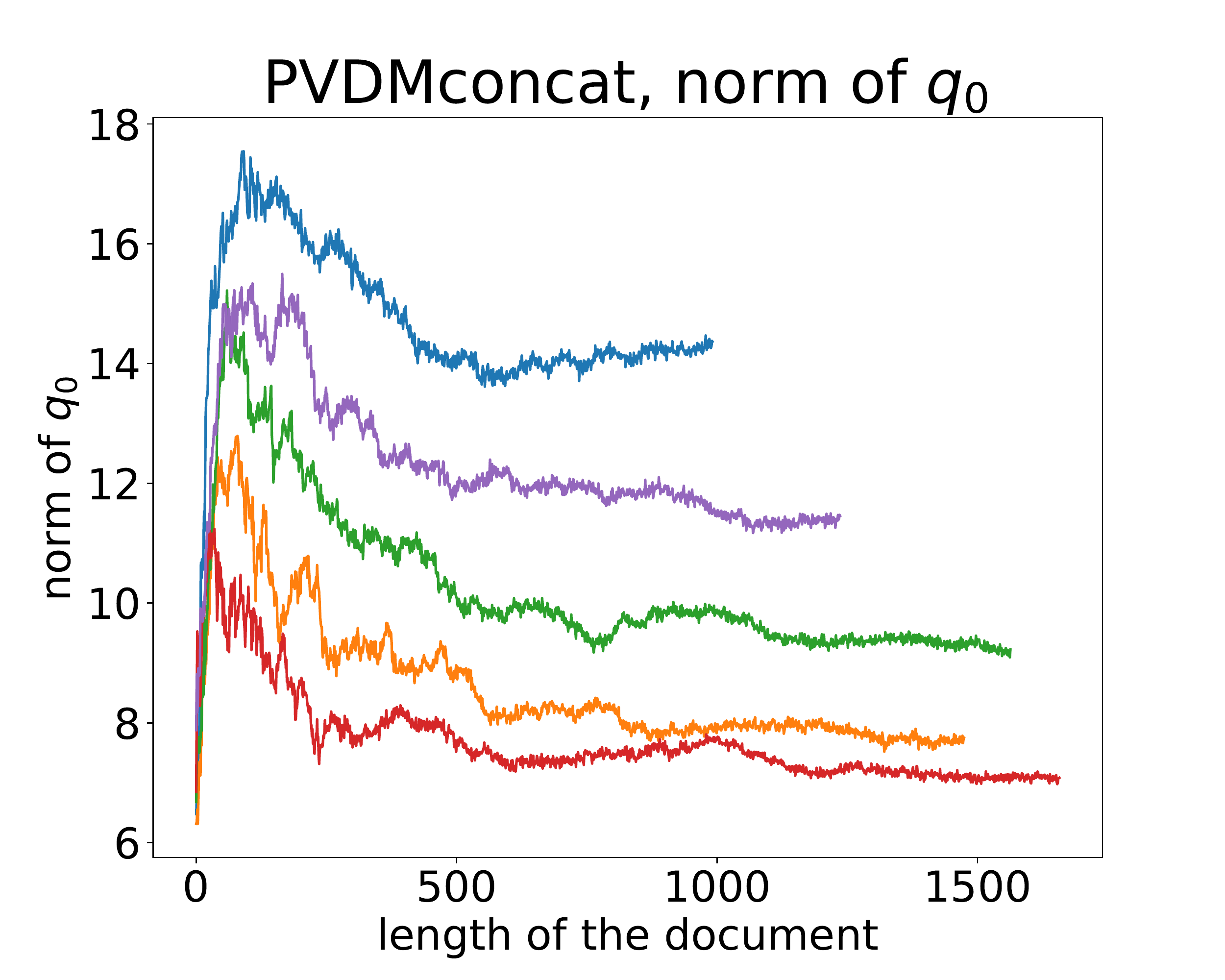}
    \hfill
    \includegraphics[scale=0.22]{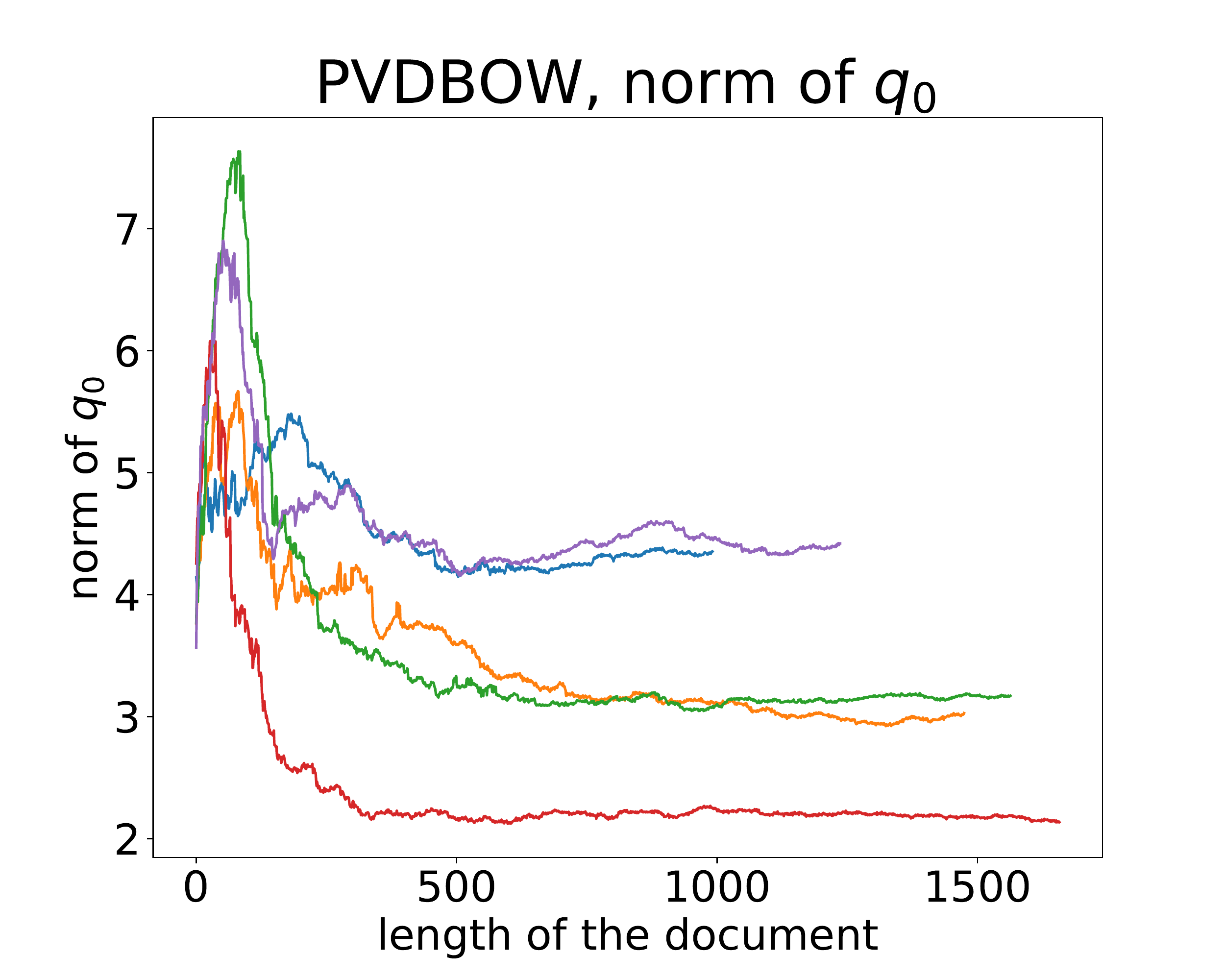}
    \caption{\label{fig:norm-original-embedding-gensim}Norm of the original embedding as a function of $T$.}
\end{figure}


\subsection{Singular values of $\Red$}
\label{sec:svd-appendix}

In Figure~\ref{fig:perso-svd}, we empirically check that the singular values of the (learned) $\Red$ are well-behaved. 
We considered the matrices from our local model and report the histogram of their singular values in log scale in Figure~\ref{fig:perso-svd}. 

\begin{figure}
    \centering
\includegraphics[scale=0.22]{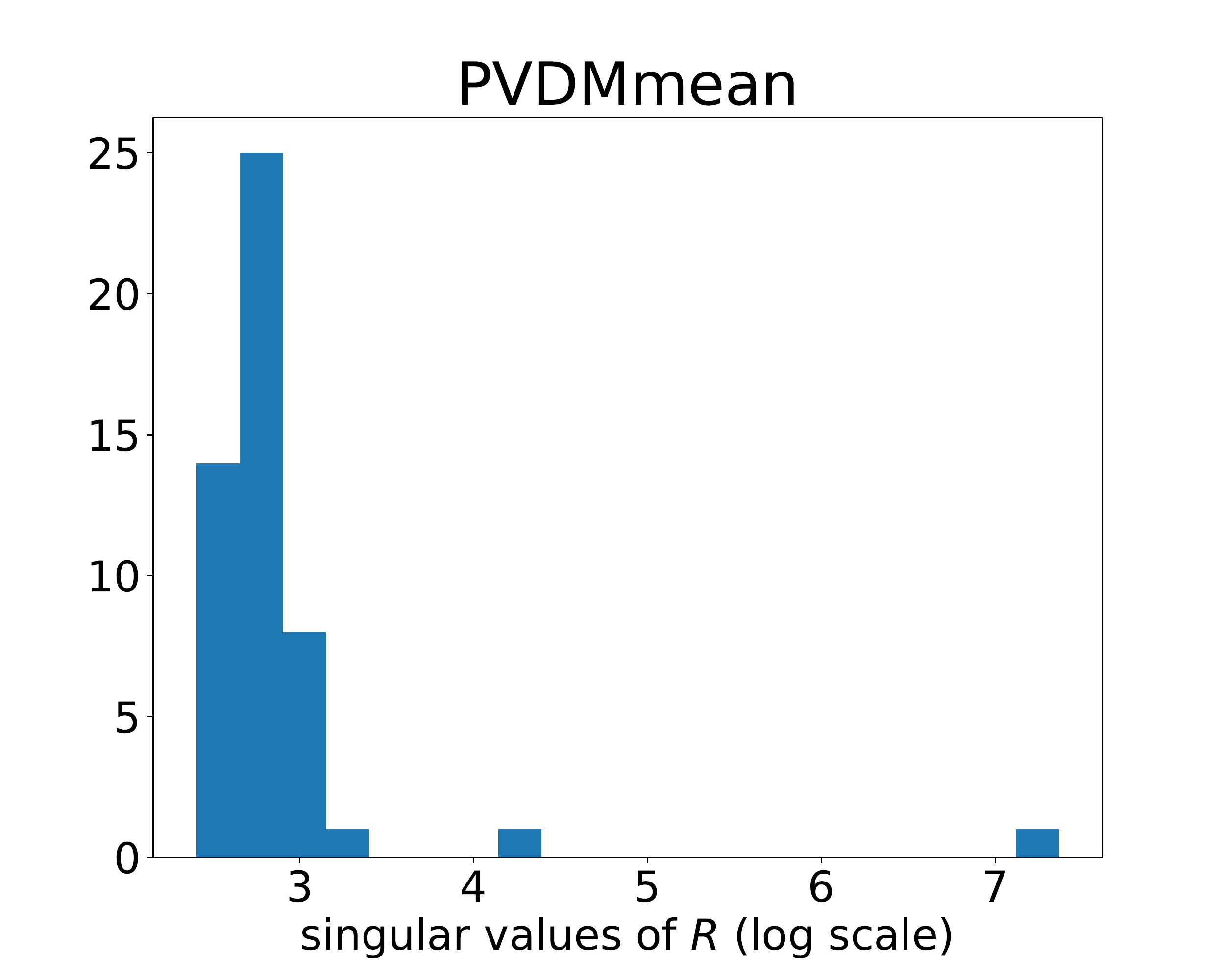}
    \hfill 
    \includegraphics[scale=0.22]{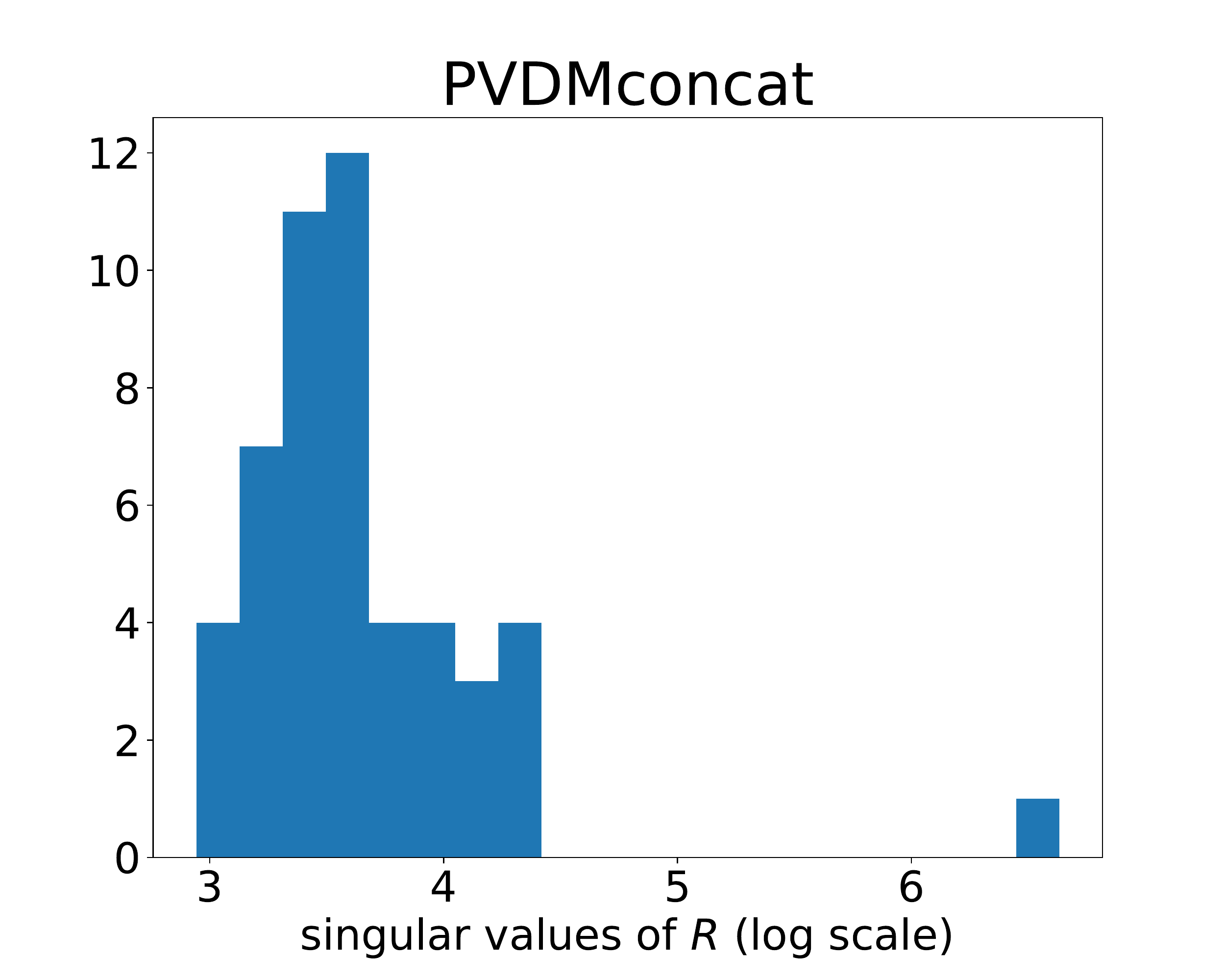}
    \hfill
    \includegraphics[scale=0.22]{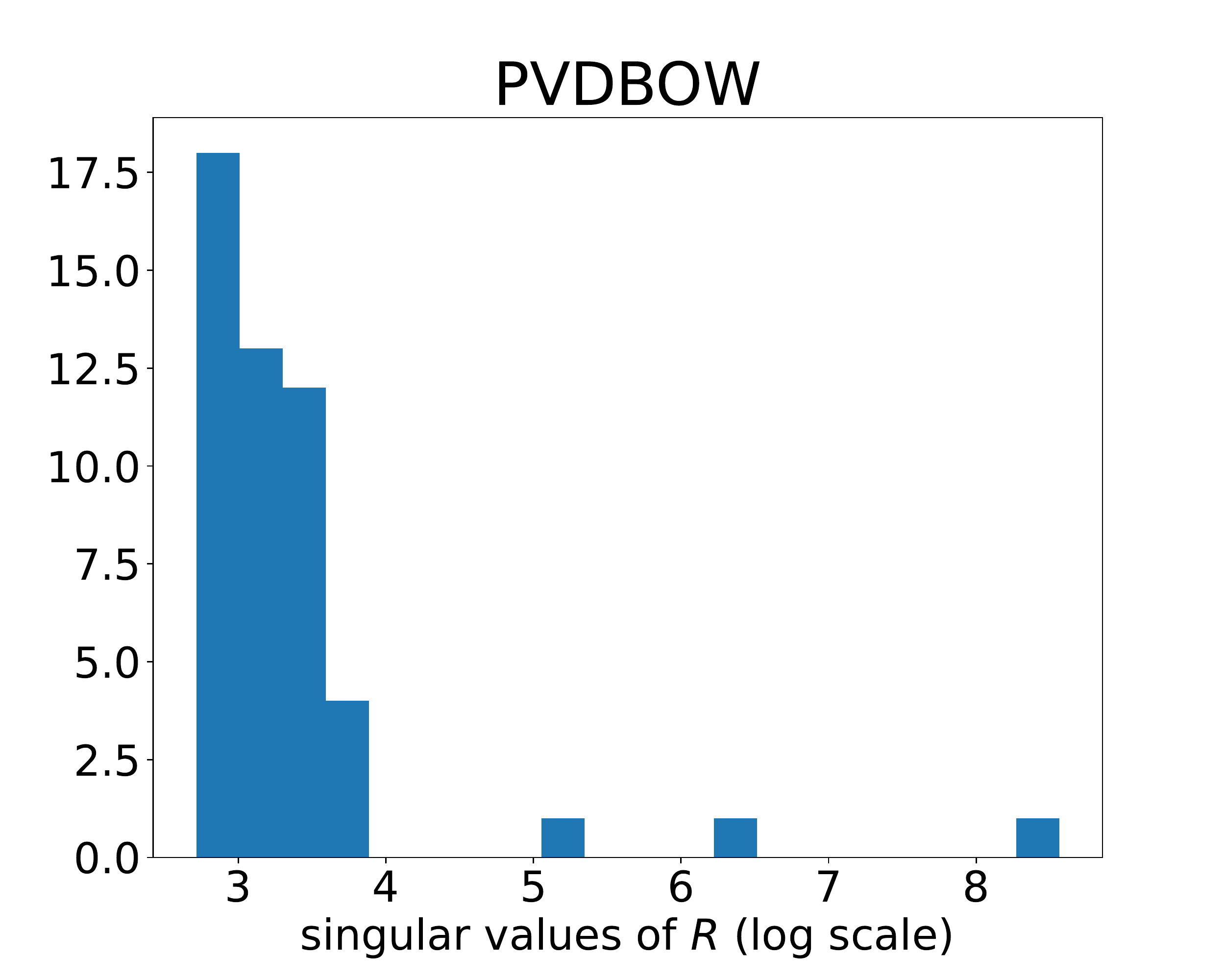}
    \caption{\label{fig:perso-svd}Singular values of $\Red$, in log scale. We observe that $\minsingular{\Red}>0$.}
\end{figure}



\end{document}